\documentclass{article}

\usepackage[preprint]{neurips_2026}

\usepackage[utf8]{inputenc} 
\usepackage[T1]{fontenc}    
\usepackage{hyperref}       
\usepackage{url}            
\usepackage{booktabs}       
\usepackage{amsfonts}       
\usepackage{nicefrac}       
\usepackage{microtype}      
\usepackage{xcolor}         

\usepackage{amsmath}
\usepackage{dsfont}
\usepackage{amsthm}
\newtheorem*{thmstar}{Theorem}
\usepackage{algorithm}
\usepackage{algpseudocode}
\usepackage{graphicx}
\newtheorem{definition}{Definition}
\newtheorem{theorem}{Theorem}
\newtheorem{lemma}[theorem]{Lemma}
\newtheorem{proposition}[theorem]{Proposition}

\title{Pure Exploration for a Good Policy in Reinforcement Learning with Bandit Feedback}

\author{%
  LI ZITIAN\thanks{} \\
  Department of Industrial Systems Engineering \& Management\\
  National University of Singapore\\
  Engineering Drive 2 Block E1A \#06-25 Singapore 117576 \\
  \texttt{lizitian@u.nus.edu} \\
  \And
  Cheung Wang Chi \\
  Department of Industrial Systems Engineering \& Management \\
  Engineering Drive 2 Block E1A \#06-25 Singapore 117576 \\
  \texttt{isecwc@nus.edu.sg} \\
}

\begin{document}

\maketitle

\begin{abstract}
  Pure exploration in episodic Reinforcement Learning has primarily focused on Best Policy Identification (BPI), which seeks to identify a (near)-optimal policy with high confidence. Motivated by practical settings where a ``good enough'' policy suffices, we study an alternate objective of Good Policy Identification (GPI). For a given reward threshold $\mu_0$, GPI only requires identifying a policy with expected reward in an episode at least $\mu_0$ if such a policy exists (positive instance), or declaring None if no such policy exists (negative instance). We formalize GPI under the fixed-confidence setting. We require the output to be correct with probability $\geq 1-\delta$, and seek to minimize the expected sample complexity, which is the expected number of episodes explored for the output. We propose a novel algorithm BEE-GPI, and derive theoretically-grounded upper bounds on its sample complexity for positive and negative instances. Notably, for positive instances, the coefficient of $\log 1/\delta$ in our upper bound is $O(H^2/(V^* - \mu_0)^2)$, where $H$ is the episode length and $V^*$ is the optimal expected reward in an episode. The coefficient does not depend on the action and state space sizes otherwise, in sharp contrast to the sample complexity in BPI. We further establish lower bound results to show the near-optimality of BEE-GPI and the necessity of the $1/(V^* -\mu)^2$ term.  Numerical experiments further validate the efficiency of our approach. 
\end{abstract}

\section{Introduction}


Pure exploration is a fundamental objective in Reinforcement Learning (RL), where a learning agent seeks to identify a policy that satisfies specific performance criteria. Unlike the regret minimization framework, pure exploration prioritizes the probability of correct identification and sample efficiency over total episodes. To date, most research in this field has concentrated on Best Policy Identification (BPI) or $\epsilon$-Best Policy Identification ($\epsilon$-PI), leaving the problem of Good Policy Identification (GPI) largely unexplored. In an episodic Markov Decision Process (MDP) with unknown transition dynamics, GPI requires the agent to identify a policy whose value function exceeds a pre-specified threshold $\mu_0$, if such a policy exists, or to return \textsf{None} otherwise.


Many real-world applications necessitate finding a ``good enough'' policy rather than a nearly optimal one, particularly when precise pairwise comparisons are computationally or practically difficult. For instance, in medical informatics, a new treatment protocol is typically evaluated against a known \textit{Standard of Care} (with known utility $\mu_0$). The objective is to identify a treatment sequence that ensures a patient recovery rate or ``Years of Life Gained'' whose utility is $\geq \mu_0$. In such cases, identifying the absolute best treatment with a high confidence $1-\delta$ is often unnecessary, especially when the performance differences among top-tier policies are marginal. Similarly, in cloud computing, service providers use RL for traffic routing and resource allocation. Rather than seeking an absolutely optimal routing rule, providers often aim to ensure reliability and efficiency levels that exceed a specified Service Level Agreement. Furthermore, institutional investors frequently evaluate whether a new trading strategy meets a minimum Sharpe ratio or Return-on-Investment (ROI) benchmark before deployment. This benchmark $\mu_0$ is often derived from established indices with extensive historical data. In these scenarios, the primary concern is determining whether a new alternative is competitive relative to the benchmark, rather than finding the best strategy among all possibile ones.


\textbf{Main Contributions.} We provide three primary contributions. First, we formalize the Good Policy Identification (GPI) problem in the fixed-confidence setting. 
Given a failure probability tolerance parameter $\delta \in (0,1)$, the learning agent must identify a policy with expected reward at least threshold $\mu_0$ (or correctly identify its non-existence) with probability $\geq 1-\delta$. Second, we develop BEE-GPI, a novel $\delta$-PAC algorithm that utilizes a early-stopping BPI sub-routine as an oracle to achieve efficient sample complexity. We provide a rigorous upper bound on its expected sample complexity, demonstrating that for positive instances, the coefficient of $\log(1/\delta)$ shrinks with the gap $(V^* - \mu_0)$, but is independent of the state and action space sizes. Third, we establish lower bounds for the GPI problem. The resulting gap between our upper and lower bounds is limited to a polynomial factor of $S, A$, and $H$, indicating that BEE-GPI is nearly optimal.

\textbf{Notation.} Denote $\mathbb{N}, \mathbb{R}$ as the sets of positive integers and real numbers respectively. For $n\in \mathbb{N}$, denote $[n] = \{1, \ldots, n\}$. We abbreviate a random variable $R$ following the probability distribution $p$ as $R\sim p$.


\textbf{Problem Formulation.} An instance of the Good Policy Identification (GPI) problem is specified by the tuple $\nu = ( \mathcal{S}, \mathcal{A},\mu_0, \delta, H, R, P, \mathsf{p})$. The sets $\mathcal{S}$ and $\mathcal{A}$ denote the finite state and action spaces, respectively. We denote $S, A$ as the respective cardinalities of $\mathcal{S}, \mathcal{A}$. The scalar $\mu_0 \in \mathbb{R}$ is the reward threshold, which is the expected total reward in an episode that the learning agent aims to accrue. The error parameter $\delta \in (0, 1)$ is the tolerance level on the probability of an incorrect output. The integer $H$ represents the episode horizon. The transition kernels and reward function are respectively denoted by $P = \{p_h(\cdot|s,a)\}_{h\in [H],s\in \mathcal{S},a\in \mathcal{A}}$ and $R = \{r_h(s,a)\}_{h\in [H],s\in \mathcal{S},a\in \mathcal{A}}$, respectively. At round $h\in [H]$ when the agent is at state $s\in \mathcal{S}$ and takes action $a\in \mathcal{A}$, the agent earns a deterministic reward $r_h(s,a)\in [0,1]$, and transits to states $s'$ with probability $p_h(s'|s, a)$. The initial state $S_1$ is distributed according to the initial distribution $\textsf{p}$. 
We assume $\mathcal{S}, \mathcal{A}, \mu_0, \delta, H, R$ are known to the agent, while $P, \mathsf{p}$ are not known. To facilitate our discussion on $P, \mathsf{p}$, we prepend each episode with a fictitious time round $0$. Round $0$ is associated with one fictitious $s_0$ and the single fictitious action $a_0$, which carry zero reward and yield state transition $p_0(\cdot | s_0, a_0) = \textsf{p}(\cdot)$.


A policy $\pi$ is expressed as $(\pi_h)^H_{h=1}\in \Pi^H$, where $\Pi = \{f ~|~ f:\mathcal{S}\rightarrow \mathcal{A}\}$ consists of mapping from the current state to an action. Thus, $\Pi^H$ is the space of all deterministic policies, and $\pi_h$ determines the action in round $h\in [H]$. The value function of $\pi$ 
from round $h$ to the end of the episode is
\begin{align*}
    V_h^{\pi}(s_h) := \mathbb{E}^{\pi} \left[ \sum_{\ell=h}^H r_{\ell}(s_{\ell}, \pi_{\ell}(s_{\ell})) \mid s_h \right],
\end{align*}
where $s_{\ell+1} \sim p_{\ell}(\cdot | s_{\ell}, \pi_{\ell}(s_{\ell}))$. Let $\pi^*\in \Pi^H$ denote an optimal policy. By Bellman optimality, we have $V_h^{\pi^*}(s) \geq V_h^{\pi}(s)$ for any policy $\pi$ and $s\in \mathcal{S}, h\in [H]$. We abbreviate $V^*_0(s_0) = V^{\pi^*}_0(s_0)$ as the optimal total  expected reward in an episode, and occasionally append the instance $\nu$ to denote $V^*_0(s_0 ~|~ \nu)$ to emphasize the dependence on $\nu$.

\textbf{Dynamics.} The agent's algorithm is characterized by a sampling rule $\{\pi^t\}_{t=1}^{\infty}$, a stopping time $\tau$, and a recommendation rule $\hat{\pi}$. When episode $t \in \mathbb{N}$ begins, the algorithm selects a policy $\pi^t = (\pi_h^t)_{h=1}^H \in \Pi^H$ based on the collected data $\mathcal{D}_{t-1}$ in episodes $1, \ldots, t-1$. Executing $\pi^t$ on the instance generates the trajectory $z_t = (s_0, a_0, S_{t,1}, A_{t,1}, r_1(S_{t,1}, A_{t,1}), \dots, S_{t,H}, A_{t,H}, r_H(S_{t,H}, A_{t,H}))$, where $S_{t,h} \sim p_{h-1}(\cdot | S_{t,h-1}, \pi^t_{h-1}(S_{t,h-1}))$. This trajectory is added to the history, i.e. $\mathcal{D}_t = \mathcal{D}_{t-1} \cup \{z_t\}$, where we initialize $\mathcal{D}_0 = \emptyset$. At the end of an episode, the agent may choose to terminate the algorithm. We define the stopping time $\tau$ as the index of the terminating episode. Upon termination, the agent outputs an answer $\hat{\pi} \in \Pi^H \cup \{\textsf{None}\}$ based on $\mathcal{D}_{\tau}$. 

The agent aims to output a policy $\hat{\pi} \in \Pi^H$ satisfying $V_0^{\hat{\pi}}(s_0) \geq \mu_0$ if $\mu_0$ can be achieved, or to output $\textsf{None}$ if the agent concludes that $V^*_0(s_0) < \mu_0$. We formalize the aim in the fixed confidence setting with some definitions that aid our discussions:
\begin{definition}[Positive and Negative Instances]
    A GPI instance $\nu$ 
    is a positive instance if $V^*_0(s_0 ~|~ \nu) > \mu_0$, and is a negative instance if $V^*_0(s_0~|~\nu) < \mu_0$.
\end{definition}
\begin{definition}[Qualified Policy]
    For a positive instance $\nu$, we say that a policy $\pi$ is qualified for $\nu$, if $V_0^{\pi}(s_0~|~\nu) \geq \mu_0$.
\end{definition}
The correctness of the output is qualified by being $\delta$-Probably Approximately Correct ($\delta$-PAC):
\begin{definition}[$\delta$-PAC]
    An algorithm is called $\delta$-PAC, if it can guarantee
    \begin{align*}
        \begin{cases}
            \Pr(V_0^{\hat{\pi}_{\tau}}(s_0~|~\nu)\geq \mu_0,\tau<+\infty)\geq 1-\delta & \text{ if }V^*_0(s_0~|~ \nu)> \mu_0\\
            \Pr(\hat{\pi}_{\tau}=\textsf{None},\tau<+\infty)\geq 1-\delta & \text{ if  } V^*_0(s_0~|~ \nu)< \mu_0
        \end{cases}
    \end{align*}
    holds for any positive or negative Good Policy Identification instance $\nu$.
\end{definition}

\cite{degenne2019pure} have proved that, on an instance $\nu$ with $V_0^*(s_0|\nu)=\mu_0$, any $\delta$-PAC algorithm suffers from $\mathbb{E}_{\nu}[\tau] =  \infty$, even when we restrict $\nu$ to the case of $|\mathcal{S}| = H = 1$. Thus, we only consider positive or negative instances in this paper.

\textbf{Objective.} The goal of the agent is to design a $\delta$-PAC algorithm $(\{\pi^t\}_{t=1}^{\infty},\tau,\hat{\pi})$ that minimizes the expected sampling complexity $\mathbb{E}_{\nu}[\tau]$.

\section{Literature Review}

A popular objective of pure exploration in Reinforcement Learning (RL) is to identify an optimal or $\epsilon$-optimal policy with high confidence. This line of research, often termed Best Policy Identification (BPI), was pioneered by \cite{fiechter1994efficient}. Much of the subsequent work assumes the availability of a generative model, which allows the agent to interact with the environment by querying arbitrary state-action pairs \cite{kearns1998finite, gheshlaghi2013minimax, sidford2018near, sidford2023variance, pmlr-v125-agarwal20b, al2021navigating, li2024breaking}. While the generative setting is theoretically insightful, a more practical and challenging branch of BPI focuses on the online (non-generative) setting, where the agent must navigate the MDP through sequential episodes to collect data. In this context, \cite{dann2015sample} proposed the \textsf{UCFH} algorithm based on the principle of optimism under uncertainty. This was followed by \textsf{BPI-UCBVI} \cite{menard2021fast}, which improved the dependence on the state space $S$ by adapting minimax-optimal regret algorithms. \cite{al2021adaptive} study BPI for RL in non-episodic MDPs. More recently, research has pivoted toward instance-dependent sample complexity. While algorithms like those in \cite{wagenmaker2022beyond} and \cite{narang2024sample} achieve state-of-the-art non-asymptotic bounds, they are often computationally intensive. To bridge this gap, computationally efficient but theoretically suboptimal optimistic sampling rules have been analyzed by \cite{kaufmann2021adaptive, menard2021fast, tirinzoni2023optimistic}.


Despite the extensive literature on BPI and $\epsilon$-PI, these methods cannot be directly applied to the Good Policy Identification (GPI) problem. Recall that for $\epsilon$-PI, the goal is to identify a policy whose expected reward is at least $V^*_0(s_0)-\epsilon $. The difficulty lies in the fact that $V^*_0(s_0)$ is unknown; consequently, an agent cannot pre-determine an appropriate $\epsilon$ for an $\epsilon$-PI algorithm to ensure that the identified policy achieves the threshold $\mu_0$. Furthermore, the lower bounds established for $\epsilon$-PI \cite{pmlr-v132-domingues21a} do not apply under the threshold-based requirements of GPI. 


In the simplified case where $S=1$ and $H=1$, GPI reduces to the multi-armed bandit problem, specifically the ``1-identification'' or ``any-low(high)'' problem studied in \cite{degenne2019pure, katz2020true, pmlr-v267-li25f, li2026closing}, also covered but not mainly focused in \cite{kano2017Good,jourdan2026anytime}. While these works provide different strategies for the bandit setting, the transition to the full episodic MDP setting remains unexplored. 

\section{Algorithms and Sample Complexity Upper Bounds}

We propose the algorithm dubbed Balanced Exploration-Exploitation for Good Policy Identification (BEE-GPI), displayed in Algorithm~\ref{alg:1-policy-identification-recycle-history}. We elaborate on BEE-GPI in Section \ref{sec:algo}, and discuss its theoretical performance guarantees in Section \ref{sec:main}.

\subsection{Algorithm BEE-GPI}\label{sec:algo}
BEE-GPI, displayed in Algorithm \ref{alg:1-policy-identification-recycle-history}, proceeds in phases. Each phase $k = 1, 2, \ldots$ consists of an adaptive exploration stage (Line \ref{alg-line:call-exploration-oracle}) that estimates $P, \mathsf{p}$, and potentially outputs a policy $\hat{\pi}_k$. The exploration stage is followed by a decision branching (Line \ref{alg-line:not-completed} to \ref{alg-line:positive-output}), which leads to three possibilities: (1) the termination of BEE-GPI, or (2) an exploitation stage for policy evaluation (Lines \ref{alg-line:positive-output}-\ref{alg-line:output-policy}), or (3) proceeding to the next phase $k+1$. We elaborate on them in the following three main steps, preceded by a preliminary set-up step 0. 

\begin{algorithm}
    \caption{Balanced Exploration-Exploitation for Good Policy Identification (BEE-GPI)}
    \label{alg:1-policy-identification-recycle-history}
    \begin{algorithmic}[1]
        \State {\bfseries Initialize:} For $k\in \mathbb{N}$, $\epsilon_k=1/2^k$, $\delta_k=1/3^k$ and $\alpha_k=5^k$. Initialize $\mathcal{H}^{\text{ee}}=\mathcal{H}^{\text{et}}=\emptyset, C = 1.01$.\label{alg-line:initialization}
        \For{Phase $k=1,2,\ldots$}\label{alg-line:enter-new-phase}
            \State {\bfseries Explore:} Call Algorithm \ref{alg:Oracle-BPI_UCRL} with $k\gets k$, $\delta\gets\delta_k$, $C\gets C$, $\mathcal{H}^{\text{ee}}\gets $ current exploration history $\mathcal{H}^{\text{ee}}$, exploration budget $T_k^{\text{ee}}\gets (\ref{eqn:def-T_k-ee})$, 

            Return Updated exploration history $\mathcal{H}^{\text{ee}}$, output $\hat{\pi}_k \in \Pi^H\cup \{\textsf{None}, \textsf{Not Completed}\}$ \label{alg-line:call-exploration-oracle}
            
            \If{$\hat{\pi}_k=\textsf{Not Completed}$ }\label{alg-line:not-completed}
                \State Go to the next phase $k+1$
            \ElsIf{$\hat{\pi}_k=\textsf{None}$ and $\delta_k \geq \delta / 10$ }\label{alg-line:None-with-large-tolerance}
                \State Go to the next phase $k+1$
            \ElsIf{$\hat{\pi}_k=\textsf{None}$ and $\delta_k < \delta /10$ }\label{alg-line:negative-output}
                \State \textbf{Return} $\textsf{None}$, and declare the instance as negative.
            \ElsIf{$\hat{\pi}_k\in \Pi^H$}\label{alg-line:positive-output}
                \State {\bfseries Exploit:} Denote $\{X_{i}^{k,\text{et}}\}_{i=1}^{+\infty}$ as i.i.d. random variables, where  the $X_{i}^{k,\text{et}}$ is the total reward collected in executing policy $\hat{\pi}_k$ for an episode in the $i$-th trial. $\{X_{i}^{k,\text{et}}\}_{i=1}^{+\infty}$ are independent of $\mathcal{H}^{\text{ee}}, \mathcal{H}^{\text{et}}$, and denote $\hat{V}_{\text{et}}^{\hat{\pi}_k,N}=\frac{\sum_{i=1}^NX_{i}^{k,\text{et}}}{N}$. Initialize $N=0$. \label{alg-line:Exploitation-starts}

                \While{$N\leq 100\cdot \frac{\log\frac{\alpha_k}{\delta} + \log\log \frac{24H^2}{\epsilon_k}}{\epsilon_k}-1$}\label{alg-line:maximum-episodes-in-exploitation-period}
                    \State Execute $\hat{\pi}_k$ for one episode, and observe the total reward $X_{N+1}^{k,\text{et}}$, $N\gets N+1$.
                    \If{$\hat{V}_{\text{et}}^{\hat{\pi}_k,N}-\sqrt{\frac{H^2 \log [ 2\alpha_k(\log_2 2N)^2 / \delta ]}{N}}\geq \mu_0$}\label{alg-line:output-a-qualified-policy}
                        \State Output $\hat{\pi}=\hat{\pi}_k$, and declare the instance as positive.\label{alg-line:output-policy}
                    \EndIf
                \EndWhile\label{alg-line:Exploitation-ends}
                \State $\mathcal{H}^{\text{et}}\gets \mathcal{H}^{\text{et}}\cup \{X_{i}^{k,\text{et}}\}_{i=1}^{N}$
            \EndIf
        \EndFor \label{alg-line:end-of-phase-k}
    \end{algorithmic}
\end{algorithm}

\textbf{Step 0: Initialization.} BEE-GPI starts by initializing the hyper-parameters 
in Line \ref{alg-line:initialization}. For the exploration stage in phase $k$, the parameter $\epsilon_k$ upper-bounds the optimality gap of the potentially output policy $\hat{\pi}_k$, and $\delta_k$ is the tolerance level on the probability of an incorrect output. The parameter $\alpha_k$ is for tuning the length of the potential exploration stage in phase $k$. The set $\mathcal{H}^\text{ee}$ is for accumulating the observed sample trajectories during the exploration stages. $\mathcal{H}^\text{et}$ is for collecting the total realized rewards by executing a candidate policy $\hat{\pi}_k$, during the exploitation stages. The absolute constant $C$ is for tuning the stopping condition of the exploration stage. 
\textbf{Step 1: Exploration Stage.} The exploration stage involves invoking Algorithm~\ref{alg:Oracle-BPI_UCRL}, which follows the optimistic sampling strategy in the BPI-UCRL framework \cite{kaufmann2021adaptive}. Different from \cite{kaufmann2021adaptive}, Algorithm~\ref{alg:Oracle-BPI_UCRL} involves a novel early-stopping mechanism, which allows saving on the sample complexity for positive instances, and hence dubbed Early-Stopping BPI-UCRL (ES-BPI-UCRL). 

We first describe the optimistic strategy, which is warm-started with the accumulated exploration trajectory set $\mathcal{H}^{\text{ee}}$. To ease the notation, we denote $\mathcal{H}^{\text{ee}} = \{z_{q}\}^{t}_{q=1}$, where $t$ is the total number of episodes in the previous exploration stages. Following \cite{kaufmann2021adaptive}, we optimistically estimate the $V$ function by the following sequence of construction. For each $s\in \mathcal{S}, a\in \mathcal{A}$, we set
$n_h^t(s,a, ) = \sum^{t}_{q=1} \mathbf{1}(S_{q, h} = s, A_{q, h} = a)$.
For each $h\in [H], s\in \mathcal{S},a\in \mathcal{A}$, set the empirical transitional kernel after $t$ episode as $\hat{p}_h^{t}(s'|s,a) =\frac{1}{n^t_h(s,a)} \sum^{t}_{q=1} \mathbf{1}(S_{q, h} = s, A_{q, h}, S_{q, h+1} = s')$ if $n_h^t(s,a) > 0$, and set $\hat{p}_h^{t}(s'|s,a) = 1/|\mathcal{S}|$ otherwise.
We define the exploration bonus $\beta_p(t,\delta) = \log\frac{2SAH}{\delta} + (S-1)\log\big(e(1+\frac{t}{S-1})\big)$, following \cite{kaufmann2021adaptive}. We further define the confidence sets for the transition kernels as: $\mathcal{C}_h^t(s,a;\delta):=\Big\{q(\cdot|s, a)\in\Delta^\mathcal{S}:\text{KL}\big(\widehat{p}_h^t(\cdot | s,a), q(\cdot |s,a)\big)\leq \frac{\beta_p(n_h^t(s,a),\delta)}{n_h^t(s,a)}\Big\}.$
For each pair $(s,h)\in \mathcal{S}\times [H]$, $\delta\in (0, 1)$, we define the optimistic value function $\overline{V}^{t}_h$ and the corresponding optimistic policy $\overline{\pi}^t(\cdot ;\delta) = (\overline{\pi}^t_h(\cdot;\delta))^H_{h=1}$ as:
\begin{align}
\overline{Q}_h^{t}(s,a;\delta):= & r_h(s,a)+\max_{\bar{p}_h(\cdot | s, a)\in \mathcal{C}_h^t(s,a;\delta)}\sum_{s'}\bar{p}_h(s'|s,a) \overline{V}_{h+1}^{t}(s';\delta)\\
    \overline{V}^{t}_h(s;\delta):= & \max_{a\in \mathcal{A}}\overline{Q}_h^{t}(s,a;\delta), \quad    \overline{\pi}^t_h(s;\delta)\in  \underset{a\in \mathcal{A}}{\text{argmax}} \left\{\overline{Q}_h^{t}(s,a;\delta)\right\}\label{eqn:overline-pi^t-definition}
\end{align} 
where $\overline{Q}^{t}_{H+1}(s,a;\delta):= 0,\forall s\in \mathcal{S}, a\in \mathcal{A}, \delta\in (0,1)$. The optimistic $\overline{Q}_h^{t}$ guides our adaptive exploration in each episode. Upon the conclusion of an episode, the exploration history $\mathcal{H}^{\text{ee}}$ and the episode index $t$ are updated. Then, the oracle checks whether any of the three stopping criteria are met (Line~\ref{line-alg:start-execute-MDP-with-new-delta_k}). 

We next discuss the novel early-stopping mechanism of ES-BPI-UCRL, Recall $\overline{\pi}^t(\cdot;\delta) = \{\overline{\pi}^t_h(\cdot ; \delta)\}^H_{h=1}$ defined in (\ref{eqn:overline-pi^t-definition}). We set a lower confidence bound on the value function:
\begin{align*}
    \underline{V}_h^{t,\overline{\pi}^t(\cdot;\delta)}(s;\delta) :=& r_h(s,\overline{\pi}^t_h(s;\delta)) + \min_{\underline{p}_h \in \mathcal{C}_h^t(s,a;\delta)}\sum_{s'} \underline{p}_h(s'|s,\overline{\pi}^t_h(s;\delta)) \underline{V}_{h+1}^{t,\overline{\pi}^t(\cdot;\delta)}(s';\delta),
\end{align*}
where we initialize $ \underline{V}_{H+1}^{t,\overline{\pi}^t(\cdot;\delta)}(s;\delta) = 0$ for all $s$. 
These definitions allow us to specify the following two stopping criteria for the oracle routine (Algorithm~\ref{alg:Oracle-BPI_UCRL}) with tolerance level $\delta_k$:
\begin{align}
\underline{V}^{t,\overline{\pi}^t(\cdot;\delta_k)}_0(s_0;\delta_k) - (C-1)(\overline{V}^t_0(s_0;\delta_k)-\underline{V}^{t,\overline{\pi}^t(\cdot;\delta_k)}_0(s_0;\delta_k)) > & \mu_0 \label{eqn:stop-criterion-pos},\\
    \overline{V}^t_0(s_0;\delta_k) < & \mu_0. \label{eqn:stop-criterion-neg}
\end{align}
The left hand side of (\ref{eqn:stop-criterion-pos}) is a carefully crafted lower confidence bound of the value function of policy $\overline{\pi}^t(\cdot ; \delta) $, while the left hand side of (\ref{eqn:stop-criterion-neg}) represents an upper confidence bound of the value function of policy $\overline{\pi}^t(\cdot ; \delta) $, which is also the upper confidence bound of $V^*_0(s_0)$. The stopping criteria also involve a budget constraint $T^{\text{ee}}_k$ on the cumulative number of exploration episodes conducted so far:
\begin{equation}
    T_k^{\text{ee}} = \frac{(H+1)^2SA\log(2SAH/\delta_k)}{{\epsilon_k}} + \frac{(H+1)^2S^2A}{\epsilon_k}\log\left(\frac{(H+1)^2S^2A \log(2SAH/\delta_k)}{\epsilon_k}\right).\label{eqn:def-T_k-ee}
\end{equation}
The oracle checks whether With probability $1-\delta_k$: (i) condition (\ref{eqn:stop-criterion-pos}) implies $\overline{\pi}^{t_k}(\cdot; \delta_k)$ is a qualified policy ($\hat{\pi}_k = \overline{\pi}^{t_k}(\cdot; \delta_k)$); (ii) condition (\ref{eqn:stop-criterion-neg}) implies no qualified policy exists ($\hat{\pi}_k = \textsf{None}$); and (iii) exceeding the budget ($t > T_k^{\text{ee}}-1$) indicates that the current exploration stage was insufficient.


\begin{algorithm}
    \caption{Exploration Oracle: Early-Stopping BPI-UCRL (ES-BPI-UCRL)}
    \label{alg:Oracle-BPI_UCRL}
    \begin{algorithmic}[1]
        \State {\bfseries Input:} Phase index $k$, Error tolerance level $\delta_k$,  tuning parameter $C>1$, accumulated exploration stage trajectories $\mathcal{H}^{\text{ee}}$, total budget $T_k^{\text{ee}}$.
        \State {\bfseries Initialize:} 
        round index $t=|\mathcal{H}^{\text{ee}}|$, 
        $n_h^t(s,a) = \sum_{q=1}^t \mathbf{1}(S_{q,h}=s, A_{q,h}=a)$, $\{S_{q,h},A_{q,h}\}_{q\in[t], h\in [H]}$ are the trajectories stored in $\mathcal{H}^{\text{ee}}$.
        \While{ $\neg$(\ref{eqn:stop-criterion-pos}) and $\neg$(\ref{eqn:stop-criterion-neg})  and $t\leq T_k^{\text{ee}}-1$} \label{line-alg:start-execute-MDP-with-new-delta_k}
            \For{round $h=0,1,2\cdots,H$} 
                \State Execute $\overline{\pi}^{t}_h(S_{t+1,h};\delta_k)$ by taking action $A_{t+1,h}=\max_a \overline{Q}^t_h(S_{t+1,h},a;\delta_k)$ \Comment{Note $S_{t+1,0} = s_0, A_{t+1,0} = a_0$} \label{line-alg:pulling-strategy-Kauffmann} 
                \State Observe reward $r_h(S_{t+1, h}, A_{t+1, h})$ and next state $S_{t+1,h+1}$
            \EndFor\label{alg-line:leave-episode}
            \State Update $\mathcal{H}^{\text{ee}}\gets \mathcal{H}^{\text{ee}}\cup \{(s_0, a_0,\{S_{t+1, h}, A_{t+1, h}, r_h(S_{t+1, h}, A_{t+1 ,h})\}^H_{h=1})\}$ \label{alg-line:update-history}
            
            \State Update $t\gets t+1$. 
        \EndWhile
        \State Denote $t_{k}\gets t$ for analysis. \Comment{To clarify the definition of $t_k$, $t_k=|\mathcal{H}^{\text{ee}}|$ also holds}
        \If{(\ref{eqn:stop-criterion-pos}) holds}
            \State Output a policy $\hat{\pi}_k=(\overline{\pi}^t_h(\cdot | \delta_k))^H_{h=1}$
        \ElsIf{(\ref{eqn:stop-criterion-neg}) holds}
            \State Output $\hat{\pi}_k=\textsf{None}$
        \ElsIf{$t > T_k^{\text{ee}}-1$}
            \State Output $\hat{\pi}_k=\textsf{Not Completed}$
        \EndIf
    \end{algorithmic}
\end{algorithm}

\textbf{Step 2: Decision Branching.} Based on the output $\hat{\pi}_k$ and the current $\delta_k$, Algorithm~\ref{alg:1-policy-identification-recycle-history} executes one of four logic branches. (i) If $\hat{\pi}_k = \textsf{Not Completed}$ (Line~\ref{alg-line:not-completed}), the phase $k$ exploration stage is inconclusive, and BEE-GPI proceeds phase $k+1$ with an increased budget $T^{\text{ee}}_{k+1}$ and decreased tolerance level $\delta_{k+1}$ to continue exploration. (ii) If $\hat{\pi}_k = \textsf{None}$ and $\delta_k \geq \delta/10$, BEE-GPI still advances to phase $k+1$ but not terminate with a result. This ensures the $\delta$-PAC requirement is met, as BEE-GPI only recommends \textsf{None} when the confidence is sufficiently high ($\delta_k < \delta/10$, the third branch at Line~\ref{alg-line:negative-output}). (iv) If $\hat{\pi}_k$ is a policy (Line~\ref{alg-line:positive-output}), BEE-GPI proceeds to the exploitation phase.

\textbf{Step 3: Exploitation and Verification.} The exploitation phase validates whether the candidate policy $\hat{\pi}_k$ is indeed qualified. Treating the total reward per episode as a random variable, the algorithm constructs a confidence interval with a coverage probability of $1-\frac{\pi^2}{6}\frac{\delta}{\alpha_k}$. By collecting $N$ independent samples, the algorithm evaluates the lower confidence bound: $\hat{V}_{\text{et}}^{\hat{\pi}_k,N} - \sqrt{H^2 \log(2\alpha_k(\log_2 2N)^2/\delta)/N}$. To ensure efficient sample complexity, 
 we require that $N\leq T_k^{\text{et}}-1 $, where 
 \begin{align}
    T_k^{\text{et}} = & 100\cdot \frac{\log(\alpha_k/\delta) + \log\log (24H^2/\epsilon_k)}{\epsilon_k}\label{eqn:def-T_k-et}.
\end{align} 
 If the lower confidence bound exceeds $\mu_0$ within the budget $N \leq T_k^{\text{et}}-1$, BEE-GPI recommends $\hat{\pi}_k$ as a qualified policy. Otherwise, it returns to the exploration phase in phase $k+1$.

\subsection{Performance Guarantee of BEE-GPI, Algorithm \ref{alg:1-policy-identification-recycle-history}}\label{sec:main}
We now establish the theoretical performance of Algorithm~\ref{alg:1-policy-identification-recycle-history}. We first prove that our algorithm satisfies the standard PAC requirement.
\begin{theorem}
    \label{theorem:delta-pac-requirement}
    Algorithm \ref{alg:1-policy-identification-recycle-history} is $\delta$-PAC.
\end{theorem}
The following theorems provide upper bounds on the expected sample complexity $\mathbb{E}[\tau]$ for both positive and negative instances.
\begin{theorem}
    \label{theorem:upper-bound-of-Etau-recycling-pos}
    On a positive instance $\nu$ where $V^*_0(s_0) > \mu_0$, BEE-GPI (Algorithm \ref{alg:1-policy-identification-recycle-history}) satisfies
    \begin{align*}
        \mathbb{E}[\tau] \leq O\left(\frac{H^2\log\frac{1}{\delta}}{(V_0^*(s_0)-\mu_0)^2}\right) + O\left(\frac{ (H+1)^4SA\log\frac{SAH}{(V_0^*(s_0)-\mu_0)^2} + (H+1)^4S^2A\log\left(HSA\right)}{(V_0^*(s_0)-\mu_0)^2}\right).
    \end{align*}
\end{theorem}
\begin{theorem}
    \label{theorem:upper-bound-of-Etau-recycling-neg}
    On a negative instance $\nu$ where $V^*_0(s_0) < \mu_0$, BEE-GPI (Algorithm \ref{alg:1-policy-identification-recycle-history}) satisfies
    \begin{align*}
      &\mathbb{E}[\tau]  \leq  O\left( \frac{(H+1)^4[SA\log\frac{1}{\delta}+S^2 A \log\log\frac{1}{\delta}]}{(V^{*}_0(s_0)-\mu_0)^2}\right) \\
       + & O\left(\frac{(H+1)^4SA\log\frac{H^2}{(V^{*}_0(s_0)-\mu_0)^2}}{(V^{*}_0(s_0)-\mu_0)^2}+ \frac{(H+1)^4S^2A\log\log\left(\frac{H^3SA}{(V^{*}_0(s_0)-\mu_0)^2}
        \right)}{(V^{*}_0(s_0)-\mu_0)^2}\right).
    \end{align*}
\end{theorem}
Theorems \ref{theorem:delta-pac-requirement}, \ref{theorem:upper-bound-of-Etau-recycling-pos}, \ref{theorem:upper-bound-of-Etau-recycling-neg} are proved in Appendix \ref{sec:proof-main-theorem}, and we provide their sketch proof after their discussions. 
The complexity results reveal an inherent asymmetry between positive and negative GPI instances, stemming from the mechanics of our exploration-exploitation stages. For positive instances, the sample complexity is decomposed into a $\delta$-dependent term $O\left(\frac{H^2\log(1/\delta)}{(V_0^*(s_0)-\mu_0)^2}\right)$ and a $\delta$-independent term $O\left(\frac{\text{poly}(S,A,H)}{(V_0^*(s_0)-\mu_0)^2}\right)$. Intuitively, the exploration stage first identifies a candidate policy $\hat{\pi}_k$ that is ``sufficiently good'' (i.e., $V^{\hat{\pi}_k}_0(s_0)-\mu_0 \geq \frac{C-1}{C}(V_0^*(s_0)-\mu_0)$) at a low confidence level, guaranteed by the stopping rule (\ref{eqn:stop-criterion-pos}). The effect of adapting the stopping rules is in Lemma \ref{lemma:output-correctness-under-good-event}. The exploitation stage then uses $O\left(\frac{H^2\log(1/\delta)}{(V_0^*(s_0)-\mu_0)^2}\right)$ samples to verify this candidate with high confidence. Notably, the coefficient of $\log(1/\delta)$ is independent of $S$ and $A$, which is a significant improvement over standard $\epsilon$-optimal BPI reductions by taking $\epsilon=V_0^*(s_0)-\mu_0$, where the coefficient of $\log(1/\delta)$ is typically polynomial in $H, S,A$ and $1/\epsilon^2 = 1/ (V_0^*(s_0)-\mu_0)^2$. 

For negative instances, the algorithm must certify that \textit{every} policy is unqualified. This requires the oracle to globally explore the MDP to ensure the upper confidence bound of the optimal value function is below $\mu_0$. Because this verification requires checking all state-action pairs $(h, s, a)$, the $\log(1/\delta)$ term remains coupled with a polynomial of $S$ and $A$.

\textbf{Sketch Proof of Theorems \ref{theorem:delta-pac-requirement}, \ref{theorem:upper-bound-of-Etau-recycling-pos}, \ref{theorem:upper-bound-of-Etau-recycling-neg}.} 
We define $\mathcal{E}^{\text{cnt}}(\delta_k)$ and $\mathcal{E}(\delta_k)$ as the concentration events respectively on $n^t_h$, $\hat{p}^t_h$  for the $k$-th exploration stage, and $\mathcal{E}_k^{\text{et}}$ as the concentration event for the exploitation stage at phase $k$ on evaluating the objective value under $\hat{\pi}^k$. For a formal definition of the concentration event, we refer the reader to Appendix~\ref{sec:proof-main-theorem}. A key feature of our algorithm is that sampling trajectories are shared across all exploration periods; consequently, these events are nested, i.e., $\mathcal{E}^{\text{cnt}}(\delta_k) \subset \mathcal{E}^{\text{cnt}}(\delta_{k+1})$ and $\mathcal{E}(\delta_k) \subset \mathcal{E}(\delta_{k+1})$.

To analyze the termination round of the algorithm, we introduce the random variable $\kappa$, representing the first phase in which the exploration concentration events hold:
\begin{align*}
    \kappa = \min\{k: \mathcal{E}(\delta_k)\cap \mathcal{E}^{\text{cnt}}(\delta_k) \text{ hold}\}.
\end{align*}
Furthermore, we define $\kappa_{\text{end}}^{\text{pos}}$ and $\kappa_{\text{end}}^{\text{neg}}$ as the indices of the phases by which the algorithm is guaranteed to terminate for positive and negative instances, respectively:
\begin{align*}
    \kappa_{\text{end}}^{\text{pos}}=& \min\left\{k:k\geq \max\left\{\kappa, O(1)\log_2 \frac{H^2}{(V_0^*(s_0)-\mu_0)^2}\right\}, \mathcal{E}_k^{\text{et}}\text{ holds}\right\},\\
    \kappa_{\text{end}}^{\text{neg}}= & \max\left\{\kappa, O(1)\log_2 \frac{H^2}{(V_0^*(s_0)-\mu_0)^2}, \log_3\frac{O(1)}{\delta}\right\}.
\end{align*}

Lemma~\ref{lemma:Termination-Round-of-algorithm} (Appendix~\ref{sec:Auxiliary-Lemma-for-Main-Theorem}) establishes that Algorithm~\ref{alg:1-policy-identification-recycle-history} terminates no later than the conclusion of phase $\kappa_{\text{end}}^{\text{pos}}$ (for positive $\nu$) or $\kappa_{\text{end}}^{\text{neg}}$ (for negative $\nu$). This lemma provides the foundation for our three main results.


\textbf{Correctness ($\delta$-PAC Guarantee).} First, we establish that the algorithm terminates almost surely. Specifically, the probability $\Pr_{\nu}(\tau = +\infty)$ is bounded by the probability that the stopping phases $\kappa_{\text{end}}$ are infinite. Since concentration theorems ensure that $\Pr_{\nu}(\kappa \geq k) \leq \delta_{k-1}$ and the conditional probability $\Pr_{\nu}(\kappa_{\text{end}}^{\text{pos}} - \kappa \geq k \mid \kappa < +\infty)\leq \frac{\pi^2}{6\alpha_{k-1}}$, we conclude that $\Pr_{\nu}(\tau < +\infty) = 1$. Given $\tau < +\infty$ almost surely, Theorem~\ref{theorem:delta-pac-requirement} follows by showing that the conditions for recommending a policy or outputting \textsf{None} only lead to wrong output when a concentration event is violated. By applying a union bound over all phases, we ensure the total error probability is at most $\delta$.


\textbf{Sample Complexity.} We bound the expected sample complexity. Now, Lemma~\ref{lemma:lower-bound-of-sum-average-radius} provides implicit upper bounds for $t_{\kappa_{\text{end}}^{\text{pos}}} - t_{\kappa-1}$ and $t_{\kappa_{\text{end}}^{\text{neg}}} - t_{\kappa-1}$. Combining this with the inequality $t_{\kappa-1}\leq T_{\kappa-1}^{\text{ee}}$, we derive an upper bound for conditional expectation $\mathbb{E}[\tau | \kappa]$. This conditional bound is formulated as $\text{poly}(S,A,H,\log(1/\delta), (V_0^*(s_0)-\mu_0)^{-2})\cdot \frac{\log\frac{\alpha_{\kappa}\epsilon_{\kappa} }{\delta_{\kappa}}}{\epsilon_{\kappa} }$, where $\text{poly}(S,A,H,\log(1/\delta), (V_0^*(s_0)-\mu_0)^{-2})$ is the right-hand side of inequalities in Theorem \ref{theorem:upper-bound-of-Etau-recycling-pos}, \ref{theorem:upper-bound-of-Etau-recycling-neg}. Finally, we compute the total expectation using the law of total expectation: $\mathbb{E}[\tau] = \mathbb{E}[\mathbb{E}[\tau | \kappa]]$. $\Pr_{\nu}(\kappa\geq k)\leq \delta_{k-1}$ implies
\begin{align*}
    \mathbb{E}\left[\frac{\log\frac{\alpha_{\kappa}\epsilon_{\kappa} }{\delta_{\kappa}}}{\epsilon_{\kappa} }\right]\leq \sum_{k=1}^{+\infty}\delta_{k-1}\cdot \frac{\log\frac{\alpha_{k}\epsilon_{k} }{\delta_{k}}}{\epsilon_{k} }  = O(1),
\end{align*}
which concludes Theorems~\ref{theorem:upper-bound-of-Etau-recycling-pos} and \ref{theorem:upper-bound-of-Etau-recycling-neg}.


\section{Sample Complexity Lower Bounds for $\delta$-PAC Algorithms}

We establish lower bounds on the sample complexity of $\delta$-PAC algorithms on positive and negative instances. Due to space constraints, we present the core definitions and the sketch of the main ideas here, while deferred formal proofs and complete notations are provided in Appendix~\ref{sec:full-def-for-lower}.

\subsection{Positive Instances}
We construct a class of difficult MDP instances termed "Uniform Instances."
\begin{definition}[Sketch Definition of a Uniform Instance, specified by $(S,A,H,r,\epsilon)$]
    \label{def:sketch-def-uniform-instance}
    Given even number $H$, integer $S\geq 3$, integer $A\geq 2$, positive real number $r,\epsilon$ such that $\frac{1}{4}<r<r+\epsilon<\frac{3}{4}$. We define a Uniform Instance with $\mathcal{S}=[S-2]\cup\{s_{\text{good}}, s_{\text{bad}}\}$, $\mathcal{A}=[A]$, reward function $r_h(s,a)=1 \text{ iff }s=s_{\text{good}} \text{ else }0$. For all $h,a\in [H]\times\mathcal{A}$, $p_h(\cdot | s_{\text{good}},a)$ $p_h(\cdot | s_{\text{bad}},a)$ are both uniform distribution on $[S-2]$. For all $h,s\in[H]\times [S-2]$,  $p_h(\cdot | s,a)$ is a two-point distribution on $\{s_{\text{bad}},s_{\text{good}}\}$ with $p_h(s_{\text{good}} | s,a=1)=r+\epsilon$, $p_h(s_{\text{good}} | s,a\geq 2)=r$. The initial distribution $p_0(\cdot)$ is uniform on $[S-2]$.
\end{definition}

The construction ensures that at any odd round $h$, the state $S_{t,h}$ follows a uniform distribution over the $[S-2]$,
regardless of the policy. At even round, the agent is forced into either $s_{\text{good}}$ or $s_{\text{bad}}$. The definition of a Uniform Instance assumes an agent can only collect positive reward at the state $s_{\text{good}}$. To maximize reward, the agent must identify the "optimal" action (action 1) at each state $s\in [S-2]$. To characterize the difficulty of identifying this optimal action, we consider the set of all possible permutations of the action space.
\begin{definition}[Permutation of an Instance $\nu$]
    \label{def:def-permutation-instance}
    Given an 1-policy-identification instance $\nu$ with reward function $\{r_h^{\nu}(s,a)\}_{(s,a,h)\in\mathcal{S}\times \mathcal{A}\times [H]}$ and transition probability kernel $\{p_h^{\nu}(s'|s,a)\}_{(h,s's,a)\in [H]\times \mathcal{S}\times\mathcal{S}\times\mathcal{A}}$, and a set of permutations $\sigma=\{\sigma_{s,h}\}_{(s,h)\in \mathcal{S}\times [H]}$, each $\sigma_{s,h}$ is a bijection $\mathcal{A}\rightarrow \mathcal{A}$, we define a permutation of $\nu$ as instance $\nu_{\sigma}$, with reward function $r_h^{\nu_{\sigma}}\big(s,\sigma_{s,h}(a)\big)=r_h^{\nu}(s,a)$, $p_h^{\nu_{\sigma}}(s'|s,\sigma(a))=p_h^{\nu}(s'|s,a)$.
\end{definition}

For a Uniform Instance $\nu$, the optimal policy for its permutation $\nu_{\sigma}$ is $\pi_h(s) = \sigma_{s,h}(1)$, with optimal value $\frac{H(r+\epsilon)}{2}$. In addition, the value function of a policy $\pi$ on instance $\nu_{\sigma}$ is
\begin{align*}
    V_0^{\pi}(s_0~|~ \nu_{\sigma})=\frac{\sum_{h\text{:odd}}\sum_{s=1}^{S-2}\mathds{1}(\pi_h(s) = \sigma_{s,h}(1))(r+\epsilon) + \mathds{1}(\pi_h(s) \neq \sigma_{s,h}(1))r}{S-2}
\end{align*}
It is evident to see $\frac{Hr}{2}\leq V_0^{\pi}(s_0; \nu_{\sigma}) \leq \frac{H(r+\epsilon)}{2}$. To explore for a qualified policy, the agent has to guarantee sufficient exploration for the optimal action $\sigma_{s,h}(1)$ on a fraction of pair $\{(h,s):h\text{:odd}, s\in [S-2]\}$, leading to the following Theorem.
\begin{theorem}
    \label{theorem:symmetric-positive-lower-bound-simplified}
    Consider any $\delta$-PAC alg and any Uniform Instance $\nu$ specified by $(S,A,H,r,\epsilon)$, with $\mu_0$ satisfying $\frac{3H(r+\epsilon)}{8}+\frac{Hr}{8} > \mu_0 > \frac{Hr}{2}$, $S\geq 100,H\geq 10$, $A\geq 100$, $(S-2)H$ is divisible by $16$, $0<\epsilon < \frac{1}{200\sqrt{5}}$. Denote $\Gamma:=\{\sigma=\{\sigma_{s,h}\}_{(s,h)\in \mathcal{S}\times H}:\forall s,h, \sigma_{s,h}\text{ is a bijection }\mathcal{A}\rightarrow\mathcal{A}\}$, the set of all the possible permutations on $\nu$,
    it holds that
    \begin{align*}
        \max_{\sigma\in\Gamma} \mathbb{E}_{\nu_{\sigma},\text{alg}}[\tau] \geq\Omega\left(\frac{H\log\frac{1}{\delta}}{(V^*_0(s_0~|~\nu_{\sigma})-\mu_0)^2}+\frac{H^2 SA}{(V^*_0(s_0~|~\nu_{\sigma})-\mu_0)^2}\right).
    \end{align*}
\end{theorem}
Theorem \ref{theorem:symmetric-positive-lower-bound-simplified} is proved in Appendix \ref{sec:Proof-of-Lower-Bound-pos}. The lower bound suggests that the upper bound in Theorem \ref{theorem:upper-bound-of-Etau-recycling-pos} is near-optimal from the minimax perspective. Specifically, the gap in the $\delta$-dependent term is at most a factor of $H$, while the $\delta$-independent term remains within poly($H,S, \log(A)$).

\subsection{Negative Instances}
We first introduce the sketch definition of Tree Instance, with full definition in the Appendix \ref{sec:full-def-for-lower}.
\begin{definition}[Sketch Definition of a Tree Instance, specified by $(S,A,H,r)$]
    Consider integer $S,A,H$ such that $S= 2^{N}+1$ for some integer $N\geq 1$, and $A\geq 3$, 
    $H-1\geq 6N=6\log_2 (S-1)$, 
    and a real number $r\in (\frac{3}{8},\frac{5}{8})$. Take $\mathcal{S}=[2^N-1]\cup \{s_{\text{good}},s_{\text{bad}}\}$, $\mathcal{A}=[A]$, and consider a binary tree with nodes set $[2^N-1]$. $s_{\text{good}},s_{\text{bad}}$ are absorbed states, while for state $s\in [2^{N-1}-1]$, action 1 will go to state $2s$, action 2 will go to state $2s+1$, for action $a\geq 3$, $p_h(s_{\text{good}}|s, a\geq 3)=r, p_h(s_{\text{bad}}|s, a\geq 3)=1-r$. For state $s\in [2^{N-1}, 2^N-1]$, $p_h(s_{\text{good}}|s, a\geq 1)=r$, $p_h(s_{\text{bad}}|s, a\geq 1)=1-r$. The reward function is $r_h(s,a)=1 \text{ iff }s=s_{\text{good}} \text{ else }0$. The initial state is $1$, with $p_0(1)=1,p_0(s)=0,\forall s\neq 1$.
\end{definition}
In a Tree Instance, an agent only collects rewards upon reaching $s_{\text{good}}$. To maximize the total reward in an episode, the agent must determine the round $h$ and state $s$ at which to transition to the absorbing states. For a threshold $\mu_0 > (H-1)r$, the instance is negative. To correctly identify it as a negative instance, the agent is forced to explore all state-action pairs $(h, s, a)$ for $a \geq 3$ to ensure that no hidden "high-reward" transitions exist. This leads to the following lower bound:
\begin{theorem}
    \label{theorem:lower-bound-negative}
    Apply any $\delta$-PAC algorithm to a Tree instance $\nu$ specified by $(S,A,H,r)$ with $S,A,H\geq 4$ and $ \frac{3}{4}(H-\log_2 (S-3))> \mu_0 > (H-1)r$, we have
    \begin{align*}
        \mathbb{E}_{\nu}\tau \geq \Omega\left(\frac{HSA\log\frac{1}{\delta}}{(\mu_0 - V^*_0(s_0~|~\nu))^2}\right).
    \end{align*}
\end{theorem}


Full proof is in Appendix \ref{sec:Proof-of-Lower-Bound-neg}. Comparing Theorem~\ref{theorem:lower-bound-negative} with our upper bound in Theorem~\ref{theorem:upper-bound-of-Etau-recycling-neg}, we observe that the gap is restricted to a factor of $H^3 \log\frac{SAH}{(\mu_0 - V^*_0)^2} + S \log\log(1/\delta)$. Since $\log\log(1/\delta) < \log(1/\delta)$ for all $\delta \in (0, 1)$, the gap between our upper and lower bounds is bounded by a polynomial in $S, A$, and $H$. Thus, our proposed algorithm BEE-GPI achieves a nearly optimal sample complexity.


\section{Numerical Experiments}



We provide numerical evidence to evaluate the empirical efficiency of our proposed BEE-GPI. We consider two classic environments: the Single Chain and Double Chain MDPs \cite{kaufmann2021adaptive}. For both settings, we fix $\delta=0.001$, $H=8$, $\mathcal{S}=\{0,1,2,3\}$, and $\mathcal{A}=\{\text{left}, \text{right}\}$. The transition dynamics are defined such that the "intended" direction occurs with probability $0.9$ and the "reverse" direction with probability $0.1$. Specifically, $p_h(\max\{s-1,0\} \mid s, \text{left}) = 0.9$ and $p_h(\min\{s+1,3\} \mid s, \text{right}) = 0.9$. The two environments primarily differ in the initial state distribution: the agent always starts at state $0$ in the Single Chain and at state $1$ in the Double Chain.

To evaluate the performance of BEE-GPI, we select threshold values $\mu_0 \in \{1, 1.5, 2.0, 2.5, 3.0\}$, ensuring that all tested instances are positive ($V^*_0 > \mu_0$). We compare our approach against the $\epsilon$-optimal BPI algorithm, BPI-UCRL in \cite{kaufmann2021adaptive} for its execution efficiency. To adapt BPI-UCRL for the GPI task, we provide it with the oracle parameter $\epsilon = V^*_0 - \mu_0$, a significant prior-knowledge advantage. Both algorithms were implemented using the MOSEK optimization solver. We run each experiment with 20 independent trials, and deploy them on an Intel i5-14500 CPU over a total computation time of approximately 72 hours.

\begin{figure}
    \centering
    \includegraphics[width=0.9\textwidth]{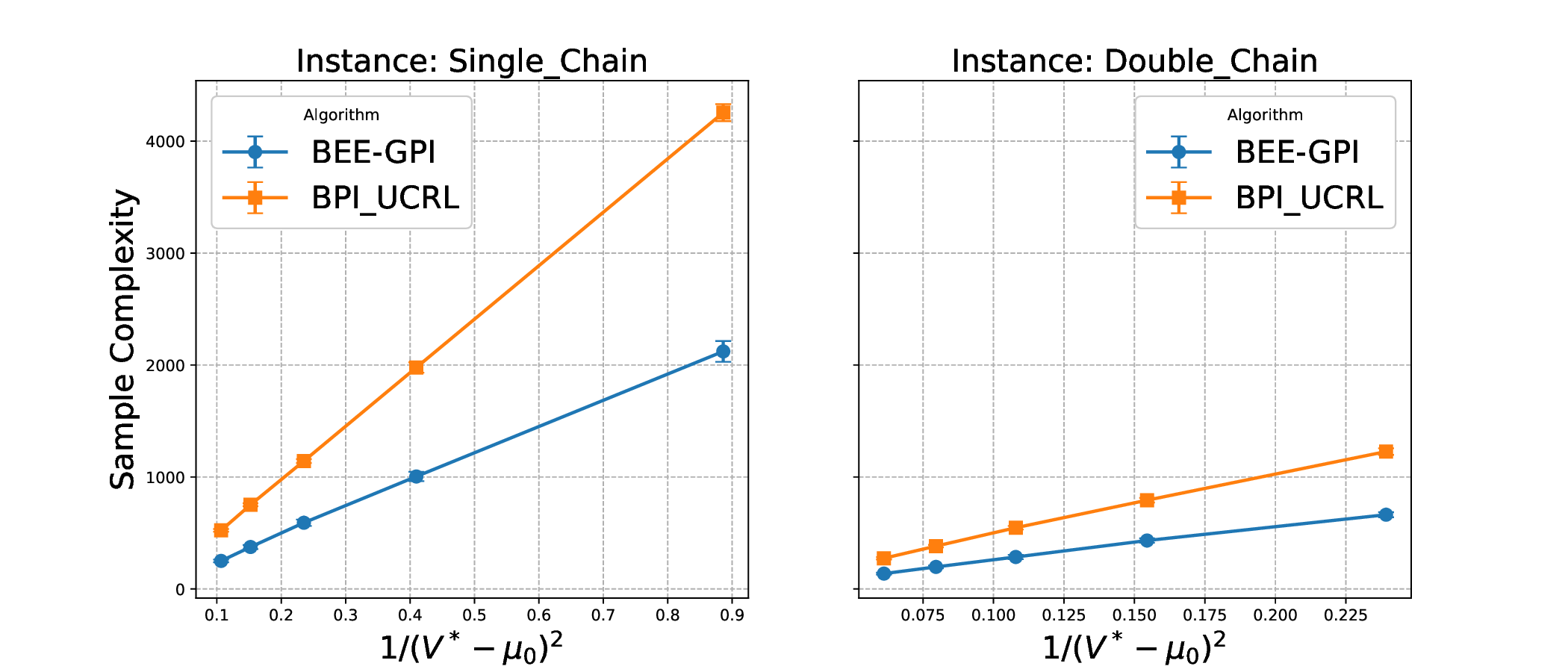}
    \caption{Numerical comparison between BEE-GPI and BPI-UCRL (the lower the better).}
    \label{fig:numeric-result}
\end{figure}


Figure~\ref{fig:numeric-result} presents the empirical stopping times averaged for each setting. The error bars represent three standard deviations ($3\sigma$), though they are notably small due to the low variance of termination round. As illustrated in Figure~\ref{fig:numeric-result}, BEE-GPI consistently outperforms BPI-UCRL by requiring significantly fewer episodes to identify a qualified policy across all tested thresholds. These observations validate our theoretical contribution in Theorem~\ref{theorem:upper-bound-of-Etau-recycling-pos}: by decoupling the $\log(1/\delta)$ coefficient from the state and action space sizes ($S, A$), BEE-GPI achieves a tighter sample complexity than a standard reductions to the state-of-the-art $\epsilon$-optimal BPI algorithms.

\section{Conclusion}

We formalized the Good Policy Identification (GPI) problem and proposed the BEE-GPI algorithm. We establish upper and lower bounds on sample complexity, demonstrating that BEE-GPI is minimax optimal up to a polynomial factor of $S, A$, and $H$. Critically, we showed that for positive instances, our algorithm decouples the $\log 1/\delta$ term from the state-action space size, an efficiency gain validated by our numerical experiments. Future research includes deriving tighter instance-dependent bounds and extending this threshold-based framework to large-scale MDPs using function approximation.





\bibliographystyle{plainnat}
\bibliography{neurips2026}

@article{narang2024sample,
  title={Sample complexity reduction via policy difference estimation in tabular reinforcement learning},
  author={Narang, Adhyyan and Wagenmaker, Andrew and Ratliff, Lillian and Jamieson, Kevin G},
  journal={Advances in Neural Information Processing Systems},
  volume={37},
  pages={22772--22826},
  year={2024}
}

@article{kaufmann2016complexity,
  title={On the complexity of best-arm identification in multi-armed bandit models},
  author={Kaufmann, Emilie and Capp{\'e}, Olivier and Garivier, Aur{\'e}lien},
  journal={The Journal of Machine Learning Research},
  volume={17},
  number={1},
  pages={1--42},
  year={2016},
  publisher={JMLR. org}
}

@inproceedings{al2021adaptive,
  title={Adaptive sampling for best policy identification in markov decision processes},
  author={Al Marjani, Aymen and Proutiere, Alexandre},
  booktitle={International Conference on Machine Learning},
  pages={7459--7468},
  year={2021},
  organization={PMLR}
}

@article{al2021navigating,
  title={Navigating to the best policy in markov decision processes},
  author={Al Marjani, Aymen and Garivier, Aur{\'e}lien and Proutiere, Alexandre},
  journal={Advances in Neural Information Processing Systems},
  volume={34},
  pages={25852--25864},
  year={2021}
}

@inproceedings{wagenmaker2022beyond,
  title={Beyond no regret: Instance-dependent pac reinforcement learning},
  author={Wagenmaker, Andrew J and Simchowitz, Max and Jamieson, Kevin},
  booktitle={Conference on Learning Theory},
  pages={358--418},
  year={2022},
  organization={PMLR}
}

@inproceedings{kaufmann2021adaptive,
  title={Adaptive reward-free exploration},
  author={Kaufmann, Emilie and M{\'e}nard, Pierre and Domingues, Omar Darwiche and Jonsson, Anders and Leurent, Edouard and Valko, Michal},
  booktitle={Algorithmic Learning Theory},
  pages={865--891},
  year={2021},
  organization={PMLR}
}

@article{JMLR:v11:jaksch10a,
  author  = {Thomas Jaksch and Ronald Ortner and Peter Auer},
  title   = {Near-optimal Regret Bounds for Reinforcement Learning},
  journal = {Journal of Machine Learning Research},
  year    = {2010},
  volume  = {11},
  number  = {51},
  pages   = {1563-1600},
  url     = {http://jmlr.org/papers/v11/jaksch10a.html}
}

@article{dann2015sample,
  title={Sample complexity of episodic fixed-horizon reinforcement learning},
  author={Dann, Christoph and Brunskill, Emma},
  journal={Advances in Neural Information Processing Systems},
  volume={28},
  year={2015}
}

@InProceedings{pmlr-v267-li25f,
  title = 	 {Near Optimal Non-asymptotic Sample Complexity of 1-Identification},
  author =       {Li, Zitian and Cheung, Wang Chi},
  booktitle = 	 {Proceedings of the 42nd International Conference on Machine Learning},
  pages = 	 {34144--34184},
  year = 	 {2025},
  editor = 	 {Singh, Aarti and Fazel, Maryam and Hsu, Daniel and Lacoste-Julien, Simon and Berkenkamp, Felix and Maharaj, Tegan and Wagstaff, Kiri and Zhu, Jerry},
  volume = 	 {267},
  series = 	 {Proceedings of Machine Learning Research},
  month = 	 {13--19 Jul},
  publisher =    {PMLR},
  url = 	 {https://proceedings.mlr.press/v267/li25f.html},
}

@inproceedings{menard2021fast,
  title={Fast active learning for pure exploration in reinforcement learning},
  author={M{\'e}nard, Pierre and Domingues, Omar Darwiche and Jonsson, Anders and Kaufmann, Emilie and Leurent, Edouard and Valko, Michal},
  booktitle={International Conference on Machine Learning},
  pages={7599--7608},
  year={2021},
  organization={PMLR}
}

@InProceedings{pmlr-v132-domingues21a,
  title = 	 {Episodic Reinforcement Learning in Finite MDPs: Minimax Lower Bounds Revisited},
  author =       {Domingues, Omar Darwiche and M{\'e}nard, Pierre and Kaufmann, Emilie and Valko, Michal},
  booktitle = 	 {Proceedings of the 32nd International Conference on Algorithmic Learning Theory},
  pages = 	 {578--598},
  year = 	 {2021},
  editor = 	 {Feldman, Vitaly and Ligett, Katrina and Sabato, Sivan},
  volume = 	 {132},
  series = 	 {Proceedings of Machine Learning Research},
  month = 	 {16--19 Mar},
  publisher =    {PMLR},
  url = 	 {https://proceedings.mlr.press/v132/domingues21a.html}
}

@inproceedings{tirinzoni2023optimistic,
  title={Optimistic pac reinforcement learning: the instance-dependent view},
  author={Tirinzoni, Andrea and Al-Marjani, Aymen and Kaufmann, Emilie},
  booktitle={International Conference on Algorithmic Learning Theory},
  pages={1460--1480},
  year={2023},
  organization={PMLR}
}

@inproceedings{fiechter1994efficient,
  title={Efficient reinforcement learning},
  author={Fiechter, Claude-Nicolas},
  booktitle={Proceedings of the seventh annual conference on Computational learning theory},
  pages={88--97},
  year={1994}
}

@article{kearns1998finite,
  title={Finite-sample convergence rates for Q-learning and indirect algorithms},
  author={Kearns, Michael and Singh, Satinder},
  journal={Advances in neural information processing systems},
  volume={11},
  year={1998}
}

@article{gheshlaghi2013minimax,
  title={Minimax PAC bounds on the sample complexity of reinforcement learning with a generative model},
  author={Gheshlaghi Azar, Mohammad and Munos, R{\'e}mi and Kappen, Hilbert J},
  journal={Machine learning},
  volume={91},
  number={3},
  pages={325--349},
  year={2013},
  publisher={Springer}
}

@InProceedings{pmlr-v125-agarwal20b,
  title = 	 {Model-Based Reinforcement Learning with a Generative Model is Minimax Optimal},
  author =       {Agarwal, Alekh and Kakade, Sham and Yang, Lin F.},
  booktitle = 	 {Proceedings of Thirty Third Conference on Learning Theory},
  pages = 	 {67--83},
  year = 	 {2020},
  editor = 	 {Abernethy, Jacob and Agarwal, Shivani},
  volume = 	 {125},
  series = 	 {Proceedings of Machine Learning Research},
  month = 	 {09--12 Jul},
  publisher =    {PMLR},
  url = 	 {https://proceedings.mlr.press/v125/agarwal20b.html}
}

@article{sidford2018near,
  title={Near-optimal time and sample complexities for solving Markov decision processes with a generative model},
  author={Sidford, Aaron and Wang, Mengdi and Wu, Xian and Yang, Lin and Ye, Yinyu},
  journal={Advances in Neural Information Processing Systems},
  volume={31},
  year={2018}
}

@article{sidford2023variance,
  title={Variance reduced value iteration and faster algorithms for solving Markov decision processes},
  author={Sidford, Aaron and Wang, Mengdi and Wu, Xian and Ye, Yinyu},
  journal={Naval Research Logistics (NRL)},
  volume={70},
  number={5},
  pages={423--442},
  year={2023},
  publisher={Wiley Online Library}
}

@article{li2024breaking,
  title={Breaking the sample size barrier in model-based reinforcement learning with a generative model},
  author={Li, Gen and Wei, Yuting and Chi, Yuejie and Chen, Yuxin},
  journal={Operations Research},
  volume={72},
  number={1},
  pages={203--221},
  year={2024},
  publisher={INFORMS}
}

@inproceedings{katz2020true,
  title={The true sample complexity of identifying good arms},
  author={Katz-Samuels, Julian and Jamieson, Kevin},
  booktitle={International Conference on Artificial Intelligence and Statistics},
  pages={1781--1791},
  year={2020},
  organization={PMLR}
}

@article{degenne2019pure,
  title={Pure exploration with multiple correct answers},
  author={Degenne, R{\'e}my and Koolen, Wouter M},
  journal={Advances in Neural Information Processing Systems},
  volume={32},
  year={2019}
}

@article{kano2017Good,
  title={Good Arm Identification via Bandit Feedback},
  author={ Kano, Hideaki  and  Honda, Junya  and  Sakamaki, Kentaro  and  Matsuura, Kentaro  and  Sugiyama, Masashi },
  journal={Machine Learning},
  number={2},
  year={2017},
}

@article{li2026closing,
  title={Closing the Gap on the Sample Complexity of 1-Identification},
  author={Li, Zitian and Cheung, Wang Chi},
  journal={arXiv preprint arXiv:2601.15620},
  year={2026}
}

@article{jourdan2026anytime,
  title={An anytime algorithm for good arm identification},
  author={Jourdan, Marc and Delahaye-Duriez, Andr{\'e}e and R{\'e}da, Cl{\'e}mence},
  journal={Journal of Machine Learning Research},
  volume={27},
  number={19},
  pages={1--90},
  year={2026}
}


\appendix

\section{Notation}
\label{sec:appendix-notations}
In this section, We list the notation used in Algorithms \ref{alg:1-policy-identification-recycle-history}, \ref{alg:Oracle-BPI_UCRL}, and we supplement the details on the notation in the analysis in Appendix \ref{sec:Performance-Guarantee-of-Algorithm-main-body}.
\begin{itemize}
    \item $r_h(s,a)$: mean reward of taking action $a$ in state $s$ at round $h$.
    \item $p_h(\cdot|s,a)$: true transition kernel when we take action $a$ in state $s$ at round $h$.
    \item $p_h^{\pi}(s,a):$ Execute a given policy $\pi$ on the MDP. Denote $S_h$ as the random state the agent locate on round $h$, we define $p_h^{\pi}(s,a) = \Pr(S_h=s,\pi_h(s)=a)$.
    \item $S_{t,h}$: realized state round $h$ at $t$ th episode.
    \item 
    $\beta^{\text{cnt}}(\delta) = \log\frac{2SAH}{\delta}$,  $\beta_p(t,\delta) = \log\frac{2SAH}{\delta} + (S-1)\log\Big(e(1+\frac{t}{S-1})\Big)$

    \item $\hat{p}_h^{t}(\cdot|s,a)$ as empirical transition kernel after executing $t$ episodes during the exploration stages. $\hat{p}_h^{t}(s'|s,a) =\frac{1}{n^t_h(s,a)} \sum^{t}_{q=1} \mathbf{1}(S_{q, h} = s, A_{q, h}, S_{q, h+1} = s')$ if $n_h^t(s,a) > 0$, and set $\hat{p}_h^{t}(s'|s,a) = 1/|\mathcal{S}|$ otherwise.
    \item $n^t_h(s,a):=\sum^{t}_{q=1} \mathbf{1}(S_{q, h} = s, A_{q, h} = a)$, visiting times of pair $(s,a,h)$ after executing $t$ episode during the exploration stages. $n^t_h(s',s,a):=\sum^{t}_{q=1} \mathbf{1}(S_{q, h+1} = s, S_{q, h} = s, A_{q, h} = a)$, visiting times of pair $(s,a,h)$ and the next state is $s'$ after executing $t$ episodes. Here we denote $\mathcal{H}^{\text{ee}}=\{S_{q,h},A_{q,h}\}_{q\in [t], h\in [H]}$.
    \item $\bar{n}_h^t(s,a):$  Denote $k_q$ as the stage index at $q$-th sampled episode, and $\bar{\pi}^q:=\overline{\pi}^{q}(\cdot; \delta_{k_q})$, defined in (\ref{eqn:overline-pi^t-definition}). We define $\bar{n}_h^t(s,a):=\sum_{q=0}^{t-1}p_h^{\bar{\pi}^q}(s,a)$. We define $\overline{\pi}^0$ by (\ref{eqn:overline-pi^t-definition}) with $\mathcal{C}_h^0(s,a;\delta_1):=\Delta^\mathcal{S}$.
    \item $V^*_0(s_0|\nu)$ optimal value function of instance $\nu$. If we only mention one instance $\nu$, we may use $V^*_0(s_0)$ for simplicity.
\end{itemize}

What's more, we introduce the confidence interval for each pair $(s,a,h)$
\begin{align*}
    \mathcal{C}_h^t(s,a;\delta):=\left\{q(\cdot|s, a)\in\Delta^\mathcal{S}:\text{KL}\big(\widehat{p}_h^t(\cdot | s,a), q(\cdot |s,a)\big)\leq \frac{\beta_p(n_h^t(s,a),\delta)}{n_h^t(s,a)}\right\}.
\end{align*}
Following \cite{al2021adaptive}, we define the same optimistic and pessimistic approximation value function, for each $(h,s,a)$.
\begin{align*}
    \overline{Q}_h^{t}(s,a;\delta):= & r_h(s,a)+\max_{\bar{p}_h(\cdot | s, a)\in \mathcal{C}_h^t(s,a;\delta)}\sum_{s'}\bar{p}_h(s'|s,a) \overline{V}_{h+1}^{t}(s';\delta)\\
    \overline{V}^{t}_h(s;\delta):= & \max_a\overline{Q}_h^{t}\big(s,a;\delta\big)\\
    \overline{V}^{t}_{H+1}(s;\delta):= & 0\\
    \overline{p}_h^{t}(s,a;\delta)\in & \arg\max_{\bar{p}_h\in \mathcal{C}_h^t(s,a,\delta)}\sum_{s'}\bar{p}_h(s'|s,a) \overline{V}_{h+1}^{t}(s';\delta)\\
    \overline{\pi}^t_h(s;\delta)\in & \arg\max_{a}\overline{Q}_h^{t}\big(s,a;\delta\big)\\
    \overline{V}_0^t(s_0;\delta) = & \max_{\bar{p}_0\in \mathcal{C}_0^t(s_0;\delta)}\sum_{s'}\bar{p}_0(s'|s_0,a_0) \overline{V}_{1}^{t}(s';\delta)
\end{align*}
\begin{align*}
    \underline{Q}_h^{t}(s,a;\delta) :=& r_h(s,a) + \min_{\underline{p}_h \in \mathcal{C}_h^t(s,a;\delta)}\sum_{s'} \underline{p}_h (s'|s,a)\underline{V}_{h+1}^{t}(s';\delta) \\
    \underline{V}_h^{t}(s;\delta) :=& \max_a \underline{Q}_h^{t}(s,a;\delta) \\
    \underline{V}_{H+1}^{t}(s;\delta) :=& 0 \\
    \underline{p}_h^{t}(s,a;\delta) &\in \arg\min_{\underline{p} \in \mathcal{C}_h^t(s,a;\delta)} \sum_{s'} \underline{p}_h (s'|s,a)\underline{V}_{h+1}^{t}(s';\delta)\\
    \underline{\pi}^t_h(s;\delta)\in & \arg\max_{a}\underline{Q}_h^{t}\big(s,a;\delta\big)\\
    \underline{V}_0^t(s_0;\delta) = & \min_{\underline{p}_0\in \mathcal{C}_0^t(s_0;\delta)}\sum_{s'}\bar{p}_0(s'|s_0,a_0) \underline{V}_{1}^{t}(s';\delta)
\end{align*}
\begin{align*}
    \overline{Q}_h^{t,\pi}(s,a;\delta):= & r_h(s,a) +\max_{\bar{p}_h\in \mathcal{C}_h^t(s,a;\delta)}\sum_{s'}\bar{p}_h(s'|s,a) \overline{V}_{h+1}^{t,\pi}(s';\delta)\\
    \overline{V}^{t,\pi}_h(s;\delta):= & \overline{Q}_h^{t,\pi}\big(s,\pi(s);\delta\big)\\
    \overline{V}^{t,\pi}_{H+1}(s;\delta):= & 0\\
    \overline{p}_h^{t,\pi}(s,a;\delta)\in & \arg\max_{\bar{p}_h\in \mathcal{C}_h^t(s,a,\delta)}\sum_{s'}\bar{p}_h(s'|s,a) \overline{V}_{h+1}^{t,\pi}(s';\delta)
\end{align*}
\begin{align*}
    \underline{Q}_h^{t,\pi}(s,a;\delta) :=& r_h(s,a) + \min_{\underline{p}_h \in \mathcal{C}_h^t(s,a;\delta)}\sum_{s'} \underline{p}_h(s'|s,a) \underline{V}_{h+1}^{t,\pi}(s';\delta) \\
    \underline{V}_h^{t,\pi}(s;\delta) :=& \underline{Q}_h^{t,\pi}(s,\pi(s);\delta) \\
    \underline{V}_{H+1}^{t,\pi}(s;\delta) :=& 0 \\
    \underline{p}_h^{t,\pi}(s,a;\delta) &\in \arg\min_{\underline{p} \in \mathcal{C}_h^t(s,a;\delta)} \sum_{s'} \underline{p}_h(s'|s,a) \underline{V}_{h+1}^{t,\pi}(s';\delta).
\end{align*}
Finally, we introduce the commonly used notation $Q$ and $V$, noted as the ``True'' value function.
\begin{align*}
    Q_h^{\pi}(s,a):= & r_h(s,a)+\sum_{s'}p_h(s'|s,a) V_{h+1}(s')\\
    V_h^{\pi}(s):= & Q_h^{\pi}(s,\pi(s))\\
    V_{H+1}^{\pi}(s):= & 0\\
    V_0^{\pi}(s_0) = & \sum_{s'}p_0(s'|s_0,a_0) V_{1}^{\pi}(s').
\end{align*}

\section{Performance Guarantee of Algorithm \ref{alg:1-policy-identification-recycle-history}}
\label{sec:Performance-Guarantee-of-Algorithm-main-body}

\subsection{Proof of Main Theorems}
\label{sec:proof-main-theorem}

Recall we take default value $\delta_k=\frac{1}{3^k}$ in Algorithm \ref{alg:1-policy-identification-recycle-history}. We define
\begin{align}
    \mathcal{E}^{\text{cnt}}(\delta)= & \left\{\forall t\in\mathbb{N}^*, \forall (s,a,h): n^t_h(s,a)\geq \frac{1}{2}\bar{n}_h^t(s,a)-\beta^{\text{cnt}}(\delta)\right\}\label{eqn:E-cnt-event}\\
    \mathcal{E}(\delta)= & \left\{\forall t\in \mathbb{N},\forall h\in [H], \forall (s,a, h), \text{KL}(\hat{p}^t_h(\cdot| s,a), p_h(\cdot| s,a))\leq \frac{\beta_p(n_h^t(s,a),\delta)}{n_h^t(s,a)}\right\}\label{eqn:E-event}\\
    \kappa = & \min\{k: \mathcal{E}(\delta_k)\cap \mathcal{E}^{\text{cnt}}(\delta_k) \text{ hold}\}\\
    \mathcal{E}_k^{\text{et}} = & \left\{\forall N>0, \left|\hat{V}_{\text{et}}^{\hat{\pi}_k,N} - V^{\hat{\pi}_k}_0\right|\leq 
    \sqrt{\frac{4H^2 \log\frac{2\alpha_k(\log_2 2N)^2}{\delta}}{N}} \right\}, \label{eqn:E-et-event}
\end{align}
where $\mathcal{E}^{\text{cnt}}(\delta), \mathcal{E}(\delta)$ are only defined for the samples collected in the exploration period. Recall A key feature of Algorithm \ref{alg:1-policy-identification-recycle-history} is that sampling trajectories are shared across all exploration periods; consequently, $\{\mathcal{E}^{\text{cnt}}(\delta_k), \mathcal{E}(\delta_k)\}_{k=1}^{+\infty}$ are nested, i.e., $\mathcal{E}^{\text{cnt}}(\delta_k) \subset \mathcal{E}^{\text{cnt}}(\delta_{k+1})$ and $\mathcal{E}(\delta_k) \subset \mathcal{E}(\delta_{k+1})$.

We first prove the output correctness of Algorithm \ref{alg:1-policy-identification-recycle-history}.
\begin{thmstar}[Restatement of Theorem \ref{theorem:delta-pac-requirement}]
    Algorithm \ref{alg:1-policy-identification-recycle-history} is $\delta$-PAC.
\end{thmstar}
\begin{proof}[Proof of Theorem \ref{theorem:delta-pac-requirement}]
    By the Lemma \ref{lemma:Must-Terminate-with-Prob1}, we know $\Pr_{\nu}(\tau<+\infty)=1$ holds no matter $\nu$ is positive or negative. The remaining work is to prove $\Pr_{\nu}(\hat{\pi}\neq \textsf{None}, V^{\hat{\pi}}_0(s_0) > \mu_0)\geq 1-\delta$ for any positive instance $\nu$ and $\Pr_{\nu}(\hat{\pi}=\textsf{None})\geq 1-\delta$ for any negative instance $\nu$. Or equivalently, $\Pr_{\nu}(\hat{\pi}= \textsf{None}\text{ or } V^{\hat{\pi}}_0(s_0) < \mu_0)<\delta$ holds for any positive instance $\nu$, and $\Pr_{\nu}(\hat{\pi}\neq\textsf{None})< \delta$

    If $\nu$ is positive, we have
    \begin{align}
        \Pr_{\nu}(\hat{\pi}= \textsf{None}\text{ or } V^{\hat{\pi}}_0(s_0) < \mu_0)
        \leq \Pr_{\nu}(\hat{\pi}= \textsf{None}) + \Pr_{\nu}(\hat{\pi}\neq \textsf{None}, V^{\hat{\pi}}_0(s_0) < \mu_0).\label{eqn:pos-wrong-output-event}
    \end{align}
    We are going to prove $\Pr_{\nu}(\hat{\pi}= \textsf{None})\leq (1+\frac{\pi^2}{3})\frac{3\delta}{10}$ and $\Pr_{\nu}(\exists k\in\mathbb{N}, \hat{\pi}\neq \textsf{None}, V^{\hat{\pi}}_0(s_0)\leq \frac{\pi^2\delta}{24}$.
    
    By the Line (\ref{alg-line:negative-output}) in Algorithm \ref{alg:1-policy-identification-recycle-history}, we have
    \begin{align}
        & \Pr_{\nu}(\hat{\pi}= \textsf{None})\notag\\
        \leq & \Pr_{\nu}\left(\exists \delta_k\leq\frac{\delta}{10}, \exists t\in\mathbb{N}, \overline{V}^{\overline{\pi}^t(\cdot;\delta_k)}_0(s_0;\delta_k) < \mu_0\right)\notag\\
        \leq & \Pr_{\nu}\left(\exists \delta_k\leq\frac{\delta}{10}, \exists t\in\mathbb{N}, \overline{V}^{\overline{\pi}^t(\cdot;\delta_k)}_0(s_0;\delta_k) < V_0^*(s_0)\right)\notag\\
        \leq & \Pr_{\nu}\left(\exists k\geq \lceil \log_3\frac{10}{\delta}\rceil, \kappa \geq k\right)\label{eqn:kappa-is-larger-than-log_delta/10}\\
        \leq & \frac{1}{3^{\lceil \log_3\frac{10}{\delta}\rceil-1}}.\label{eqn:output-None}\\
        \leq & \frac{3\delta}{10}\label{eqn:pos-prob-output-failure-None}.
    \end{align}
    Step (\ref{eqn:kappa-is-larger-than-log_delta/10}) is by the Proposition \ref{proposition:Vbar-Vunder-bound-V} and Step (\ref{eqn:output-None}) by the Lemma \ref{lemma:kappa-tail-distribution}.

    By the Line (\ref{alg-line:output-a-qualified-policy}) in Algorithm \ref{alg:1-policy-identification-recycle-history}, we have
    \begin{align*}
        & \Pr_{\nu}(\exists k\in\mathbb{N},\hat{\pi}_k\text{ is a policy}, V^{\hat{\pi}_k}_0(s_0) < \mu_0, \hat{\pi}=\hat{\pi}_k)\\
        \leq & \Pr_{\nu}\left(\exists k\in\mathbb{N}, \exists N\in\mathbb{N}, \hat{V}_{\text{et}}^{\hat{\pi}_k,N} - \sqrt{\frac{4H^2 \log\frac{2\alpha_k(\log_2 2N)^2}{\delta}}{N}}\geq \mu_0\right),
    \end{align*}
    where $\hat{V}_{\text{et}}^{\hat{\pi}_k,N} = \frac{\sum_{i=1}^NX_i^{(k)}}{N}$, $\{X_i^{(k)}\}_{i=1}^{+\infty}$ are i.i.d random variables, denoting the rewards collected by executing policy $\hat{\pi}_k$ during a single episode. The mean value for each $X_i^{(k)}$ is $V_0^{\hat{\pi}_k}(s_0)<\mu_0$. Since $X_i^{(k)}$ is bounded by the interval $[0, H]$, we can assert $X_i^{(k)}-V_0^{\hat{\pi}_k}(s_0)$ is $\frac{H^2}{4}$-subgaussin. By the union bound, we have
    \begin{align}
        & \Pr_{\nu}(\exists k\in\mathbb{N}, \hat{\pi}\neq \textsf{None}, V^{\hat{\pi}}_0(s_0) < \mu_0)\notag\\
        \leq & \sum_{k=1}^{+\infty}\Pr_{\nu}\left(\exists N\in\mathbb{N}, \frac{\sum_{i=1}^N X_i^{(k)}}{N} - \sqrt{\frac{H^2 \log\frac{2\alpha_k(\log_2 2N)^2}{\delta}}{N}}\geq \mu_0\right)\notag\\
        \leq & \sum_{k=1}^{+\infty}\Pr_{\nu}\left(\exists N\in\mathbb{N}, \frac{\sum_{i=1}^N X_i^{(k)}}{N} - \sqrt{\frac{H^2 \log\frac{2\alpha_k(\log_2 2N)^2}{\delta}}{N}}\geq V_0^{\hat{\pi}_k}(s_0)\right)\notag\\
        \leq & \sum_{k=1}^{+\infty}\frac{\pi^2}{6}\frac{\delta}{\alpha_k}\notag\\
        \leq & \frac{\pi^2\delta}{6}\frac{1}{4}.\label{eqn:prob-output-unqualified-policy}
    \end{align}
    The second last step is by the Lemma \ref{lemma:concentration-event-for-exploitation} 
    and the last step is by the default setting $\alpha_k=\frac{1}{5^k}$. Substitute (\ref{eqn:prob-output-unqualified-policy}) and (\ref{eqn:pos-prob-output-failure-None}) into (\ref{eqn:pos-wrong-output-event}), we have
    \begin{align*}
        \Pr_{\nu}(\hat{\pi}= \textsf{None}\text{ or } V^{\hat{\pi}}_0(s_0) < \mu_0)\leq \frac{\pi^2\delta}{6}\frac{1}{4}+\frac{3\delta}{10} < \delta.
    \end{align*}

    If $\nu$ is negative, we can follow the calculation of (\ref{eqn:prob-output-unqualified-policy}), and conclude
    \begin{align*}
        \Pr_{\nu}(\hat{\pi}\neq\textsf{None})\leq \frac{\pi^2\delta}{6}\frac{1}{4} <\delta.
    \end{align*}
\end{proof}

Then, we turn to prove the upper bound of $\mathbb{E}\tau$, by applying Algorithm \ref{alg:1-policy-identification-recycle-history}.
\begin{thmstar}[Restatement of Theorem \ref{theorem:upper-bound-of-Etau-recycling-pos}]
    Apply Algorithm \ref{alg:1-policy-identification-recycle-history} to a positive instance $\nu$ with $V^*_0(s_0) > \mu_0$, we have
    \begin{align*}
        \mathbb{E}\tau \leq O\left(\frac{H^2\log\frac{1}{\delta} + (H+1)^4SA\log\frac{SAH}{(V_0^*(s_0)-\mu_0)^2} + (H+1)^4S^2A\log\left(HSA\right)}{(V_0^*(s_0)-\mu_0)^2}\right)
    \end{align*}
\end{thmstar}
\begin{proof}[Proof of Theorem \ref{theorem:upper-bound-of-Etau-recycling-pos}]
    Denote $L^{\text{pos}}=\max\{\kappa,\lceil\log_2 \frac{9600(H+1)^2C^2}{(C-1)^2(V_0^{*}(s_0)-\mu_0)^2}\rceil\}$, and define 
    \begin{align*}
        \kappa_{\text{end}}=\min\left\{k:k\geq L^{\text{pos}}, \mathcal{E}_k^{\text{et}}\text{ holds}\right\}.
    \end{align*}
    By the Lemma \ref{lemma:Termination-Round-of-algorithm}, We know Algorithm \ref{alg:1-policy-identification-recycle-history} will terminate no later than the end of phase $\kappa_{\text{end}}$. 

    Given $\kappa < +\infty$, and index $k\geq 1$, we know
    \begin{align}
        & \Pr(\kappa_{\text{end}}\geq L^{\text{pos}}+k| \kappa)\notag\\
        \leq & \Pr\big((\neg \mathcal{E}_{L^{\text{pos}}}^{\text{et}}) \cap \cdots \cap (\neg \mathcal{E}_{L^{\text{pos}}+k-1}^{\text{et}})| \kappa\big)\notag\\
        \leq & \Pr(\neg \mathcal{E}_{L^{\text{pos}}+k-1}^{\text{et}}| \kappa)\notag\\
        \leq & \frac{\pi^2\delta}{6\alpha_{L^{\text{pos}}+k-1}}\label{eqn:kappa-end-tail-prob-L_pos-k}
    \end{align}
    The last step is by the Lemma \ref{lemma:concentration-event-for-exploitation}.

    Denote $\tau^{\text{ee}}$, $\tau^{\text{et}}$ as the total collected episodes by the exploration and exploitation oracle. Theorem \ref{theorem:performance-guarantee-of-exploration-oracle} suggests that
    \begin{align*}
        \tau^{\text{ee}} \leq & \frac{101(H+1)^4SA\log\frac{2SAH}{\delta_{\kappa_{\text{end}}}}}{{\epsilon^{\text{pos}}}^2}+ T_{\kappa-1}^{\text{ee}} +\\
        & \frac{400(H+1)^4S^2A}{{\epsilon^{\text{pos}}}^2}\log\left(\frac{402(H+1)^4S^2A \log\frac{2SAH}{\delta_{\kappa_{\text{end}}}}}{{\epsilon^{\text{pos}}}^2}\right)+ \frac{400(H+1)^4S^2A}{{\epsilon^{\text{pos}}}^2}\log\left(2T_{\kappa-1}^{\text{ee}}\right)
        \\
        \tau^{\text{et}}\leq & \sum_{k=1}^{\kappa_{\text{end}}} 100\cdot \frac{\log\frac{\alpha_k}{\delta} + \log\log \frac{24H^2}{\epsilon_k}}{\epsilon_k}
    \end{align*}
    where $\epsilon^{\text{pos}}=\frac{V_0^*(s_0)-\mu_0}{C}$. For simplicity, we denote $V_0^*:=V_0^*(s_0)$ in the following proof. We first derive an upper bound for $\mathbb{E}\tau^{\text{ee}}$. By (\ref{eqn:kappa-end-tail-prob-L_pos-k}),
    we know $\Pr(\kappa_{\text{end}}-L^{\text{pos}}+k| \kappa) \leq \frac{\pi^2\delta}{6\alpha_{L^{\text{pos}}+k-1}}, k\geq 1$, thus, we can conclude
    \begin{align*}
        & \mathbb{E}[\tau^{\text{ee}} | \kappa]\\
        \leq & \frac{101(H+1)^4SA\log\frac{2SAH}{\delta_{L^{\text{pos}}}}}{{\epsilon^{\text{pos}}}^2} + \sum_{k=1}^{+\infty
        }\frac{\pi^2}{6}\frac{\delta}{5^{k-1}}\frac{101(H+1)^4SA\log\frac{2SAH}{\delta_{L^{\text{pos}} + k}}}{{\epsilon^{\text{pos}}}^2}\\
        & T_{\kappa-1}^{\text{ee}} + \frac{400(H+1)^4S^2A}{{\epsilon^{\text{pos}}}^2}\log\left(\frac{402(H+1)^4S^2A \log\frac{2SAH}{\delta_{L^{\text{pos}} }}}{{\epsilon^{\text{pos}}}^2}\right)+ \\
        & \sum_{k=1}^{+\infty}\frac{\pi^2}{6}\frac{\delta}{5^{k-1}}\frac{400(H+1)^4S^2A}{{\epsilon^{\text{pos}}}^2}\log\left(\frac{402(H+1)^4S^2A \log\frac{2SAH}{\delta_{L^{\text{pos}} + k}}}{{\epsilon^{\text{pos}}}^2}\right) + \\
        & \frac{400(H+1)^4S^2A}{{\epsilon^{\text{pos}}}^2}\log\left(2T_{\kappa-1}^{\text{ee}}\right)\\
        \leq & O(1)\underbrace{\frac{(H+1)^4SA\log\frac{2SAH}{\delta_{L^{\text{pos}}}}}{{\epsilon^{\text{pos}}}^2}}_{\dagger1} + O(1)\underbrace{\frac{(H+1)^4S^2A}{{\epsilon^{\text{pos}}}^2}\log\left(\frac{(H+1)^4S^2A \log\frac{2SAH}{\delta_{L^{\text{pos}} }}}{{\epsilon^{\text{pos}}}^2}\right)}_{\dagger2}  +\\
        & \underbrace{\frac{400(H+1)^4S^2A}{{\epsilon^{\text{pos}}}^2}\log\left(2T_{\kappa-1}^{\text{ee}}\right) + T_{\kappa-1}^{\text{ee}}}_{\dagger3}
    \end{align*}
    To move forward, we take expectation on $\dagger1$, $\dagger2$ and $\dagger3$. By the Lemma \ref{lemma:prob-good-event}, we have
    \begin{align}
        & \mathbb{E} \frac{(H+1)^4SA\log\frac{2SAH}{\delta_{L^{\text{pos}}}}}{{\epsilon^{\text{pos}}}^2}\notag\\
        \leq & \mathbb{E}\left(\frac{(H+1)^4SA\log\frac{2SAH}{\delta_{\kappa}}}{{\epsilon^{\text{pos}}}^2} + \frac{(H+1)^4SA\log\frac{2SAH}{\delta_{\lceil\log_2 \frac{9600H^2C^2}{(C-1)^2(V_0^{*}(s_0)-\mu_0)^2}\rceil}}}{{\epsilon^{\text{pos}}}^2}\right)\notag\\
        \leq & \sum_{k=1}^{+\infty}\frac{1}{3^{k-1}}\cdot\frac{(H+1)^4SA\log(2SAH \cdot 3^{k})}{{\epsilon^{\text{pos}}}^2}+
        O\left( \frac{(H+1)^4SA\log\frac{SAH C^2}{(C-1)^2(V_0^*(s_0)-\mu_0)^2}}{(V_0^*-\mu_0)^2}\right)\notag\\
        \leq & O\left( \frac{(H+1)^4SA\log\frac{SAH C^2}{(C-1)^2(V_0^*(s_0)-\mu_0)^2}}{(V_0^*(s_0)-\mu_0)^2}\right)\label{eqn:pos-tau_ee-Term1-upper}.
    \end{align}
    \begin{align}
        & \mathbb{E}\frac{(H+1)^4S^2A}{{\epsilon^{\text{pos}}}^2}\log\left(\frac{(H+1)^4S^2A \log\frac{2SAH}{\delta_{L^{\text{pos}} }}}{{\epsilon^{\text{pos}}}^2}\right)\notag\\
        \leq & \frac{(H+1)^4S^2A}{{\epsilon^{\text{pos}}}^2}\log\left(\frac{(H+1)^4S^2A }{{\epsilon^{\text{pos}}}^2}\right) + \mathbb{E}\frac{(H+1)^4S^2A\log \log\frac{2SAH}{\delta_{L^{\text{pos}} }}}{{\epsilon^{\text{pos}}}^2}\notag\\
        \leq & \frac{(H+1)^4S^2A}{{\epsilon^{\text{pos}}}^2}\log\left(\frac{(H+1)^4S^2A }{{\epsilon^{\text{pos}}}^2}\right) + \sum_{k=1}^{+\infty}\frac{1}{3^{k-1}}\frac{(H+1)^4S^2A\log \log\frac{2SAH}{\delta_{k }}}{{\epsilon^{\text{pos}}}^2} + \notag\\
        & \frac{(H+1)^4S^2A\log \log\frac{2SAH}{\delta_{ \lceil\log_2 \frac{9600(H+1)^2C^2}{(C-1)^2(V_0^{*}(s_0)-\mu_0)^2}\rceil }}}{{\epsilon^{\text{pos}}}^2}\notag\\
        \leq & O\left(\frac{C^2(H+1)^4S^2A}{(V_0^*(s_0)-\mu_0)^2}\log\left(\frac{C^2(H+1)^4S^2A \log\frac{C^2(H+1)^2}{(C-1)^2(V_0^{*}(s_0)-\mu_0)^2} }{(V_0^*(s_0)-\mu_0)^2}\right)\right)\label{eqn:pos-tau_ee-Term2-upper}
    \end{align}
    \begin{align}
        & \mathbb{E}\frac{400(H+1)^4S^2A}{{\epsilon^{\text{pos}}}^2}\log\left(2T_{\kappa-1}^{\text{ee}}\right) + T_{\kappa-1}^{\text{ee}}\notag\\
        \leq & \sum_{k=1}^{+\infty}\frac{1}{3^{k-1}}\left(\frac{(H+1)^2SA\log\frac{2SAH}{\delta_k}}{{\epsilon_k}} + \frac{(H+1)^2S^2A}{\epsilon_k}\log\left(\frac{(H+1)^2S^2A \log\frac{2SAH}{\delta_k}}{\epsilon_k}\right)\right)+\notag\\
        & \frac{400(H+1)^4S^2A\log 2}{{\epsilon^{\text{pos}}}^2} + \frac{400(H+1)^4S^2A}{{\epsilon^{\text{pos}}}^2} \cdot\notag\\
        & \sum_{k=1}^{+\infty}\frac{1}{3^{k-1}}\log\left(\frac{(H+1)^2SA\log\frac{2SAH}{\delta_k}}{{\epsilon_k}} + \frac{(H+1)^2S^2A}{\epsilon_k}\log\left(\frac{(H+1)^2S^2A \log\frac{2SAH}{\delta_k}}{\epsilon_k}\right)\right)\notag\\
        \leq & O\Big(\frac{C^2(H+1)^4S^2A \log (SAH)}{(V_0^*(s_0)-\mu_0)^2} + (H+1)^2S^2A\log(SAH) \Big)\notag\\
        \leq & O\Big(\frac{C^2(H+1)^4S^2A \log (SAH)}{(V_0^*(s_0)-\mu_0)^2}\Big).\label{eqn:pos-tau_ee-Term3-upper}
    \end{align}
    The last step is by the fact that $(V_0^*(s_0)-\mu_0)^2\leq H^2$.
    Sum up (\ref{eqn:pos-tau_ee-Term1-upper}), (\ref{eqn:pos-tau_ee-Term2-upper}) (\ref{eqn:pos-tau_ee-Term3-upper}) and plug in $C=1.01$, we can conclude
    \begin{align}
        & \mathbb{E}\tau^{\text{ee}}\notag\\
        = & \mathbb{E}\mathbb{E}[\tau^{\text{ee}} | \kappa]\notag\\
        \leq & O\left(\frac{(H+1)^4SA\log\frac{SAH}{(V_0^*(s_0)-\mu_0)^2}}{(V_0^*(s_0)-\mu_0)^2}+\frac{(H+1)^4S^2A\log\left(HSA\right)}{(V_0^*(s_0)-\mu_0)^2}\right)\label{eqn:pos-tau-ee-upper}
    \end{align}

    The remaining work is to derive an upper bound for $\mathbb{E}\tau^{\text{et}}$, we also first derive an upper bound for $\mathbb{E}[\tau^{\text{et}}|\kappa]$, which is 
    \begin{align}
        \mathbb{E}[\tau^{\text{et}}|\kappa]\leq & \mathbb{E}\left[\sum_{k=1}^{\kappa_{\text{end}}} 100\cdot \frac{\log\frac{\alpha_k}{\delta} + \log\log \frac{24H^2}{\epsilon_k}}{\epsilon_k} | \kappa\right]\notag\\
        \leq & \mathbb{E}\left[200\cdot \frac{\log\frac{\alpha_{\kappa_{\text{end}}}}{\delta} + \log\log \frac{24H^2}{\epsilon_{\kappa_{\text{end}}}}}{\epsilon_{\kappa_{\text{end}}}}| \kappa\right]\notag\\
        = & \mathbb{E}\left[200\cdot \frac{\log\frac{\alpha_{\kappa_{\text{end}}-L^{\text{pos}} } \cdot\alpha_{L^{\text{pos}}} }{\delta} + \log\log \frac{24H^2}{\epsilon_{\kappa_{\text{end}} -L^{\text{pos}} }\cdot \epsilon_{L^{\text{pos}}} } }{\epsilon_{\kappa_{\text{end}} -L^{\text{pos}} } \cdot \epsilon_{L^{\text{pos}}} }| \kappa\right]\notag\\
        \leq & \frac{200\log\frac{ \alpha_{L^{\text{pos}}} }{\delta} + 200\log\log \frac{24H^2}{\epsilon_{L^{\text{pos}}} } }{ \epsilon_{L^{\text{pos}}} } + \sum_{k=1}^{+\infty}\frac{200\pi^2\delta}{6\alpha_{k-1}} \cdot \frac{\log\frac{\alpha_{k } \cdot\alpha_{L^{\text{pos}}} }{\delta} + \log\log \frac{24H^2}{\epsilon_{k }\cdot \epsilon_{L^{\text{pos}}} } }{\epsilon_{k } \cdot \epsilon_{L^{\text{pos}}} }\label{eqn:pos-tauet-tail-dist-kappaend-L}\\
        = & \frac{200\log\frac{ \alpha_{L^{\text{pos}}} }{\delta} + 200\log\log \frac{24H^2}{\epsilon_{L^{\text{pos}}} } }{ \epsilon_{L^{\text{pos}}} } + \sum_{k=1}^{+\infty}\frac{200\pi^2\delta\cdot 2^k}{6\cdot 5^{k-1}} \cdot \frac{k\log 5+\log\frac{\alpha_{L^{\text{pos}}} }{\delta} + \log\log \frac{24H^2\cdot 2^k}{ \epsilon_{L^{\text{pos}}} } }{ \epsilon_{L^{\text{pos}}} }\notag\\
        \leq & O\left(\frac{\log\frac{ \alpha_{L^{\text{pos}}} }{\delta} + \log\log \frac{24H^2}{\epsilon_{L^{\text{pos}}} } }{ \epsilon_{L^{\text{pos}}} }\right).\notag
    \end{align}
    Step (\ref{eqn:pos-tauet-tail-dist-kappaend-L}) is from the (\ref{eqn:kappa-end-tail-prob-L_pos-k}). To derive an upper bound for $\mathbb{E}\tau^{\text{et}}$, we can conduct the following calculation
    \begin{align*}
        & \mathbb{E}\tau^{\text{et}}\\
        \leq & \mathbb{E}\mathbb{E}[\tau^{\text{et}}|\kappa]\\
        \leq & O\left(\mathbb{E}\frac{\log\frac{ \alpha_{L^{\text{pos}}} }{\delta} + \log\log \frac{H^2}{\epsilon_{L^{\text{pos}}} } }{ \epsilon_{L^{\text{pos}}} }\right)\\
        \leq & O\left(\mathbb{E}\frac{\log\frac{ \alpha_{\kappa} }{\delta} + \log\log \frac{24H^2}{\epsilon_{\kappa} } }{ \epsilon_{\kappa} }\right) + O\left(\frac{\log\frac{ \alpha_{\lceil\log_2 \frac{9600(H+1)^2C^2}{(C-1)^2(V_0^*(s_0)-\mu_0)^2}\rceil} }{\delta} + \log\log \frac{24H^2}{\epsilon_{\lceil\log_2 \frac{9600(H+1)^2C^2}{(C-1)^2(V_0^*(s_0)-\mu_0)^2}\rceil} } }{ \epsilon_{\lceil\log_2 \frac{9600(H+1)^2C^2}{(C-1)^2(V_0^*(s_0)-\mu_0)^2}\rceil} }\right)\\
        \leq & O\left(\log\frac{\log H}{\delta} + \sum_{k=1}^{+\infty}\delta_{k-1}\frac{\log\frac{ \alpha_{k} }{\delta} + \log\log \frac{24H^2}{\epsilon_{k} } }{ \epsilon_{k} }\right)+\\
        & O\left(\frac{(H+1)^2C^2}{(C-1)^2}\cdot \frac{\log\frac{1}{\delta} + \log\frac{(H+1)^2C^2}{(C-1)^2(V_0^*(s_0)-\mu_0)^2} }{(V_0^*(s_0)-\mu_0)^2}\right)\\
        \leq & O\left(\frac{(H+1)^2C^2}{(C-1)^2}\cdot \frac{\log\frac{1}{\delta} + \log\frac{(H+1)^2C^2}{(C-1)^2(V_0^{*}(s_0)-\mu_0)^2} }{(V_0^{*}-\mu_0)^2} + \log\frac{\log H}{\delta}\right).
    \end{align*}
    The second last step is by the Lemma \ref{lemma:prob-good-event}. Plug in $C=1.01$, we can derive
    \begin{align}
        \mathbb{E}\tau^{\text{et}}\leq O\left(\frac{H^2\log\frac{1}{\delta} + H^2\log\frac{H^2}{(V_0^*(s_0)-\mu_0)^2} }{(V_0^{*}(s_0)-\mu_0)^2} + \log\frac{\log H}{\delta}\right)\label{eqn:pos-tau-et-upper}
    \end{align}
    Summing up (\ref{eqn:pos-tau-ee-upper}) and (\ref{eqn:pos-tau-et-upper}), by the fact that $(V_0^*(s_0)-\mu_0)^2\leq H^2$, we have
    \begin{align*}
        \mathbb{E}\tau \leq O\left(\frac{H^2\log\frac{1}{\delta} + (H+1)^4SA\log\frac{SAH}{(V_0^*(s_0)-\mu_0)^2} + (H+1)^4S^2A\log\left(HSA\right)}{(V_0^*(s_0)-\mu_0)^2}\right).
    \end{align*}
\end{proof}

\begin{thmstar}[Restatement of Theorem \ref{theorem:upper-bound-of-Etau-recycling-neg}]
    Apply algorithm \ref{alg:1-policy-identification-recycle-history} to a negative instance $\nu$ with $V^*_0(s_0) < \mu_0$, we have
    \begin{align*}
        & \mathbb{E}\tau\\
        \leq & O\left( \frac{(H+1)^4SA(\log\frac{1}{\delta}+\log\frac{H^2}{(V^{*}_0(s_0)-\mu_0)^2})}{(V^{*}_0(s_0)-\mu_0)^2}\right) +\\
        & O\left(\frac{(H+1)^4S^2A\log\left(\log(SAH)+\log\frac{1}{\delta} + \log\frac{H^2}{(V^{*}_0(s_0)-\mu_0)^2}
        \right)}{(V^{*}_0(s_0)-\mu_0)^2}\right).
    \end{align*}
\end{thmstar}
\begin{proof}[Proof of Theorem \ref{theorem:upper-bound-of-Etau-recycling-neg}]
    Denote $\tau^{\text{ee}}$, $\tau^{\text{et}}$ as the total collected episodes by the exploration and exploitation oracle. By the Lemma \ref{lemma:Termination-Round-of-algorithm}, we know Algorithm \ref{alg:1-policy-identification-recycle-history} will terminate no later than the end of phase $\max\{\lceil \log_2 \frac{9600(H+1)^2}{(V^{*}_0(s_0)-\mu_0)^2}\rceil, \lceil\log_3\frac{10}{\delta}\rceil, \kappa\}=:L^{\text{neg}}$. The remaining work is to prove upper bounds for $\tau^{\text{ee}}$ and $\tau^{\text{et}}$, at the end of phase $L^{\text{neg}}$.

    For $\tau^{\text{et}}$, by the Theorem \ref{theorem:performance-guarantee-of-exploration-oracle}, we can conclude at the end of phase index $k\geq \kappa$, the Exploration Oracle(Line \ref{alg-line:call-exploration-oracle} in Algorithm \ref{alg:1-policy-identification-recycle-history}) will always output \textsf{None} or \textsf{Not Completed}.
    This observation implies Algorithm \ref{alg:1-policy-identification-recycle-history} will not enter the exploitation period at Line \ref{alg-line:Exploitation-starts} for phase index $k\geq \kappa$, suggesting that
    \begin{align}
        \tau^{\text{et}}\leq \sum_{k=1}^{\kappa-1} 100\cdot \frac{\log\frac{\alpha_k}{\delta} + \log\log \frac{24H^2}{\epsilon_k}}{\epsilon_k} \label{eqn:tau-et-upper-bound-neg}.
    \end{align}

    For $\tau^{\text{ee}}$, by the Theorem \ref{theorem:performance-guarantee-of-exploration-oracle}, we have
    \begin{equation}\label{eqn:tau-ee-upper-bound-neg}
        \begin{split}
            \tau^{\text{ee}} \leq & \frac{101(H+1)^4SA\log\frac{2SAH}{\delta_{L^{\text{neg}}}}}{{\epsilon^{\text{neg}}}^2} + T_{\kappa-1}^{\text{ee}} +\\
            & \frac{400(H+1)^4S^2A}{{\epsilon^{\text{neg}}}^2}\log\left(\frac{402(H+1)^4S^2A \log\frac{2SAH}{\delta_{L^{\text{neg}}}}}{{\epsilon^{\text{neg}}}^2}\right)+ \frac{400(H+1)^4S^2A}{{\epsilon^{\text{neg}}}^2}\log\left(2T_{\kappa-1}^{\text{ee}}\right).
        \end{split}
    \end{equation}
    The remaining is to take expectation on both sides of (\ref{eqn:tau-et-upper-bound-neg}), (\ref{eqn:tau-ee-upper-bound-neg}) and derive upper bounds. For (\ref{eqn:tau-et-upper-bound-neg}), by the Lemma \ref{lemma:length-of-exploration-exploitation-period}, we have
    \begin{align}
        \mathbb{E}\tau^{\text{et}}\leq & \mathbb{E}\sum_{k=1}^{\kappa-1} 100\cdot \frac{\log\frac{\alpha_k}{\delta} + \log\log \frac{24H^2}{\epsilon_k}}{\epsilon_k}  \notag\\
        \leq & \mathbb{E}200\cdot \frac{\log\frac{\alpha_{\kappa-1}}{\delta} + \log\log \frac{24H^2}{\epsilon_{\kappa-1}}}{\epsilon_{\kappa-1}}\notag\\
        \leq & \sum_{k=1}^{+\infty}\frac{200}{3^{k-1}}\frac{\log\frac{\alpha_{k-1}}{\delta} + \log\log \frac{24H^2}{\epsilon_{k-1}}}{\epsilon_{k-1}}\notag\\
        \leq & O\left(\log\frac{\log H}{\delta}\right).\label{eqn:upper-Etau-et-neg}
    \end{align}
    The second step is by taking $\epsilon_k=\frac{1}{2^k}$ and $\Big(2\cdot \frac{\log\frac{\alpha_{k+1}}{\delta} + \log\log \frac{24H^2}{\epsilon_{k+1}}}{\epsilon_{k+1}}\Big) / \Big(2\cdot \frac{\log\frac{\alpha_k}{\delta} + \log\log \frac{24H^2}{\epsilon_k}}{\epsilon_k}\Big)\geq 2$. The second last step is by the Lemma \ref{lemma:kappa-tail-distribution}.

    For (\ref{eqn:tau-ee-upper-bound-neg}), we have
    \begin{align}
        & \mathbb{E}\tau^{\text{ee}}\notag\\
        \leq & \underbrace{\mathbb{E}\left(\frac{101(H+1)^4SA\log (\frac{1}{\delta_{L^{\text{neg}}}})}{{\epsilon^{\text{neg}}}^2} + \frac{400(H+1)^4S^2A}{{\epsilon^{\text{neg}}}^2}\log\left(\log\frac{2SAH}{\delta_{L^{\text{neg}}}}\right)\right)}_{\dagger 4} + \notag\\
        & \underbrace{\mathbb{E}\left(T_{\kappa-1}^{\text{ee}} + \frac{400(H+1)^4S^2A}{{\epsilon^{\text{neg}}}^2}\log\left(2T_{\kappa-1}^{\text{ee}}\right)\right)}_{\dagger 5} + \notag\\
        & \frac{101(H+1)^4SA\log(2SAH)}{{\epsilon^{\text{neg}}}^2} + \frac{400(H+1)^4S^2A}{{\epsilon^{\text{neg}}}^2}\log\left(\frac{402(H+1)^4S^2A }{{\epsilon^{\text{neg}}}^2}\right)\label{eqn:neg-over-tau-ee-upper}
    \end{align}
    For $\dagger 4$, 
    \begin{align}
        & \mathbb{E}\frac{101(H+1)^4SA\log (\frac{1}{\delta_{L^{\text{neg}}}})}{{\epsilon^{\text{neg}}}^2} + \frac{400(H+1)^4S^2A}{{\epsilon^{\text{neg}}}^2}\log\left(\log\frac{2SAH}{\delta_{L^{\text{neg}}}}\right)\notag\\
        \leq & \frac{101(H+1)^4SA\log (\frac{1}{\delta_{\lceil \log_2 \frac{9600(H+1)^2}{(V^{*}_0(s_0)-\mu_0)^2}\rceil}})}{{\epsilon^{\text{neg}}}^2} + \frac{400(H+1)^4S^2A}{{\epsilon^{\text{neg}}}^2}\log\left(\log\frac{2SAH}{\delta_{\lceil \log_2 \frac{9600(H+1)^2}{(V^{*}_0(s_0)-\mu_0)^2}\rceil}}\right) + \notag\\
        & \frac{101(H+1)^4SA\log (\frac{1}{\delta_{\lceil\log_3 \frac{10}{\delta}\rceil }})}{{\epsilon^{\text{neg}}}^2} + \frac{400(H+1)^4S^2A}{{\epsilon^{\text{neg}}}^2}\log\left(\log\frac{2SAH}{\delta_{\lceil\log_3 \frac{10}{\delta}\rceil }}\right) +\notag\\
        & \mathbb{E}\frac{101(H+1)^4SA\log (\frac{1}{\delta_{\kappa }})}{{\epsilon^{\text{neg}}}^2} + \frac{400(H+1)^4S^2A}{{\epsilon^{\text{neg}}}^2}\log\left(\log\frac{2SAH}{\delta_{\kappa}}\right).\label{eqn:expected-kappa-in-Etau-ee-neg}
    \end{align}
    By Lemma \ref{lemma:kappa-tail-distribution}, we can derive
    \begin{align*}
        & \mathbb{E}\frac{101(H+1)^4SA\log (\frac{1}{\delta_{\kappa }})}{{\epsilon^{\text{neg}}}^2} + \frac{400(H+1)^4S^2A}{{\epsilon^{\text{neg}}}^2}\log\left(\log\frac{2SAH}{\delta_{\kappa}}\right)\\
        \leq & \sum_{k=1}^{+\infty}\delta_{k-1}\left(\frac{101(H+1)^4SA\log (\frac{1}{\delta_{k }})}{{\epsilon^{\text{neg}}}^2} + \frac{400(H+1)^4S^2A}{{\epsilon^{\text{neg}}}^2}\log\left(\log\frac{2SAH}{\delta_{k}}\right)\right)\\
        \leq & O\left( \frac{(H+1)^4S^2A\log\left(\log(SAH)\right)}{{\epsilon^{\text{neg}}}^2}\right).
    \end{align*}
    Following (\ref{eqn:expected-kappa-in-Etau-ee-neg}), we can derive
    \begin{align}
        & \dagger 4\notag\\
        \leq & O\left( \frac{(H+1)^4SA\log\frac{1}{\delta}}{(V^{*}_0(s_0)-\mu_0)^2} + \frac{(H+1)^4S^2A\log\left(\log(SAH)+\log\frac{1}{\delta} + \lceil\log\frac{9600(H+1)^2}{(V^{*}_0(s_0)-\mu_0)^2}\rceil
        \right)}{(V^{*}_0(s_0)-\mu_0)^2}\right) + \notag\\
        & O\left(\frac{(H+1)^4SA\lceil\log\frac{9600(H+1)^2}{(V^{*}_0(s_0)-\mu_0)^2}\rceil}{(V^{*}_0(s_0)-\mu_0)^2}\right)\label{eqn:expected-kappa-in-Etau-ee-neg-upper-bound}
    \end{align}

    For $\dagger 5$, by Lemma \ref{lemma:kappa-tail-distribution}
    \begin{align}
        & \mathbb{E}\left(T_{\kappa-1}^{\text{ee}} + \frac{400(H+1)^4S^2A}{{\epsilon^{\text{neg}}}^2}\log\left(2T_{\kappa-1}^{\text{ee}}\right)\right)\notag\\
        \leq & \sum_{k=1}^{+\infty}\delta_{k-1}\left(\frac{(H+1)^2SA\log\frac{2SAH}{\delta_k}}{{\epsilon_k}} + \frac{(H+1)^2S^2A}{\epsilon_k}\log\left(\frac{(H+1)^2S^2A \log\frac{2SAH}{\delta_k}}{\epsilon_k}\right)\right) + \notag\\
        & \frac{400(H+1)^4S^2A \log 2}{{\epsilon^{\text{neg}}}^2} + \frac{400(H+1)^4S^2A}{{\epsilon^{\text{neg}}}^2}\cdot \notag\\
        & \sum_{k=1}^{+\infty}\delta_{k-1}\log\left(\frac{(H+1)^2SA\log\frac{2SAH}{\delta_k}}{{\epsilon_k}} + \frac{(H+1)^2S^2A}{\epsilon_k}\log\left(\frac{(H+1)^2S^2A \log\frac{2SAH}{\delta_k}}{\epsilon_k}\right)\right)\notag\\
        \leq & O\left(\frac{(H+1)^4S^2A \log (SAH)}{(V^{*}_0(s_0)-\mu_0)^2}+(H+1)^2S^2A\log(SAH)\right)\notag\\
        \leq & O\left(\frac{(H+1)^4S^2A \log (SAH)}{(V^{*}_0(s_0)-\mu_0)^2}\right).\label{eqn:upper-Etau-ee-neg-dagger5}
    \end{align}
    The last step is by the fact that $(V^{*}_0(s_0)-\mu_0)^2\leq H^2$.

    Substitute $\dagger 4$, $\dagger 5$ in (\ref{eqn:neg-over-tau-ee-upper}) with (\ref{eqn:expected-kappa-in-Etau-ee-neg-upper-bound}) and (\ref{eqn:upper-Etau-ee-neg-dagger5}), we prove
    \begin{align}
        \mathbb{E}\tau^{\text{ee}} \leq & O\left( \frac{(H+1)^4SA(\log\frac{1}{\delta}+\log\frac{H^2}{(V^{*}_0(s_0)-\mu_0)^2})}{(V^{*}_0(s_0)-\mu_0)^2}\right) + \notag \\
        & O\left(\frac{(H+1)^4S^2A\log\left(\log(SAH)+\log\frac{1}{\delta} + \log\frac{H^2}{(V^{*}_0(s_0)-\mu_0)^2}
        \right)}{(V^{*}_0(s_0)-\mu_0)^2}\right).\label{eqn:neg-Etau-ee-final-upper-bound}
    \end{align}
    Combine (\ref{eqn:upper-Etau-et-neg}) and (\ref{eqn:neg-Etau-ee-final-upper-bound}), we have
    \begin{align*}
        \mathbb{E}\tau\leq & O\left( \frac{(H+1)^4SA(\log\frac{1}{\delta}+\log\frac{H^2}{(V^{*}_0(s_0)-\mu_0)^2})}{(V^{*}_0(s_0)-\mu_0)^2}\right) + \notag \\
        & O\left(\frac{(H+1)^4S^2A\log\left(\log(SAH)+\log\frac{1}{\delta} + \log\frac{H^2}{(V^{*}_0(s_0)-\mu_0)^2}
        \right)}{(V^{*}_0(s_0)-\mu_0)^2}\right).
    \end{align*}
\end{proof}

\subsection{Auxiliary Lemmas}
\label{sec:Auxiliary-Lemma-for-Main-Theorem}

In this section, we provide the proof of the required Lemmas in section \ref{sec:proof-main-theorem}. The first one is about the maximum episodes collected in exploration and exploitation periods, up to the end of phase $k$.
\begin{lemma}
    \label{lemma:length-of-exploration-exploitation-period}
    Each time Algorithm \ref{alg:1-policy-identification-recycle-history} enters Line \ref{alg-line:end-of-phase-k}, i.e. the end of phase $k$, we have $|\mathcal{H}^{\text{ee}}|\leq T_k^{\text{ee}},|\mathcal{H}^{\text{et}}|\leq \sum_{\ell=1}^{k} 100\cdot \frac{\log\frac{\alpha_\ell}{\delta} + \log\log \frac{24H^2}{\epsilon_\ell}}{\epsilon_\ell} $ hold with certainty.
\end{lemma}
\begin{proof}[Proof of Lemma \ref{lemma:length-of-exploration-exploitation-period}]
    $|\mathcal{H}^{\text{ee}}|\leq T_k^{\text{ee}}$ is a direct conclusion from the condition $t\leq T_k^{\text{ee}}-1$ at the Line \ref{line-alg:start-execute-MDP-with-new-delta_k} of Algorithm \ref{alg:Oracle-BPI_UCRL}.

    Regarding $|\mathcal{H}^{\text{et}}|$, its upper bound is from the condition $N\leq 100\cdot \frac{\log\frac{\alpha_k}{\delta} + \log\log \frac{24H^2}{\epsilon_k}}{\epsilon_k}$ at the Line \ref{alg-line:maximum-episodes-in-exploitation-period} of Algorithm \ref{alg:1-policy-identification-recycle-history}. Since exploitation period $k$ allows at most collecting $100\cdot \frac{\log\frac{\alpha_k}{\delta} + \log\log \frac{24H^2}{\epsilon_k}}{\epsilon_k}$ episodes, $|\mathcal{H}^{\text{et}}|\leq \sum_{\ell=1}^{k} 100\cdot \frac{\log\frac{\alpha_\ell}{\delta} + \log\log \frac{24H^2}{\epsilon_\ell}}{\epsilon_\ell}$ must hold.
\end{proof}

The following is a basis Lemma, asserting the Algorithm must terminate no later than the end of specific phases. 
\begin{lemma}
    \label{lemma:Termination-Round-of-algorithm}
    Recall the definition $\mathcal{E}_k^{\text{et}} = \left\{\forall N>0, \left|\hat{V}_{\text{et}}^{\hat{\pi}_k,N} - V^{\hat{\pi}_k}_0\right|\leq \sqrt{\frac{4H^2 \log\frac{2\alpha_k(\log_2 2N)^2}{\delta}}{N}} \right\}$.
    
    Apply Algorithm \ref{alg:1-policy-identification-recycle-history} to an instance $\nu$
    \begin{itemize}
        \item If the instance $\nu$ is positive, denote $L^{\text{pos}}=\max\{\kappa,\lceil\log_2 \frac{9600(H+1)^2C^2}{(C-1)^2(V_0^{*}(s_0)-\mu_0)^2}\rceil\}$, and define 
        \begin{align*}
            \kappa_{\text{end}}=\min\left\{k:k\geq L^{\text{pos}}, \mathcal{E}_k^{\text{et}}\text{ holds}\right\}.
        \end{align*}
        We assert the Algorithm \ref{alg:1-policy-identification-recycle-history} must terminate no later than the end of phase $\kappa_{\text{end}}$.
        \item If the instance $\nu$ is negative, denote $L^{\text{neg}}:=\max\{\lceil \log_2 \frac{9600(H+1)^2}{(V^{*}_0(s_0)-\mu_0)^2}\rceil, \lceil\log_3\frac{10}{\delta}\rceil, \kappa\}$.We assert the Algorithm \ref{alg:1-policy-identification-recycle-history} must terminate no later than the end of phase $L^{\text{neg}}$.
    \end{itemize}
\end{lemma}
\begin{proof}[Proof of Lemma \ref{lemma:Termination-Round-of-algorithm}]
    We split the proof into two parts by discussing $\nu$ is is either positive or negative.

    If $\nu$ is positive, 
    by the Theorem \ref{theorem:performance-guarantee-of-exploration-oracle}, we can conclude at the end of phase index $k\geq \max\{\lceil \log_2 \frac{9600(H+1)^2C^2}{(V^{*}_0(s_0)-\mu_0)^2}\rceil, \kappa\}$, the Exploration Oracle(Line \ref{alg-line:call-exploration-oracle} in Algorithm \ref{alg:1-policy-identification-recycle-history}) will always output $\hat{\pi}_k$ such that $V^{\hat{\pi}_k}_0(s_0)-\mu_0 \geq \frac{C-1}{C}\Big(V^*_0(s_0)-\mu_0\Big)$. 
    
    By the Lemma \ref{lemma:scale-of-k-exploit} and \ref{lemma:correctness-length-of-exploitation}, for any phase index $k\geq \lceil\log_2 \frac{H^2C^2}{(C-1)^2(V_0^{*}(s_0)-\mu_0)^2}\rceil$, if the Exploration Oracle returns a $\hat{\pi}_k$ satisfying that $V_0^{\hat{\pi}_k}(s_0)-\mu_0 \geq \frac{C-1}{C}(V_0^*-\mu_0)$ and the concentration event $\mathcal{E}_k^{\text{et}}$ holds,  Algorithm \ref{alg:1-policy-identification-recycle-history} will enter Line \ref{alg-line:output-policy} and take $\hat{\pi}_k$ as the final output. Recall the definition $\kappa_{\text{end}}=\min\left\{k:k\geq L^{\text{pos}}, \mathcal{E}_k^{\text{et}}\text{ holds}\right\}$, we know Algorithm \ref{alg:1-policy-identification-recycle-history} will terminate at the end of phase $\kappa_{\text{end}}$.

    If $\nu$ is negative, 
    by the Theorem \ref{theorem:performance-guarantee-of-exploration-oracle}, we can conclude at the end of phase index $k\geq \max\{\lceil \log_2 \frac{9600(H+1)^2}{(V^{*}_0(s_0)-\mu_0)^2}\rceil, \kappa\}$, the Exploration Oracle(Line \ref{alg-line:call-exploration-oracle} in Algorithm \ref{alg:1-policy-identification-recycle-history}) will always output \textsf{None}. Because of the condition of Line \ref{alg-line:negative-output} in Algorithm \ref{alg:1-policy-identification-recycle-history}, we can assert the algorithm will terminate no later than the end of phase $\max\{\lceil \log_2 \frac{9600(H+1)^2}{(V^{*}_0(s_0)-\mu_0)^2}\rceil, \lceil\log_3\frac{10}{\delta}\rceil, \kappa\}=:L^{\text{neg}}$.
\end{proof}

The following Lemma is a direct conclusion from Lemma \ref{lemma:Termination-Round-of-algorithm}, which guarantee Algorithm \ref{alg:1-policy-identification-recycle-history} must terminate almost surely.
\begin{lemma}
    \label{lemma:Must-Terminate-with-Prob1}
    Apply Algorithm \ref{alg:1-policy-identification-recycle-history} to an instance $\nu$. If $V^*_0(s_0)\neq \mu_0$, which means $\nu$ is either positive or negative, we have $\Pr(\tau < +\infty) = 1$.
\end{lemma}
\begin{proof}[Proof of Lemma \ref{lemma:Must-Terminate-with-Prob1}]
    Similar to the Lemma \ref{lemma:Termination-Round-of-algorithm}, we split the proof into two parts by discussing $\nu$ is either positive or negative.

    If $\nu$ is positive, by the Lemma \ref{lemma:Termination-Round-of-algorithm}, we have
    \begin{align*}
        \Pr(\tau = +\infty)\leq & \Pr(\kappa_{\text{end}}=+\infty)\\
        \leq & \Pr(\kappa=+\infty) + \Pr(\kappa_{\text{end}}-L^{\text{pos}}=+\infty | \kappa < +\infty),
    \end{align*}
    where $L^{\text{pos}}=\max\{\kappa,\lceil\log_2 \frac{9600(H+1)^2C^2}{(C-1)^2(V_0^{*}(s_0)-\mu_0)^2}\rceil\}$. By the Lemma \ref{lemma:kappa-tail-distribution}, we have $\Pr(\kappa=+\infty)=0$. In addition, Lemma 
    \ref{lemma:concentration-event-for-exploitation} suggests that
    \begin{align}
        & \Pr(\kappa_{\text{end}}\geq L^{\text{pos}}+k| \kappa)\notag\\
        \leq & \Pr\big((\neg \mathcal{E}_{L^{\text{pos}}}^{\text{et}}) \cap \cdots \cap (\neg \mathcal{E}_{L^{\text{pos}}+k-1}^{\text{et}})| \kappa\big)\notag\\
        \leq & \Pr(\neg \mathcal{E}_{L^{\text{pos}}+k-1}^{\text{et}}| \kappa)\notag\\
        \leq & \frac{\pi^2\delta}{6\alpha_{L^{\text{pos}}+k-1}}\\
        \leq & \frac{\pi^2\delta}{6\alpha_{k-1}}\label{eqn:kappa-end-tail-prob}
    \end{align}
    By the (\ref{eqn:kappa-end-tail-prob}), we have $\Pr(\kappa_{\text{end}}-L^{\text{pos}}=+\infty | \kappa < +\infty)=0$. Thus, we have $\Pr(\tau = +\infty)=0$, which means $\Pr(\tau < +\infty) = 1$.

    If $\nu$ is negative, by the Lemma \ref{lemma:Termination-Round-of-algorithm} and \ref{lemma:prob-good-event}, we have
    \begin{align*}
        \Pr(\tau = +\infty)\leq \Pr(\kappa=+\infty)=0,
    \end{align*}
    which suggests $\Pr(\tau < +\infty) = 1$.
\end{proof}

\subsection{Performance Guarantee of Algorithm \ref{alg:Oracle-BPI_UCRL}}
This section mainly follows the existing analysis of \cite{al2021adaptive}, which also provides required Lemmas for Appendix \ref{sec:proof-main-theorem}. For self-completeness, we still provide detailed proof. In the following proof, we follow the convention that takeing $\sum_{n=n_1}^{n_2} a_n=0$ for non-negative $\{a_n\}_{n=1}^{+\infty}$ and $n_1>n_2\geq 0; n_1,n_2\in \mathbb{Z}$.

By the same argument as the Lemma 10 in \cite{kaufmann2021adaptive}, we can conclude
\begin{lemma}
    \label{lemma:prob-good-event}
    $\Pr\big(\mathcal{E}(\delta)\cap\mathcal{E}^{\text{cnt}}(\delta)\big)\geq 1-\delta$ holds for all $\delta\in(0,1)$,
    where $\mathcal{E}^{\text{cnt}}(\delta)$, $\mathcal{E}(\delta)$ are defined in (\ref{eqn:E-cnt-event}), (\ref{eqn:E-event}).
\end{lemma}
\begin{proof}[Proof of Lemma \ref{lemma:prob-good-event}]
    Lemma 10 of \cite{kaufmann2021adaptive} asserts
    \begin{align*}
        \Pr\big(\mathcal{E}(\delta)\big)\geq & 1-\frac{1}{2}\delta\\
        \Pr\big(\mathcal{E}^{\text{cnt}}(\delta)\big)\geq & 1-\frac{1}{2}\delta.
    \end{align*}
\end{proof}

Given Lemma \ref{lemma:prob-good-event}, the following Lemma is obvious.
\begin{lemma}
    \label{lemma:kappa-tail-distribution}
    $\Pr(\kappa \geq k)\leq  \delta_{k-1}$.
\end{lemma}
\begin{proof}[Proof of Lemma \ref{lemma:kappa-tail-distribution}]
    Recall the definition of $\kappa:= \min\{k: \mathcal{E}(\delta_k)\cap \mathcal{E}^{\text{cnt}}(\delta_k) \text{ hold}\}$. Evident to see $\Pr(\kappa \geq k)\leq  \Pr(\mathcal{E}(\delta_{k-1})\cap \mathcal{E}^{\text{cnt}}(\delta_{k-1})\text{ don't hold})$. Then Lemma \ref{lemma:prob-good-event} implies Lemma \ref{lemma:kappa-tail-distribution}. 
\end{proof}

The following lemmas shows the optimistic and pessimistic estimation defined in Appendix \ref{sec:appendix-notations} indeed contains the value function or Q function.
\begin{lemma}
    \label{lemma:Boundness-Q-V-function}
    
    For any round $t$ during phase $k\geq \kappa$, and any policy $\pi$, we have
    \begin{align*}
        \underline{Q}_h^{t,\pi}(s,a;\delta_k)\leq Q_h^{\pi}(s,a) &  \leq \overline{Q}_h^{t,\pi}(s,a;\delta_k)\\
        \underline{Q}_h^{t}(s,a;\delta_k)\leq Q_h^{*}(s,a) & \leq \overline{Q}_h^{t}(s,a;\delta_k).
    \end{align*}
\end{lemma}
\begin{proof}[Proof of Lemma \ref{lemma:Boundness-Q-V-function}]
    We prove Lemma \ref{lemma:Boundness-Q-V-function} by Induction. Evident to see 
    \begin{align*}
        & \underline{Q}_{H+1}^{t,\pi}(s,a;\delta)= Q_{H+1}^{\pi}(s,a) = \overline{Q}_{H+1}^{t,\pi}(s,a;\delta)=0\\
        & \underline{Q}_{H+1}^{t}(s,a;\delta)= Q_{H+1}^{*}(s,a) = \overline{Q}_{H+1}^{t}(s,a;\delta) = 0.
    \end{align*}
    Now, assume the conclusion holds for $h+1, h+2,\cdots,H+1$, we turn to work on the case $h$. To prove $\underline{Q}_h^{t,\pi}(s,a;\delta)\leq Q_{h}^{\pi}(s,a)$, we can derive
    \begin{align*}
        \underline{Q}_h^{t,\pi}(s,a;\delta) =& r_h(s,a) + \min_{\underline{p}_h \in \mathcal{C}_h^t(s,a;\delta)}\sum_{s'} \underline{p}_h(s'|s,a) \underline{V}_{h+1}^{t,\pi}(s';\delta)\\
        \leq & r_h(s,a) + \sum_{s'} p_h(s'|s,a) \underline{V}_{h+1}^{t,\pi}(s';\delta)\\
        = & r_h(s,a) + \sum_{s'} p_h(s'|s,a) \underline{Q}_{h+1}^{t,\pi}(s',\pi(s');\delta)\\
        \stackrel{\text{Induction}}{\leq} & r_h(s,a) + \sum_{s'} p_h(s'|s,a) Q_{h+1}^{\pi}(s',\pi(s'))\\
        = & Q_{h}^{\pi}(s,a).
    \end{align*}
    To prove $\overline{Q}_h^{t,\pi}(s,a;\delta)\geq Q_{h}^{\pi}(s,a)$, we can derive
    \begin{align*}
        \overline{Q}_h^{t,\pi}(s,a;\delta) = & r_h(s,a) +\max_{\bar{p}_h\in \mathcal{C}_h^t(s,a;\delta)}\sum_{s'}\bar{p}_h(s'|s,a) \overline{V}_{h+1}^{t,\pi}(s';\delta)\\
        \geq & r_h(s,a) + \sum_{s'}p_h(s'|s,a) \overline{V}_{h+1}^{t,\pi}(s';\delta)\\
        = & r_h(s,a) + \sum_{s'}p_h(s'|s,a) \overline{Q}_{h+1}^{t,\pi}(s',\pi(s');\delta)\\
        \stackrel{\text{Induction}}{\geq} & r_h(s,a) + \sum_{s'}p_h(s'|s,a) Q_{h+1}^{\pi}(s',\pi(s'))\\
        = & Q_{h}^{\pi}(s,a).
    \end{align*}
    To prove $\underline{Q}_h^{t}(s,a;\delta)\leq  Q^*_h(s,a)$, we can derive
    \begin{align*}
        \underline{Q}_h^{t}(s,a;\delta) =& r_h(s,a) + \min_{\underline{p}_h \in \mathcal{C}_h^t(s,a;\delta)}\sum_{s'} \underline{p}_h (s'|s,a)\underline{V}_{h+1}^{t}(s';\delta) \\
        \leq & r_h(s,a) + \sum_{s'} p_h (s'|s,a)\underline{V}_{h+1}^{t}(s';\delta) \\
        = & r_h(s,a) + \sum_{s'} p_h (s'|s,a)\max_a \underline{Q}_{h+1}^{t}(s,a;\delta) \\ 
        \stackrel{\text{Induction}}{\leq} & r_h(s,a) + \sum_{s'} p_h (s'|s,a)\max_a Q_{h+1}^*(s,a) \\ 
        = & r_h(s,a) + \sum_{s'} p_h (s'|s,a)V_{h+1}^*(s) \\ 
        = & Q^*_h(s,a).
    \end{align*}
    To prove $\overline{Q}_h^{t}(s,a;\delta)\geq Q^*_h(s,a)$, we can derive
    \begin{align*}
        \overline{Q}_h^{t}(s,a;\delta)= & r_h(s,a)+\max_{\bar{p}_h\in \mathcal{C}_h^t(s,a;\delta)}\sum_{s'}\bar{p}_h(s'|s,a) \overline{V}_{h+1}^{t}(s';\delta)\\
        \geq & r_h(s,a)+\sum_{s'}p_h(s'|s,a) \overline{V}_{h+1}^{t}(s';\delta)\\
        = & r_h(s,a)+\sum_{s'}p_h(s'|s,a) \max_a \overline{Q}_{h+1}^{t}(s',a;\delta)\\
        \stackrel{\text{Induction}}{\geq} & r_h(s,a)+\sum_{s'}p_h(s'|s,a) \max_a Q_{h+1}(s',a)\\
        = & r_h(s,a) + \sum_{s'} p_h (s'|s,a)V_{h+1}^*(s) \\ 
        = & Q^*_h(s,a).
    \end{align*}
    Then, we complete the induction.
\end{proof}

Given Lemma \ref{lemma:Boundness-Q-V-function}, the following proposition is evident.
\begin{proposition}
    \label{proposition:Vbar-Vunder-bound-V}
    For any round $t$ during phase $k\geq \kappa$, we have
    \begin{align*}
        \max_a\underline{Q}_h^{t}(s,a;\delta_k)\leq \max_a Q_h^{*}(s,a) \leq \max_a\overline{Q}_h^{t}(s,a;\delta_k), \forall s,h
    \end{align*}
    which is equivalent to $\underline{V}_h^t(s;\delta)\leq V^*_h(s)\leq \overline{V}_h^t(s;\delta)$, $\forall s,h$. Further, for any policy $\pi$,
    \begin{align*}
        \underline{Q}_h^{t,\pi}(s,\pi_h(s);\delta_k)\leq Q_h^{\pi}(s,\pi_h(s)) &  \leq \overline{Q}_h^{t,\pi}(s,\pi_h(s);\delta_k), \forall s,h
    \end{align*}
    which is equivalent to $\underline{V}_h^{t,\pi}(s;\delta)\leq V^{\pi}_h(s)\leq \overline{V}_h^{t,\pi}(s;\delta)$, $\forall s,h$.
\end{proposition}


Following the Lemma 3 in \cite{kaufmann2021adaptive}, we have
\begin{lemma}
    \label{lemma:output-correctness-under-good-event}
    For any round $t$ during phase $k\geq \kappa$ in Algorithm \ref{alg:1-policy-identification-recycle-history}, meaning that $\mathcal{E}(\delta_k)\cap\mathcal{E}^{\text{cnt}}(\delta_k)$ holds, we have
    \begin{itemize}
        \item if the instance is positive, the output $\hat{a}$ from Algorithm \ref{alg:Oracle-BPI_UCRL} is either policy $\overline{\pi}^t(\cdot;\delta_k)$, satisfying that
        \begin{align*}
            V^{\overline{\pi}^t}_0(s_0)-\mu_0 \geq \frac{C-1}{C}\Big(V^*_0(s_0)-\mu_0\Big)
        \end{align*}
        or ``\textsf{Not Complete}''
        \item If the instance is negative, the output $\hat{a}$ from Algorithm \ref{alg:Oracle-BPI_UCRL} is either "Negative" or \textsf{Not Completed}.
    \end{itemize}
\end{lemma}
\begin{proof}[Proof of Lemma \ref{lemma:output-correctness-under-good-event}]
    For simpliciy, we denote $\overline{V}^t:=\overline{V}^t_0(s_0;\delta_k)$, $\underline{V}^{t,\overline{\pi}^t}:=\underline{V}^{t,\overline{\pi}^t(\cdot;\delta_k)}_0(s_0;\delta_k)$, $V^*:=V^*_0(s_0)$, $V^{\overline{\pi}^t}:=V^{\overline{\pi}^t(\cdot;\delta_k)}_0(s_0)$ in the following proof.

    If the instance is positive, Proposition \ref{proposition:Vbar-Vunder-bound-V} suggests that $\overline{V}^t \geq V^* > \mu_0$, suggesting that Algorithm \ref{alg:Oracle-BPI_UCRL} can only leave the loop at the Line \ref{line-alg:start-execute-MDP-with-new-delta_k} when (\ref{eqn:stop-criterion-pos}) holds or $t \geq T_k^{\text{ee}}$. If $t \geq T_k^{\text{ee}}$ holds, $\hat{a}=\textsf{Not Completed}$. If (\ref{eqn:stop-criterion-pos}) holds, $\hat{a}$ must be $\overline{\pi}^t$, notice that
    \begin{align}
        & \mu_0 < \underline{V}^{t,\overline{\pi}^t} - (C-1)(\overline{V}^t-\underline{V}^{t,\overline{\pi}^t})\notag\\
        \Leftrightarrow & \mu_0 <\underline{V}^{t,\overline{\pi}^t} + (\overline{V}^t-\underline{V}^{t,\overline{\pi}^t}) - C(\overline{V}^t-\underline{V}^{t,\overline{\pi}^t})\notag\\
        \Leftrightarrow & C(\overline{V}^t-\underline{V}^{t,\overline{\pi}^t}) < \overline{V}^t- \mu_0. \label{eqn:positive-stopping-suggest-interval}
    \end{align}
    Meanwhile, Lemma \ref{lemma:Boundness-Q-V-function} and Proposition \ref{proposition:Vbar-Vunder-bound-V} imply $V^{\overline{\pi}^t}, V^*\in [\underline{V}^{t,\overline{\pi}^t}, \overline{V}^t]$, suggesting
    \begin{align}
        & (\overline{V}^t- \mu_0) \leq (\overline{V}^t-\underline{V}^{t,\overline{\pi}^t}) +  (V^{\overline{\pi}^t} -\mu_0)\\
        \Rightarrow & (\overline{V}^t- \mu_0) - (V^{\overline{\pi}^t} -\mu_0)\leq (\overline{V}^t-\underline{V}^{t,\overline{\pi}^t})\label{eqn:positive-stopping-suggest-interval-gap}.
    \end{align}
    Replace $\overline{V}^t-\underline{V}^t$ in (\ref{eqn:positive-stopping-suggest-interval}) with (\ref{eqn:positive-stopping-suggest-interval-gap}), we can conclude
    \begin{align*}
        & C(\overline{V}^t-\underline{V}^t) < \overline{V}^t- \mu_0\\
        \Rightarrow & C\Big((\overline{V}^t- \mu_0) - (V^{\overline{\pi}^t} -\mu_0)\Big) < \overline{V}^t- \mu_0\\
        \Leftrightarrow & (C-1)(\overline{V}^t- \mu_0) < C (V^{\overline{\pi}^t} -\mu_0)\\
        \Rightarrow & (C-1)(V^*- \mu_0) < C (V^{\overline{\pi}^t} -\mu_0)\\
        \Leftrightarrow & \frac{C-1}{C}(V^*- \mu_0) < V^{\overline{\pi}^t} -\mu_0.
    \end{align*}
    The last line must complete the proof for the case that the instance is positive.

    If the instance is negative, Proposition \ref{proposition:Vbar-Vunder-bound-V} suggests that $\underline{V}^t\leq V^* < \mu_0$, suggesting that Algorithm \ref{alg:Oracle-BPI_UCRL} can only leave the loop at the Line \ref{line-alg:start-execute-MDP-with-new-delta_k} when (\ref{eqn:stop-criterion-neg}) holds or $t \geq T_k^{\text{ee}}$. If $t \geq T_k^{\text{ee}}$, $\hat{a}=\textsf{Not Completed}$. If (\ref{eqn:stop-criterion-neg}) holds, $\hat{a}=\textsf{None}$, which completes the proof for negative case.
\end{proof}

The following Lemma builds the suggests a necessary condition for Algorithm \ref{alg:Oracle-BPI_UCRL} to continue sampling.
\begin{lemma}
    \label{lemma:nonstopping-implies-error-bar-lower-bound}
    
    For any round $t$ during phase $k\geq \kappa$, meaning that $\mathcal{E}(\delta_k)\cap\mathcal{E}^{\text{cnt}}(\delta_k)$ holds, if Algorithm \ref{alg:Oracle-BPI_UCRL} doesn't terminate after sampling $t$-th episode, 
    \begin{itemize}
         \item if the instance is positive, we have
         \begin{align*}
             \frac{V^{*}_0(s_0)-\mu_0}{C} \leq \overline{V}^t_0(s_0;\delta_k)-\underline{V}^{t,\overline{\pi}^t(\cdot;\delta_k)}_0(s_0;\delta_k)
         \end{align*}
         \item if the instance is negative, we have
         \begin{align*}
             \mu_0-V^{*}_0(s_0) \leq \overline{V}^t_0(s_0;\delta_k)-\underline{V}^{t,\overline{\pi}^t(\cdot;\delta_k)}_0(s_0;\delta_k)
         \end{align*}
    \end{itemize}
\end{lemma}
\begin{proof}[Proof of Lemma \ref{lemma:nonstopping-implies-error-bar-lower-bound}]

    If the instance is positive, Algorithm \ref{alg:Oracle-BPI_UCRL} doesn't terminate after sampling $t$-th episode means 
    \begin{align*}
        & \underline{V}^{t,\overline{\pi}^t(\cdot;\delta_k)}_0(s_0;\delta_k) - (C-1)(\overline{V}^t_0(s_0;\delta_k)-\underline{V}^{t,\overline{\pi}^t(\cdot;\delta_k)}_0(s_0;\delta_k)) \leq  \mu_0\\
        \Leftrightarrow & C(\overline{V}^t_0(s_0;\delta_k)-\underline{V}^{t,\overline{\pi}^t(\cdot;\delta_k)}_0(s_0;\delta_k)) \geq \overline{V}^t_0(s_0;\delta_k)-\mu_0\\
        \Rightarrow & C(\overline{V}^t_0(s_0;\delta_k)-\underline{V}^{t,\overline{\pi}^t(\cdot;\delta_k)}_0(s_0;\delta_k)) \geq V^{*}_0(s_0)-\mu_0\\
        \Leftrightarrow & \overline{V}^t_0(s_0;\delta_k)-\underline{V}^{t,\overline{\pi}^t(\cdot;\delta_k)}_0(s_0;\delta_k) \geq \frac{V^{*}_0(s_0)-\mu_0}{C}.
    \end{align*}
    The second last line is by $\overline{V}^t_0(s_0;\delta_k)\geq V^{*}_0(s_0)$, guarantee by the $\mathcal{E}(\delta_k)\cap\mathcal{E}^{\text{cnt}}(\delta_k)$, Lemma \ref{lemma:Boundness-Q-V-function} and Proposition \ref{proposition:Vbar-Vunder-bound-V}.

    If the instance is negative, Algorithm \ref{alg:Oracle-BPI_UCRL} doesn't terminate after sampling $t$-th episode means $\overline{V}^t_0(s_0;\delta_k) \geq \mu_0$, which is equivalent to
    \begin{align*}
        & \overline{V}^t_0(s_0;\delta_k) \geq \mu_0\\
        \Leftrightarrow & \overline{V}^t_0(s_0;\delta_k)-\underline{V}^{t,\overline{\pi}^t(\cdot;\delta_k)}_0(s_0;\delta_k) \geq \mu_0-\underline{V}^{t,\overline{\pi}^t(\cdot;\delta_k)}_0(s_0;\delta_k)\\
        \Rightarrow & \overline{V}^t_0(s_0;\delta_k)-\underline{V}^{t,\overline{\pi}^t(\cdot;\delta_k)}_0(s_0;\delta_k) \geq \mu_0 - V^{*}_0(s_0).
    \end{align*}
    The last line is by $V^{*}_0(s_0) \geq V^{\overline{\pi}^t(\cdot;\delta_k)}_0(s_0) 
    \geq \underline{V}^{t,\overline{\pi}^t(\cdot;\delta_k)}_0(s_0;\delta_k)$, guarantee by the $\mathcal{E}(\delta_k)\cap\mathcal{E}^{\text{cnt}}(\delta_k)$, Lemma \ref{lemma:Boundness-Q-V-function} and Proposition \ref{proposition:Vbar-Vunder-bound-V}.
\end{proof}

To utilize Lemma \ref{lemma:nonstopping-implies-error-bar-lower-bound} to prove an implicit upper bound for the length of a phse $k\geq \kappa$, we introduce the following Lemma \ref{lemma:recursive-gap-Q-1}. Lemma \ref{lemma:recursive-gap-Q-1} can be considered an intermediate step for the proof of implicit upper bound in Lemma \ref{lemma:lower-bound-of-sum-average-radius}.
\begin{lemma}
    \label{lemma:recursive-gap-Q-1}
    For any round $t$ during phase $k\geq \kappa$, $\forall s,a,h$, we have
    \begin{align*}
        & \overline{Q}_h^t(s,a;\delta_k)-\underline{Q}_h^{t, \overline{\pi}^t(\cdot;\delta_k)}(s,a;\delta_k)\\
        \leq & 4(H-h)\min\left\{\sqrt{\frac{\beta_p(n_h^t(s,a),\delta_k)}{n_h^t(s,a)}}, 1\right\} + \sum_{s'}p_h(s'|s,a)\big(\overline{V}^t_{h+1}(s';\delta_k) - \underline{V}_{h+1}^{t,\overline{\pi}^t(\cdot;\delta_k)}(s';\delta_k)\big).
    \end{align*}
\end{lemma}
\begin{proof}[Proof of Lema \ref{lemma:recursive-gap-Q-1}]
    By the definition of $\overline{Q}_h^t(s,a)$, $\underline{Q}_h^{t, \overline{\pi}^t(\cdot;\delta_k)}(s,a)$, we have
    \begin{align}
        & \overline{Q}_h^t(s,a)-\underline{Q}_h^{t, \overline{\pi}^t(\cdot;\delta_k)}(s,a)\notag\\
        = & \sum_{s'} \overline{p}_h^t (s'|s,a)\overline{V}_{h+1}^{t}(s';\delta_k) - \sum_{s'} \underline{p}_h^{t,\overline{\pi}^t} (s'|s,a)\underline{V}_{h+1}^{t, \overline{\pi}^t(\cdot;\delta_k)}(s';\delta_k)\notag\\
        = & \sum_{s'} \overline{p}_h^t (s'|s,a)\overline{V}_{h+1}^{t}(s';\delta_k) - \sum_{s'} p_h (s'|s,a)\overline{V}_{h+1}^{t}(s';\delta_k) + \notag\\
        & \sum_{s'} p_h (s'|s,a)\overline{V}_{h+1}^{t}(s';\delta_k) -  \sum_{s'} p_h (s'|s,a)\underline{V}_{h+1}^{t, \overline{\pi}^t(\cdot;\delta_k)}(s';\delta_k) + \notag\\
        & \sum_{s'} p_h (s'|s,a)\underline{V}_{h+1}^{t, \overline{\pi}^t(\cdot;\delta_k)}(s';\delta_k) - \sum_{s'} \underline{p}_h^{t,\overline{\pi}^t(s';\delta_k)} (s'|s,a)\underline{V}_{h+1}^{t, \overline{\pi}^t(\cdot;\delta_k)}(s';\delta_k)\notag\\
        \leq & \|\overline{p}_h^t(\cdot|s,a)-p_h(\cdot|s,a) \|_1(H-h) + \sum_{s'} p_h (s'|s,a)\big(\overline{V}_{h+1}^{t}(s';\delta_k)-\underline{V}_{h+1}^{t, \overline{\pi}^t(\cdot;\delta_k)}(s';\delta_k)\big) + \notag\\
        & \|p_h(\cdot|s,a)-\underline{p}_h^{t,\overline{\pi}^t(s';\delta_k)}(\cdot|s,a) \|_1(H-h)\label{eqn:boundness-Q-V-H}\\
        \leq & 4(H-h)\min\left\{\sqrt{\frac{\beta_p(n_h^t(s,a),\delta_k)}{n_h^t(s,a)}},1\right\} + \sum_{s'} p_h (s'|s,a)\big(\overline{V}_{h+1}^{t}(s';\delta_k)-\underline{V}_{h+1}^{t, \overline{\pi}^t(\cdot;\delta_k)}(s';\delta_k)\big)\label{eqn:pinsker-inequality}
    \end{align}
    Step (\ref{eqn:boundness-Q-V-H}) is by the fact that $\overline{V}_{h+1}^{t}(s';\delta_k), \underline{V}_{h+1}^{t, \overline{\pi}^t(\cdot;\delta_k)}(s';\delta_k)\leq H$. Step (\ref{eqn:pinsker-inequality}) is by the fact that $\|p(\cdot)-q(\cdot)\|_1\leq 2$, and the Pinsker's Inequality $\|p(\cdot)-q(\cdot)\|_1\leq \sqrt{\frac{1}{2}\text{KL}(p, q)}$, as $k\geq \kappa$ suggests that $\overline{p}_h^t(\cdot|s,a), \underline{p}_h^{t,\overline{\pi}^t}(\cdot|s,a)\in \mathcal{C}_h^t(s,a;\delta_k)$.
\end{proof}

Given Lemma \ref{lemma:recursive-gap-Q-1}, we can immediately derive the following proposition by taking $a=\overline{\pi}^t_h(s;\delta_k)$.
\begin{proposition}
    \label{proposition:recursive-gap-V-1}
    For any round $t$ during phase $k\geq \kappa$, $\forall s,a,h$, we have
    \begin{align*}
        & \overline{V}_h^t(s;\delta_k)-\underline{V}_h^{t, \overline{\pi}^t(\cdot;\delta_k)}(s;\delta_k)\\
        \leq & 4(H-h)\min\left\{\sqrt{\frac{\beta_p(n_h^t(s,\overline{\pi}^t_h(s)),\delta_k)}{n_h^t(s,\overline{\pi}^t_h(s))}}, 1\right\} + \\
        & \sum_{s'}p_h(s'|s,\overline{\pi}^t_h(s;\delta_k))\big(\overline{V}^t_{h+1}(s';\delta_k) - \underline{V}_{h+1}^{t, \overline{\pi}^t(\cdot;\delta_k)}(s';\delta_k)\big).
    \end{align*}
\end{proposition}

By repeatedly apply Proposition \ref{proposition:recursive-gap-V-1}, we can prove an upper bound for the $\overline{V}^t_0(s_0;\delta_k)-\underline{V}^{t, \overline{\pi}^t(\cdot;\delta_k)}_0(s_0;\delta_k)$, which is the following Lemma.
\begin{lemma}
    \label{lemma:recursive-relationship-V_0-gap}
    For any round $t$ during phase $k\geq \kappa$, meaning that $\mathcal{E}(\delta)\cap\mathcal{E}^{\text{cnt}}(\delta)$ holds, we have
    \begin{align*}
        & \overline{V}^t_0(s_0;\delta_k)-\underline{V}^{t, \overline{\pi}^t(\cdot;\delta_k)}_0(s_0;\delta_k) \\
        \leq & 4H\min\left\{\sqrt{\frac{\beta_p(t,\delta_k)}{t}}, 1\right\} + 4\sum_{h=1}^{H}(H-h)\sum_{s',a'}p_h^{\overline{\pi}^t}(s',a')\min\left\{\sqrt{\frac{\beta_p(n_h^t(s',a'),\delta_k)}{n_h^t(s',a')}}, 1\right\}
    \end{align*}
\end{lemma}
\begin{proof}[Proof of Lemma \ref{lemma:recursive-relationship-V_0-gap}]
    Given $\overline{\pi}^t(\cdot;\delta_k)$ and the definition of $p_h^{\overline{\pi}^t(\cdot;\delta_k)}(s',a'):=\Pr(S_{t,h}=s,\overline{\pi}^t_h(s;\delta_k)=a)$.

    We use induction to prove
    \begin{align}
        & \overline{V}^t_0(s_0;\delta_k)-\underline{V}^{t,\overline{\pi}^t(\cdot;\delta_k)}_0(s_0;\delta_k)\notag \\
        \leq & 4H\min\left\{\sqrt{\frac{\beta_p(t,\delta_k)}{t}}, 1\right\} + 4\sum_{h=1}^{h'-1}(H-h)\sum_{s',a'}p_h^{\overline{\pi}^t}(s',a')\min\left\{\sqrt{\frac{\beta_p(n_h^t(s',a'),\delta_k)}{n_h^t(s',a')}}, 1\right\} + \notag\\
        & \sum_{s',a'}p_{h'}^{\overline{\pi}^t}(s',a')\big(\overline{V}^t_{h'}(s';\delta_k) - \underline{V}_{h'}^{t,\overline{\pi}^t(\cdot;\delta_k)}(s';\delta_k)\big)\label{eqn:V_0-gap-induction-h'}
    \end{align}
    holds for all $h'=1,\cdots,H+1$.

    For $h'=1$, by the Proposition \ref{proposition:recursive-gap-V-1}, we know
    \begin{align*}
        & \overline{V}_0^t(s_0;\delta_k)-\underline{V}_0^{t, \overline{\pi}^t(\cdot;\delta_k)}(s_0;\delta_k)\\
        \leq & 4H\min\left\{\sqrt{\frac{\beta_p(n_h^t(s_0,a_0),\delta_k)}{n_h^t(s_0,a_0)}}, 1\right\} + \sum_{s'}p_0(s')\big(\overline{V}^t_{1}(s';\delta_k) - \underline{V}_{1}^{t,\overline{\pi}^t(\cdot;\delta_k)}(s';\delta_k)\big)\\
        = & 4H\min\left\{\sqrt{\frac{\beta_p(t,\delta_k)}{t}}, 1\right\} + \sum_{s'}p_0(s')\big(\overline{V}^t_{1}(s';\delta_k) - \underline{V}_{1}^{t,\overline{\pi}^t(\cdot;\delta_k)}(s';\delta_k)\big)\\
        = & 4H\min\left\{\sqrt{\frac{\beta_p(t,\delta_k)}{t}}, 1\right\} + \sum_{s',a'}p_0(s')\mathds{1}(a'=\overline{\pi}^t_1(s'))\big(\overline{V}^t_{1}(s';\delta_k) - \underline{V}_{1}^{t,\overline{\pi}^t(\cdot;\delta_k)}(s';\delta_k)\big)\\
        = & 4H\min\left\{\sqrt{\frac{\beta_p(t,\delta_k)}{t}}, 1\right\} + \sum_{s',a'}p_{1}^{\overline{\pi}^t}(s',a')\big(\overline{V}^t_{1}(s';\delta_k) - \underline{V}_{1}^{t,\overline{\pi}^t(\cdot;\delta_k)}(s';\delta_k)\big).
    \end{align*}
    The third line is by $n_h^t(s_0,a_0)=t$ holds for all $t$, the last line is by the definition $p_1^{\overline{\pi}^t}(s,a) = \Pr(S_1=s,\overline{\pi}^t_h(s)=a)$.

    Given (\ref{eqn:V_0-gap-induction-h'}) holds for $h'$, apply the Proposition \ref{proposition:recursive-gap-V-1} again, we can derive
    \begin{align*}
        & \overline{V}^t_0(s_0;\delta_k)-\underline{V}^{t,\overline{\pi}^t(\cdot;\delta_k)}_0(s_0;\delta_k) \\
        \leq & 4H\min\left\{\sqrt{\frac{\beta_p(t,\delta_k)}{t}}, 1\right\} + 4\sum_{h=1}^{h'-1}(H-h)\sum_{s',a'}p_h^{\overline{\pi}^t(\cdot;\delta_k)}(s',a')\min\left\{\sqrt{\frac{\beta_p(n_h^t(s',a'),\delta_k)}{n_h^t(s',a')}}, 1\right\} + \\
        & \sum_{s',a'}p_{h'}^{\overline{\pi}^t(\cdot;\delta_k)}(s',a')\left(4(H-h')\min\left\{\sqrt{\frac{\beta_p(n_{h'}^t(s',\overline{\pi}^t_{h'}(s';\delta_k)),\delta_k)}{n_{h'}^t(s',\overline{\pi}^t_{h'}(s';\delta_k))}}, 1\right\} + \right. \\
        & \left. \sum_{s''}p_{h'}(s''|s',\overline{\pi}^t_{h'}(s';\delta_k))\big(\overline{V}^t_{h'+1}(s'';\delta_k) - \underline{V}_{h'+1}^{t,\overline{\pi}^t(\cdot;\delta_k)}(s'';\delta_k)\big)\right)\\
        = & 4H\min\left\{\sqrt{\frac{\beta_p(t,\delta_k)}{t}}, 1\right\} + 4\sum_{h=1}^{h'}(H-h)\sum_{s',a'}p_h^{\overline{\pi}^t(\cdot;\delta_k)}(s',a')\min\left\{\sqrt{\frac{\beta_p(n_h^t(s',a'),\delta_k)}{n_h^t(s',a')}}, 1\right\} + \\
        & \sum_{s',a'}\sum_{s''}p_{h'}^{\overline{\pi}^t(\cdot;\delta_k)}(s',a')p_{h'}(s''|s',\overline{\pi}^t_{h'}(s';\delta_k))\big(\overline{V}^t_{h'+1}(s'';\delta_k) - \underline{V}_{h'+1}^{t,\overline{\pi}^t}(s'';\delta_k)\big)\\
        = & 4H\min\left\{\sqrt{\frac{\beta_p(t,\delta_k)}{t}}, 1\right\} + 4\sum_{h=1}^{h'}(H-h)\sum_{s',a'}p_h^{\overline{\pi}^t}(s',a')\min\left\{\sqrt{\frac{\beta_p(n_h^t(s',a'),\delta_k)}{n_h^t(s',a')}}, 1\right\} + \\
        & \sum_{s'}\sum_{s''}p_{h'}^{\overline{\pi}^t(\cdot;\delta_k)}(s',\overline{\pi}^t_{h'}(s';\delta_k))p_{h'}(s''|s',\overline{\pi}^t_{h'}(s';\delta_k))\big(\overline{V}^t_{h'+1}(s'';\delta_k) - \underline{V}_{h'+1}^{t,\overline{\pi}^t}(s'';\delta_k)\big)\\
        = & 4H\min\left\{\sqrt{\frac{\beta_p(t,\delta_k)}{t}}, 1\right\} + 4\sum_{h=1}^{h'}(H-h)\sum_{s',a'}p_h^{\overline{\pi}^t(\cdot;\delta_k)}(s',a')\min\left\{\sqrt{\frac{\beta_p(n_h^t(s',a'),\delta_k)}{n_h^t(s',a')}}, 1\right\} + \\
        & \sum_{s''}\Pr_{\overline{\pi}^t(\cdot;\delta_k)}(S_{h'+1}=s'')\big(\overline{V}^t_{h'+1}(s'';\delta_k) - \underline{V}_{h'+1}^{t,\overline{\pi}^t(\cdot;\delta_k)}(s'';\delta_k)\big)\\
        = & 4H\min\left\{\sqrt{\frac{\beta_p(t,\delta_k)}{t}}, 1\right\} + 4\sum_{h=1}^{h'}(H-h)\sum_{s',a'}p_h^{\overline{\pi}^t(\cdot;\delta_k)}(s',a')\min\left\{\sqrt{\frac{\beta_p(n_h^t(s',a'),\delta_k)}{n_h^t(s',a')}}, 1\right\} + \\
        & \sum_{s',a'}p^{\overline{\pi}^t(\cdot;\delta_k)}_{h'+1}(s',a')\big(\overline{V}^t_{h'+1}(s'';\delta_k) - \underline{V}_{h'+1}^{t,\overline{\pi}^t(\cdot;\delta_k)}(s'';\delta_k)\big),
    \end{align*}
    which completed the induction and (\ref{eqn:V_0-gap-induction-h'}) holds for all $h'=1,2,\cdots,H+1$.
    
    By taking $h'=H$ and utilizing the fact that $\overline{V}^t_{H+1}(s;\delta_k)=0, \underline{V}_{H+1}^{t,\overline{\pi}^t(\cdot;\delta_k)}(s;\delta_k)=0, \forall s$, we complete the proof of Lemma \ref{lemma:recursive-relationship-V_0-gap}.
\end{proof}

\begin{definition}
    \label{definition:t_k-split-the-pulling-process-into-phases}
    Denote $t_{k-1}$ as the starting t index of phase $k$ for phase index $k\geq 1$. For $k\geq 2$, $t_{k-1}$ is a random variable. For consistency, we set $t_0 = 0$.
\end{definition}
What's more, we know $\neg$(\ref{eqn:stop-criterion-pos}) and $\neg$(\ref{eqn:stop-criterion-neg}) must hold for episode index $t=t_{k-1}, t_{k-1}+1,\cdots, t_k-1$. It is indeed possible that $t_{k-1}=t_k$ for some sampling path $\{\{S_{t,h}, A_{t,h}\}_{h=1}^{H}\}_{t=1}^{+\infty}$ and index $k\in\mathbb{N}$. In this case, the history $\mathcal{H}^{\text{ee}}$ with length $|\mathcal{H}^{\text{ee}}| = t_{k-1}$ guarantees that $\underline{V}^{t,\overline{\pi}^t(\cdot;\delta_k)}_0(s_0;\delta_{k+1}) - (C-1)(\overline{V}^t_0(s_0;\delta_{k+1})-\underline{V}^{t,\overline{\pi}^t(\cdot;\delta_k)}_0(s_0;\delta_{k+1})) > \mu_0 $ or $\overline{V}^t_0(s_0;\delta_{k+1}) < \mu_0$. Following this idea, we can combine Lemma \ref{lemma:nonstopping-implies-error-bar-lower-bound} and \ref{lemma:recursive-relationship-V_0-gap} to derive upper bounds for the length of each phase.
\begin{lemma}
    \label{lemma:lower-bound-of-sum-average-radius}
    Consider any phase index $k\geq \kappa$. If the instance is positive, we have
    \begin{align*}
        & (t_{k}-t_{k-1}) \frac{V^*_0(s_0)-\mu}{C}\\
        \leq &  \sum_{t=t_{k-1}}^{t_{k}-1}\sum_{h=0}^H\sum_{(s,a)}p_h^{\overline{\pi}^t(\cdot;\delta_k)}(s,a)4(H-h)\min\left\{\sqrt{\frac{\beta_p(n_h^t(s,a),\delta_{k})}{n_h^t(s,a)}}, 1\right\}.
    \end{align*}
    If the instance is negative, we have
    \begin{align*}
        & (t_{k}-t_{k-1}) (\mu-V^*_0(s_0))\\
        \leq & \sum_{t=t_{k-1}}^{t_{k}-1}\sum_{h=0}^H\sum_{(s,a)}p_h^{\overline{\pi}^t(\cdot;\delta_k)}(s,a)4(H-h)\min\left\{\sqrt{\frac{\beta_p(n_h^t(s,a),\delta_{k})}{n_h^t(s,a)}}, 1\right\}.
    \end{align*}
\end{lemma}
\begin{proof}[Proof of Lemma \ref{lemma:lower-bound-of-sum-average-radius}]
    If the instance is positive, by Lemma \ref{lemma:nonstopping-implies-error-bar-lower-bound} and \ref{lemma:recursive-relationship-V_0-gap}, we can conclude
    \begin{align*}
        & \frac{V^{*}_0(s_0)-\mu_0}{C} \\
        \leq & 4H\min\left\{\sqrt{\frac{\beta_p(t,\delta_k)}{t}}, 1\right\} + 4\sum_{h=1}^{H}(H-h)\sum_{s',a'}p_h^{\overline{\pi}^t(\cdot;\delta_k)}(s',a')\min\left\{\sqrt{\frac{\beta_p(n_h^t(s',a'),\delta_k)}{n_h^t(s',a')}}, 1\right\}
    \end{align*}
    holds for all $t=t_{k-1}, t_{k-1}+1,\cdots.t_k-1$. Summing up the above inequalities for $t=t_{k-1}, t_{k-1}+1,\cdots t_k-1$ leads to the inequality for the positive case.

    If the instance is negative, by Lemma \ref{lemma:nonstopping-implies-error-bar-lower-bound} and \ref{lemma:recursive-relationship-V_0-gap}, we can conclude
    \begin{align*}
        & \mu_0-V^{*}_0(s_0) \\
        \leq & 4H\min\left\{\sqrt{\frac{\beta_p(t,\delta_k)}{t}}, 1\right\} + 4\sum_{h=1}^{H}(H-h)\sum_{s',a'}p_h^{\overline{\pi}^t(\cdot;\delta_k)}(s',a')\min\left\{\sqrt{\frac{\beta_p(n_h^t(s',a'),\delta_k)}{n_h^t(s',a')}}, 1\right\}
    \end{align*}
    holds for all $t=t_{k-1}, t_{k-1}+1,\cdots.t_k-1$. Summing up the above inequalities for $t=t_{k-1}, t_{k-1}+1,\cdots t_k-1$ leads to the inequality for the negative case.
\end{proof}
To derive an upper bound for $t_k$, we need are going to prove an upper bound for the right-hand side of inequalities in Lemma \ref{lemma:lower-bound-of-sum-average-radius}. The following Lemma will be useful.
\begin{lemma}[Lemma 7 in \cite{kaufmann2021adaptive}]
    \label{lemma:lower-bound-on-n_h^t(s,a)}

    For any round $t$ during phase $k\geq \kappa$, any pair $(s,a,h)$
    \begin{align*}
        \min\left\{\sqrt{\frac{\beta_p(n_h^t(s,a),\delta_{k})}{n_h^t(s,a)}}, 1\right\}\leq 2\sqrt{\frac{\beta_p(\bar{n}_h^t(s,a),\delta_{k})}{\max\{\bar{n}_h^t(s,a), 1\}}}
    \end{align*}
    holds for all $t<\tau$.
\end{lemma}
\begin{proof}[Proof of Lemma \ref{lemma:lower-bound-on-n_h^t(s,a)}]
    Recall the definition $\beta^{\text{cnt}}(\delta) = \log\frac{2SAH}{\delta}$, $\beta_p(t,\delta) = \log\frac{2SAH}{\delta} + (S-1)\log\Big(e(1+\frac{t}{S-1})\Big)$. $k\geq \kappa$ suggests that $n^t_h(s,a)\geq \frac{1}{2}\bar{n}_h^t(s,a)-\beta^{\text{cnt}}(\delta_k)$ holds for all $s,a,h$.

    We prove Lemma \ref{lemma:lower-bound-on-n_h^t(s,a)} in two cases, $\beta^{cnt}(\delta_k)\leq \frac{1}{4}\bar{n}_h^t(s,a)$ or $\beta^{cnt}(\delta_k)> \frac{1}{4}\bar{n}_h^t(s,a)$. 
    
    First, if $\beta^{cnt}(\delta_k)\leq \frac{1}{4}\bar{n}_h^t(s,a)$, we have $n_h^t(s,a) \geq \frac{1}{4}\bar{n}_h^t(s,a)$. Notice that $\beta_p(t,\delta) = \log\frac{2SAH}{\delta} + (S-1)\log\Big(S-1+t\Big)+(S-1)\log\frac{e}{S-1}$, we can derive its partial derivative regarding $t$,
    \begin{align*}
        \frac{\partial\frac{\beta_p(t,\delta)}{t}}{\partial t} = & -\frac{\log\frac{2SAH}{\delta}}{t^2}+\frac{(S-1)\left(\frac{1}{S-1+t} - \log(S-1+t)\right)}{t^2}
    \end{align*}
    For $t\geq 1$, $\frac{1}{S-1+t} - \log(S-1+t)\leq \frac{1}{S}-\log S < 0$ holds for $S\geq 2$. We can conclude $\frac{\partial\frac{\beta_p(t,\delta)}{t}}{\partial t}<0$ for $t\geq 1$. Since $\beta^{cnt}(\delta_k) > 1$ holds for all $k$, we can conclude $\frac{1}{4}\bar{n}_h^t(s,a)>1$. Together with $n_h^t(s,a) \geq \frac{1}{4}\bar{n}_h^t(s,a)$, we have
    \begin{align*}
        \min\left\{\sqrt{\frac{\beta_p(n_h^t(s,a),\delta_k)}{n_h^t(s,a)}}, 1\right\}\leq \sqrt{\frac{\beta_p(n_h^t(s,a),\delta_k)}{n_h^t(s,a)}}\leq \sqrt{\frac{\beta_p(\frac{1}{4}\bar{n}_h^t(s,a),\delta_k)}{\frac{1}{4}\bar{n}_h^t(s,a)}}\leq 2\sqrt{\frac{\beta_p(\bar{n}_h^t(s,a),\delta_k)}{\bar{n}_h^t(s,a)}}.
    \end{align*}
    Again, by the fact that $\frac{4\beta_p(\bar{n}_h^t(s,a),\delta_k)}{\bar{n}_h^t(s,a)}=\frac{4\beta_p(\bar{n}_h^t(s,a),\delta_k)}{\max\{\bar{n}_h^t(s,a), 1\}}$, we can conclude
    \begin{align*}
        \min\left\{\sqrt{\frac{\beta_p(n_h^t(s,a),\delta_{k})}{n_h^t(s,a)}}, 1\right\}\leq 2\sqrt{\frac{\beta_p(\bar{n}_h^t(s,a),\delta_{k})}{\max\{\bar{n}_h^t(s,a), 1\}}},
    \end{align*}
    which suggests the Lemma holds.

    Second, if $\beta^{cnt}(\delta_k)> \frac{1}{4}\bar{n}_h^t(s,a)$, we have $\frac{4\beta^{cnt}(\delta_k)}{\bar{n}_h^t(s,a)}\geq 1$. Meanwhile, $\beta^{cnt}(\delta_k)\geq 1$ by the definition, suggesting that
    \begin{align}
        \min\left\{\sqrt{\frac{\beta_p(n_h^t(s,a),\delta_k)}{n_h^t(s,a)}}, 1\right\}\leq 1 \leq \sqrt{\frac{4\beta^{cnt}(\delta_k)}{\max\{\bar{n}_h^t(s,a), 1\}}}=2\sqrt{\frac{\beta^{cnt}(\delta_k)}{\max\{\bar{n}_h^t(s,a), 1\}}}\label{eqn:beta_p/t-upper-bound}
    \end{align}
    By the definition of $\beta_p$ and $\beta^{\text{cnt}}$, we know $\beta_p(t,\delta_k) > \beta^{\text{cnt}}(\delta_k)$ holds for all $t > 0$. Together with (\ref{eqn:beta_p/t-upper-bound}), we conclude
    \begin{align*}
        \min\left\{\sqrt{\frac{\beta_p(n_h^t(s,a),\delta_k)}{n_h^t(s,a)}}, 1\right\}\leq 2\sqrt{\frac{\beta^{cnt}(\delta_k)}{\max\{\bar{n}_h^t(s,a), 1\}}} < 2\sqrt{\frac{\beta(\bar{n}_h^t(s,a),\delta_k)}{\max\{\bar{n}_h^t(s,a), 1\}}},
    \end{align*}
    which means the Lemma still holds.
\end{proof}

Given all the above Lemmas, we are ready to prove a main theorem, showing that Algorithm \ref{alg:Oracle-BPI_UCRL} is able to output the correct prediction for large enough phase index $k$, with an upper bound of $t_k$.
\begin{theorem}
    \label{theorem:performance-guarantee-of-exploration-oracle}
    Apply Algorithm \ref{alg:Oracle-BPI_UCRL} on an instance $\nu$.

    If $\nu$ is positive, for any phase index $k\geq \max\{\kappa, \lceil\log_2 \frac{9600(H+1)^2}{{\epsilon^{\text{pos}}}^2}\rceil\}$, Algorithm \ref{alg:Oracle-BPI_UCRL} will output a policy $\hat{\pi}_k$ satisfying
    \begin{align*}
        V^{\hat{\pi}_k}_0(s_0)-\mu_0 \geq \frac{C-1}{C}\Big(V^*_0(s_0)-\mu_0\Big)
    \end{align*}
    at the end of phase $k$, with
    \begin{equation}\label{eqn:t_k-pos-upper-good-event-large-T_k-ee}
        \begin{split}
            t_k\leq & \frac{101(H+1)^4SA\log\frac{2SAH}{\delta_k}}{{\epsilon^{\text{pos}}}^2}+ T_{\kappa-1}^{\text{ee}} +\\
            & \frac{400(H+1)^4S^2A}{{\epsilon^{\text{pos}}}^2}\log\left(\frac{402(H+1)^4S^2A \log\frac{2SAH}{\delta_k}}{{\epsilon^{\text{pos}}}^2}\right)+ \frac{400(H+1)^4S^2A}{{\epsilon^{\text{pos}}}^2}\log\left(2T_{\kappa-1}^{\text{ee}}\right).
        \end{split}
    \end{equation}

    If $\nu$ is negative, for any phase index $k\geq\max\{\kappa, \lceil\log_2 \frac{9600(H+1)^2}{{\epsilon^{\text{neg}}}^2}\rceil\}$, Algorithm \ref{alg:Oracle-BPI_UCRL} will output \textsf{None} at the end of phase $k$, with
    \begin{equation}\label{eqn:t_k-neg-upper-good-event-large-T_k-ee}
        \begin{split}
            t_k\leq & \frac{101(H+1)^4SA\log\frac{2SAH}{\delta_k}}{{\epsilon^{\text{neg}}}^2}+ T_{\kappa-1}^{\text{ee}} +\\
            & \frac{400(H+1)^4S^2A}{{\epsilon^{\text{neg}}}^2}\log\left(\frac{402(H+1)^4S^2A \log\frac{2SAH}{\delta_k}}{{\epsilon^{\text{neg}}}^2}\right)+ \frac{400(H+1)^4S^2A}{{\epsilon^{\text{neg}}}^2}\log\left(2T_{\kappa-1}^{\text{ee}}\right)
        \end{split}
    \end{equation}
    Here $\epsilon^{\text{pos}}=\frac{V^{*}_0(s_0)-\mu_0}{C}$, and $\epsilon^{\text{neg}} = \mu_0-V^{*}_0(s_0)$.
\end{theorem}
\begin{proof}[Proof of Theorem \ref{theorem:performance-guarantee-of-exploration-oracle}]
    We temporarily focus on the positive case. The analysis for the negative case is similar.
    
    From Lemma \ref{lemma:lower-bound-of-sum-average-radius}, we know
    \begin{align}
        (t_{k}-t_{k-1}) \epsilon^{\text{pos}}
        \leq & \sum_{t=t_{k-1}}^{t_{k}-1}\sum_{h=0}^H\sum_{(s,a)}p_h^{\overline{\pi}^t}(s,a)4(H-h)\min\left\{\sqrt{\frac{\beta_p(n_h^t(s,a),\delta_k)}{n_h^t(s,a)}}, 1\right\}\label{eqn:length-of-period-k-main-theorem}
    \end{align}
    holds for all $k\geq \kappa$. Recall we denote $t_{k-1}$ as the starting t index of stage $k$, we first show (\ref{eqn:t_k-pos-upper-good-event-large-T_k-ee}) holds for $k\geq \kappa$. 
    In the following, we will consider consider multiple phase indexes $\ell=\kappa,\kappa+1,\cdots,k$. For simplicity, we will use $\overline{\pi}^t$ as the shorthand of $\overline{\pi}^{t}(\cdot;\delta_{\ell})$ in the following calculation.
    
    By summing up (\ref{eqn:length-of-period-k-main-theorem}) for stage index $\ell = \kappa,\kappa+1,\cdots,k$, We have
    \begin{align}
        & (t_k-t_{\kappa-1})\epsilon^{\text{pos}}\notag\\
        \leq & \sum_{\ell = \kappa}^{k} \sum_{t=t_{\ell-1}}^{t_{\ell}-1}\sum_{h=0}^H\sum_{(s,a)}p_h^{\overline{\pi}^t(\cdot;\delta_{\ell})}(s,a)4(H-h)\min\left\{\sqrt{\frac{\beta_p(n_h^t(s,a),\delta_{\ell})}{n_h^t(s,a)}}, 1\right\}\notag\\
        \stackrel{\text{Lemma }\ref{lemma:lower-bound-on-n_h^t(s,a)}}{\leq} & \sum_{\ell = \kappa}^{k} \sum_{t=t_{\ell-1}}^{t_{\ell}-1}\sum_{h=0}^H\sum_{(s,a)}p_h^{\overline{\pi}^t}(s,a)8(H-h)\sqrt{\frac{\beta_p(\bar{n}_h^t(s,a),\delta_{\ell})}{\max\{\bar{n}_h^t(s,a), 1\}}}\notag\\
        \leq & 8\sqrt{\beta_p(t_k,\delta_{k})}\sum_{h=0}^H(H-h)\sum_{(s,a)} \sum_{\ell = \kappa}^{k} \sum_{t=t_{\ell-1}}^{t_{\ell}-1} p_h^{\overline{\pi}^t}(s,a)\sqrt{\frac{1}{\max\{\bar{n}_h^t(s,a), 1\}}}\notag\\
        = & 8\sqrt{\beta_p(t_k,\delta_{k})}\sum_{h=0}^H(H-h)\sum_{(s,a)} \sum_{\ell = \kappa}^{k} \sum_{t=t_{\ell-1}}^{t_{\ell}-1}\notag\\
        & p_h^{\overline{\pi}^t}(s,a)\sqrt{\frac{1}{\max\{ \sum_{v=0}^{t_{\kappa-1}-1}  p_{h}^{\overline{\pi}^v}(s,a) + \sum_{v=t_{\kappa-1}}^{t-1} p_{h}^{\overline{\pi}^v}(s,a), 1\}}}\notag\\
        = & 8\sqrt{\beta_p(t_k,\delta_{k})}\sum_{h=0}^H(H-h)\sum_{(s,a)} \sum_{t=t_{\kappa-1}}^{t_{k}-1}\notag\\
        & p_h^{\overline{\pi}^t}(s,a)\sqrt{\frac{1}{\max\{ \sum_{v=0}^{t_{\kappa-1}-1}  p_{h}^{\overline{\pi}^v}(s,a) + \sum_{v=t_{\kappa-1}}^{t-1} p_{h}^{\overline{\pi}^v}(s,a), 1\}}}\notag\\
        \stackrel{\text{Lemma }\ref{lemma:summation-of-prob}}{\leq} & 8(1+\sqrt{2})\sqrt{\beta_p(t_k,\delta_{k})}\sum_{h=0}^H(H-h)\sum_{(s,a)}\sqrt{\sum_{v=t_{\kappa-1}}^{t_k-1} p_{h}^{\overline{\pi}^{v}}(s,a)}.\label{eqn:upper-for-tk-t-(kappa-1)}
    \end{align}
    To get the last step, we apply Lemma \ref{lemma:summation-of-prob} by taking $M=\sum_{v=0}^{t_{\kappa-1}-1}p_{h}^v(s,a)$. By the Cauchy-Schwarz Inequality, 
    \begin{align*}
        \sum_{(s,a)}\sqrt{\sum_{v=t_{\kappa-1}}^{t_k-1} p_{h}^v(s,a)} \leq \sqrt{\left(\sum_{s,a}\sum_{v=t_{\kappa-1}}^{t_k-1} p_{h}^v(s,a)\right)(\sum_{(s,a)} 1)} = \sqrt{SA(t_k-t_{\kappa-1})}.
    \end{align*}
    Following (\ref{eqn:upper-for-tk-t-(kappa-1)}), we conclude
    \begin{align}
        & (t_k-t_{\kappa-1})\epsilon^{\text{pos}}\notag\\
        \leq & 8(1+\sqrt{2})\sqrt{SA\beta_p(t_k,\delta_{k})}\sum_{h=0}^H(H-h)\sqrt{t_k-t_{\kappa-1}}\notag\\
        = & 4(1+\sqrt{2})\sqrt{SA\beta_p(t_k,\delta_{k})}H(H+1)\sqrt{t_k-t_{\kappa-1}}.\label{eqn:sqrt-t_k-t_kappa-ineqn}
    \end{align}
    By the definition $\beta_p(t,\delta)=\log\frac{2SAH}{\delta} + (S-1)\log\Big(e(1+\frac{t}{S-1})\Big)$, (\ref{eqn:sqrt-t_k-t_kappa-ineqn}) suggests that
    \begin{align*}
        & (t_k-t_{\kappa-1})\epsilon^{\text{pos}}\leq 4(1+\sqrt{2})H(H+1)\sqrt{\log\frac{2SAH}{\delta_k} + S\log (et_k)}\sqrt{SA} \sqrt{ t_k-t_{\kappa-1}}\\
        \Leftrightarrow & (t_k-t_{\kappa-1}){\epsilon^{\text{pos}}}^2\leq 100(H+1)^4SA \log\frac{2SAH}{\delta_k}+100(H+1)^4S^2A \log (et_k)\\
        \Leftrightarrow & t_k\leq \frac{100(H+1)^4SA \log\frac{2SAH}{\delta_k}}{{\epsilon^{\text{pos}}}^2}+ t_{\kappa-1}+ \frac{100(H+1)^4S^2A }{{\epsilon^{\text{pos}}}^2} + \frac{100(H+1)^4S^2A \log t_k}{{\epsilon^{\text{pos}}}^2}\\
        \stackrel{\text{Lemma }\ref{lemma:inequality-t-logt}}{\Rightarrow} & t_k\leq \frac{100(H+1)^4SA \log\frac{2SAH}{\delta_k}}{{\epsilon^{\text{pos}}}^2}+ t_{\kappa-1}+ \frac{100(H+1)^4S^2A }{{\epsilon^{\text{pos}}}^2} + \\
        & \frac{300(H+1)^4S^2A}{{\epsilon^{\text{pos}}}^2}\log\left(\frac{100(H+1)^4SA \log\frac{2SAH}{\delta_k}}{{\epsilon^{\text{pos}}}^2}+ t_{\kappa-1}+ \frac{100(H+1)^4S^2A }{{\epsilon^{\text{pos}}}^2}\right)\\
        \Rightarrow & t_k\leq \frac{100(H+1)^4SA\log\frac{2SAH}{\delta_k}}{{\epsilon^{\text{pos}}}^2}+t_{\kappa-1} + \\
        & \frac{400(H+1)^4S^2A}{{\epsilon^{\text{pos}}}^2}\log\left(\frac{200(H+1)^4S^2A\log\frac{2SAH}{\delta_k}}{{\epsilon^{\text{pos}}}^2}+t_{\kappa-1}\right)\\
        \stackrel{\text{Lemma }\ref{lemma:length-of-exploration-exploitation-period} }{\Rightarrow} & t_k\leq \frac{100(H+1)^4SA\log\frac{2SAH}{\delta_k}}{{\epsilon^{\text{pos}}}^2}+T_{\kappa-1}^{\text{ee}} + 1\\
        & \frac{400(H+1)^4S^2A}{{\epsilon^{\text{pos}}}^2}\log\left(\frac{200(H+1)^4S^2A\log\frac{2SAH}{\delta_k}}{{\epsilon^{\text{pos}}}^2}+T_{\kappa-1}^{\text{ee}}+1\right)\\
        \stackrel{\text{Lemma }\ref{lemma:log_a_plus_b-upper_bound}}{\Rightarrow} & t_k\leq \frac{101(H+1)^4SA\log\frac{2SAH}{\delta_k}}{{\epsilon^{\text{pos}}}^2}+ T_{\kappa-1}^{\text{ee}} +\\
        & \frac{400(H+1)^4S^2A}{{\epsilon^{\text{pos}}}^2}\log\left(\frac{402(H+1)^4S^2A \log\frac{2SAH}{\delta_k}}{{\epsilon^{\text{pos}}}^2}\right)+ \frac{400(H+1)^4S^2A}{{\epsilon^{\text{pos}}}^2}\log\left(2T_{\kappa-1}^{\text{ee}}\right).
    \end{align*}
    For phase index $k\geq \max\{\kappa, \lceil\log_2 \frac{9600(H+1)^2}{{\epsilon^{\text{pos}}}^2}\rceil\}$, by the Lemma \ref{lemma:minimum-budget-for-T_k^ee}, we can conclude
    \begin{align*}
        & T_{k}^{\text{ee}}\\
        \geq & \frac{101(H+1)^4SA\log\frac{2SAH}{\delta_k}}{{\epsilon^{\text{pos}}}^2}+ T_{k-1}^{\text{ee}} +\\
        & \frac{400(H+1)^4S^2A}{{\epsilon^{\text{pos}}}^2}\log\left(\frac{402(H+1)^4S^2A \log\frac{2SAH}{\delta_k}}{{\epsilon^{\text{pos}}}^2}\right)+ \frac{400(H+1)^4S^2A}{{\epsilon^{\text{pos}}}^2}\log\left(2T_{k-1}^{\text{ee}}\right)\\
        \geq & \frac{101(H+1)^4SA\log\frac{2SAH}{\delta_k}}{{\epsilon^{\text{pos}}}^2}+ T_{\kappa-1}^{\text{ee}} +\\
        & \frac{400(H+1)^4S^2A}{{\epsilon^{\text{pos}}}^2}\log\left(\frac{402(H+1)^4S^2A \log\frac{2SAH}{\delta_k}}{{\epsilon^{\text{pos}}}^2}\right)+ \frac{400(H+1)^4S^2A}{{\epsilon^{\text{pos}}}^2}\log\left(2T_{\kappa-1}^{\text{ee}}\right).
    \end{align*}
    That means Algorithm \ref{alg:Oracle-BPI_UCRL} will fulfill ( \ref{eqn:stop-criterion-pos}) and quit the loop at Line \ref{line-alg:start-execute-MDP-with-new-delta_k}. By the Lemma \ref{lemma:output-correctness-under-good-event}, we can conclude Algorithm \ref{alg:Oracle-BPI_UCRL} will output a policy $\hat{\pi}_k$ satisfying $V^{\hat{\pi}_k}_0(s_0)-\mu_0 \geq \frac{C-1}{C}\Big(V^*_0(s_0)-\mu_0\Big)$ at the end of phase $k$.
    
    For the case of negative instance, the analysis is nearly the same. We can replace $\epsilon^{\text{pos}}$ with $\epsilon^{\text{neg}} = \mu_0-V^{*}_0(s_0)$. By the Lemma \ref{lemma:output-correctness-under-good-event}, we complete the proof.
\end{proof}

The upcoming Lemma \ref{lemma:minimum-budget-for-T_k^ee} shows the minimum phase index for sufficiently large $T_k^{\text{ee}}$..
\begin{lemma}
    \label{lemma:minimum-budget-for-T_k^ee}
    For any $\epsilon >0$, $k\geq \lceil\log_2 \frac{9600(H+1)^2}{\epsilon^2}\rceil $ implies
    \begin{align*}
        & T_{k}^{\text{ee}}\\
        \geq & \frac{101(H+1)^4SA\log\frac{2SAH}{\delta_k}}{{\epsilon}^2}+ T_{k-1}^{\text{ee}} +\\
        & \frac{400(H+1)^4S^2A}{{\epsilon}^2}\log\left(\frac{402(H+1)^4S^2A \log\frac{2SAH}{\delta_k}}{{\epsilon}^2}\right)+ \frac{400(H+1)^4S^2A}{{\epsilon}^2}\log\left(2T_{k-1}^{\text{ee}}\right).
    \end{align*}
\end{lemma}
\begin{proof}
    We can validate the lemma by conducting direct calculation. Recall the definition of $T_k^{\text{ee}}$, 
    \begin{align*}
        T_k^{\text{ee}} = \frac{(H+1)^2SA\log\frac{2SAH}{\delta_k}}{{\epsilon_k}} + 
        \frac{(H+1)^2S^2A}{\epsilon_k}\log\left(\frac{(H+1)^2S^2A \log\frac{2SAH}{\delta_k}}{\epsilon_k}\right),
    \end{align*}
    and we take $\epsilon_k=\frac{1}{2^k}$. We can conclude $k\geq \lceil\log_2 \frac{9600(H+1)^2}{\epsilon^2}\rceil$ implies
    \begin{align*}
        T_k^{\text{ee}} \geq \frac{404(H+1)^4SA\log\frac{2SAH}{\delta_k}}{{\epsilon}^2} + \frac{3200(H+1)^4S^2A}{{\epsilon}^2}\log\left(\frac{9600(H+1)^4S^2A \log\frac{2SAH}{\delta_k}}{{\epsilon}^2}\right).
    \end{align*}
    We can further conclude
    \begin{align}
        \frac{1}{4}T_k^{\text{ee}} \geq & \frac{101(H+1)^4SA\log\frac{2SAH}{\delta_k}}{{\epsilon}^2} + \frac{400(H+1)^4S^2A}{{\epsilon}^2}\log\left(\frac{402(H+1)^4S^2A \log\frac{2SAH}{\delta_k}}{{\epsilon}^2}\right)\label{eqn:T_k-ee-scale-1}\\
        2T_k^{\text{ee}} \geq & \frac{2\cdot 3200(H+1)^4S^2A}{\epsilon^2}\log\left(\frac{3\cdot3200(H+1)^4S^2A}{\epsilon^2}\right)\label{eqn:T_k-ee-scale-2}
    \end{align}
    By the Lemma \ref{lemma:x-alogx>0}, (\ref{eqn:T_k-ee-scale-2}) implies
    \begin{align}
        & 2T_k^{\text{ee}} \geq \frac{2\cdot 3200(H+1)^4S^2A}{\epsilon^2}\log\left(\frac{3\cdot3200(H+1)^4S^2A}{\epsilon^2}\right)\notag\\
        \Rightarrow & 2T_k^{\text{ee}} \geq \frac{3200(H+1)^4S^2A}{\epsilon^2}\log\left(2T_{k}^{\text{ee}}\right)\notag\\
        \Rightarrow & \frac{1}{4}T_k^{\text{ee}} \geq \frac{400(H+1)^4S^2A}{{\epsilon}^2}\log\left(2T_{k-1}^{\text{ee}}\right).\label{eqn:T_k-ee-x-logx-inequality}
    \end{align}
    Meanwhile, $\epsilon_k=\frac{1}{2^k}$ implies
    \begin{align}
        \frac{1}{2}T_k^{\text{ee}} \geq T_{k-1}^{\text{ee}}.\label{eqn:T_k-ee-increasing-speed}
    \end{align}
    Sum up (\ref{eqn:T_k-ee-scale-1}),(\ref{eqn:T_k-ee-x-logx-inequality}) and (\ref{eqn:T_k-ee-increasing-speed}), we have
    \begin{align*}
        & T_{k}^{\text{ee}}\\
        \geq & \frac{101(H+1)^4SA\log\frac{2SAH}{\delta_k}}{{\epsilon}^2}+ T_{k-1}^{\text{ee}} +\\
        & \frac{400(H+1)^4S^2A}{{\epsilon}^2}\log\left(\frac{402(H+1)^4S^2A \log\frac{2SAH}{\delta_k}}{{\epsilon}^2}\right)+ \frac{400(H+1)^4S^2A}{{\epsilon}^2}\log\left(2T_{k-1}^{\text{ee}}\right).
    \end{align*}
\end{proof}

\subsection{Performance Guarantee of Exploitation Stage}
In this subsection, we present the performance Guarantee of the exploitation stage, i.e. Line \ref{alg-line:Exploitation-starts} to \ref{alg-line:Exploitation-ends} in Algorithm \ref{alg:1-policy-identification-recycle-history}.

The following Lemma shows the minimum phase index for sufficiently large $T_k^{\text{et}}$.
\begin{lemma}
    \label{lemma:scale-of-k-exploit}
    \begin{align*}
        & k\geq \lceil \log_2 \frac{H^2}{\omega^2(V_0^{*}(s_0)-\mu_0)^2}\rceil\\
        \Rightarrow & 100\cdot \frac{\log\frac{\alpha_k}{\delta} + \log\log \frac{1}{\epsilon_k}}{\epsilon_k} \geq \frac{28H^2\log\frac{\alpha_k}{\delta}+16H^2\log\left(\log\left(\frac{24H^2}{\omega^2(V_0^{*}(s_0)-\mu_0)^2}\right) \right)}{\omega^2(V_0^{*}(s_0)-\mu_0)^2},
    \end{align*}
    where $\omega$ is an arbitrary constant in $(0, 1)$.
\end{lemma}
\begin{proof}
    We can validate Lemma \ref{lemma:scale-of-k-exploit} by straight forward calculation.
    \begin{align*}
        & k\geq \lceil \log_2 \frac{H^2}{\omega^2(V_0^{*}(s_0)-\mu_0)^2}\rceil\\
        \Rightarrow & \frac{1}{\epsilon_k}\geq \frac{H^2}{\omega^2(V_0^{*}(s_0)-\mu_0)^2}\\
        \Rightarrow & 100\cdot \frac{\log\frac{\alpha_k}{\delta} + \log\log \frac{1}{\epsilon_k}}{\epsilon_k} >\frac{28H^2\log\frac{\alpha_k}{\delta}+16H^2\log\left(\log\left(\frac{24H^2}{\omega^2(V_0^{*}(s_0)-\mu_0)^2}\right) \right)}{\omega^2(V_0^{*}(s_0)-\mu_0)^2}
    \end{align*}
\end{proof}

The upcoming Lemma \ref{lemma:concentration-event-for-exploitation}, \ref{lemma:correctness-length-of-exploitation} are to prove sufficient conditions that Algorithm \ref{alg:1-policy-identification-recycle-history} will accept a candidate policy $\hat{\pi}_k$.
\begin{lemma}
    \label{lemma:concentration-event-for-exploitation}
    Given a policy $\hat{\pi}$, denote $\hat{V}_0^{\hat{\pi},N}$ as the empirical estimation of $V_0^{\hat{\pi}}(s_0)$ by executing $\hat{\pi}$ $N$ times, we have
    \begin{align*}
        \Pr\left(\forall N\in\mathbb{N}, |\hat{V}_0^{\hat{\pi},N}-V_0^{\hat{\pi}}(s_0)| < \sqrt{\frac{H^2 \log\frac{2(\log_2 2N)^2}{\delta}}{N}}\right) \geq 1-\frac{\pi^2}{6}\delta
    \end{align*}
    holds for all $\delta>0$.
\end{lemma}
\begin{proof}
    Denote $X_i$ as the total random reward collected by executing the policy $\hat{\pi}$ in $i$ th time. We have $\hat{V}_0^{\hat{\pi},N}=\frac{\sum_{i=1}^N X_i}{N}$. From the assumption that the reward is always bounded in $[0,1]$, we know $X_i$ is $\frac{H^2}{4}$-subgaussian. By the Lemma \ref{lemma:lil-concentraion-event}, we complete the proof.
\end{proof}

\begin{lemma}
    \label{lemma:correctness-length-of-exploitation}
    Consider a positive instance with $V_0^*(s_0)>\mu_0$. Assume policy $\hat{\pi}_k$ satisfies $V_0^{\hat{\pi}_k}(s_0)-\mu_0 \geq \omega \Big(V_0^*(s_0)-\mu_0\Big)$ for some $\omega > 0$ and $k\geq \lceil \log_2 \frac{H^2}{\omega^2(V_0^{*}(s_0)-\mu_0)^2}\rceil$. If the concentration event 
    \begin{align*}
        \mathcal{E}_k^{\text{et}}=\left\{\forall N\in \mathbb{N}, |\hat{V}_0^{\hat{\pi}_k,N}-V_0^{\hat{\pi}_k}(s_0)| < \sqrt{\frac{H^2 \log\frac{2\alpha_k(\log_2 2N)^2}{\delta}}{N}}\right\}
    \end{align*}
    holds, where $\hat{V}_0^{\hat{\pi}_k,N}$ is the empirical mean value of $V_0^{\hat{\pi}_k}(s_0)$ by executing $\hat{\pi}_k$ $N$ times,  Line \ref{alg-line:Exploitation-starts} to \ref{alg-line:Exploitation-ends} in Algorithm \ref{alg:1-policy-identification-recycle-history} will enter Line \ref{alg-line:output-policy} and output $\hat{\pi}_k$.
\end{lemma}
\begin{proof}
    Given the event $\mathcal{E}_k^{\text{et}}$, we have
    \begin{align*}
        \hat{V}_0^{\hat{\pi}_k,N}-\sqrt{\frac{H^2 \log\frac{2\alpha_k(\log_2 2N)^2}{\delta}}{N}} < V_0^{\hat{\pi}_k}(s_0) < \hat{V}_0^{\hat{\pi}_k,N}+\sqrt{\frac{H^2 \log\frac{2\alpha_k(\log_2 2N)^2}{\delta}}{N}}
    \end{align*}
    holds for all $N\in\mathbb{N}$. Thus, we can conclude
    \begin{align*}
        \Rightarrow & N\geq \frac{28H^2\log\frac{\alpha_k}{\delta}+16H^2\log\left(\log\left(\frac{24H^2}{\omega^2(V_0^{*}(s_0)-\mu_0)^2}\right) \right)}{\omega^2(V_0^{*}(s_0)-\mu_0)^2}\\
        \Rightarrow & N\geq \frac{28H^2\log\frac{\alpha_k}{\delta}+16H^2\log\left(\log\left(\frac{24H^2}{(V_0^{\hat{\pi}_k}(s_0)-\mu_0)^2}\right) \right)}{(V_0^{\hat{\pi}_k}(s_0)-\mu_0)^2}\\
        \Rightarrow & V_0^{\hat{\pi}_k}(s_0)-2\sqrt{\frac{H^2 \log\frac{2\alpha_k(\log_2 2N)^2}{\delta}}{N}} > \mu_0\\
        \Rightarrow & \hat{V}_0^{\hat{\pi}_k,N}-\sqrt{\frac{H^2 \log\frac{2\alpha_k(\log_2 2N)^2}{\delta}}{N}} > \mu_0
    \end{align*}
    The second last step is by the Lemma \ref{lemma:inequality-application-t-loglogt} and the last step is by the event $\mathcal{E}_k^{\text{et}}$. Meanwhile, by the Lemma \ref{lemma:scale-of-k-exploit}, we know $100H^2\cdot \frac{\log\frac{\alpha_k}{\delta} + \log\log \frac{1}{\epsilon_k}}{\epsilon_k} > \frac{28H^2\log\frac{\alpha_k}{\delta}+16H^2\log\left(\log\left(\frac{24H^2}{(V_0^{\hat{\pi}_k}(s_0)-\mu_0)^2}\right) \right)}{(V_0^{\hat{\pi}_k}(s_0)-\mu_0)^2} $, that means the Algorithm \ref{alg:1-policy-identification-recycle-history} will enter Line \ref{alg-line:output-policy} and output $\hat{\pi}_k$..
\end{proof}

\section{Proof of Lower Bound}
\subsection{Full Definition of Instances}
\label{sec:full-def-for-lower}

In this subsection, we present the full definition of Uniform Instance and Tree Instance.

\begin{definition}[Uniform Instance specified by $(S,A,H,r,\epsilon)$]
    Given integers $S,H\geq 10$, $A>100$, $H$ is even, $(S-2)H$ is divisible by $16$, $0<\epsilon < \frac{1}{200\sqrt{5}}$, $\frac{1}{4}<r<r+\epsilon<\frac{3}{4}$, we define $\mathcal{S}=[S-2]\cup\{s_{\text{good},s_{\text{bad}}}\},\mathcal{A}=[A],\mathcal{H}=[H]$. Define a Meaningless instance $\nu$ with reward function $r_h(s,a)=0, \forall (h,s,a)\in[H]\times [S-2]\times\mathcal{A}$, $r_h(s_{\text{good}},a)=1,\forall (h,a)\in [H]\times \mathcal{A}$, $r_h(s_{\text{bad}},a)=0,\forall (h,a)\in [H]\times \mathcal{A}$; initial distribution $p_0(s)=\frac{1}{S-2},\forall s\in [S-2]$, $p_0(s)=0,\forall s\in \{s_{\text{good}},s_{\text{bad}}\}$; transition probability $p_h(s'|s,1)=\begin{cases}
        r+\epsilon & s'=s_{\text{good}}\\
        1-r-\epsilon & s'=s_{\text{bad}}\\
        0 & else
    \end{cases},\forall (h,s)\in [H]\times [S-2]$; 
    $p_h(s'|s,a)=\begin{cases}
        r & s'=s_{\text{good}}\\
        1-r & s'=s_{\text{bad}}\\
        0 & else
    \end{cases},\forall (h,s)\in [H]\times [S-2],a\in\{2,3,\cdots,A\}$;
    $p_h(s'|s,a)=\begin{cases}
        \frac{1}{S-2} & s'\in [S-2]\\
        0 & else
    \end{cases},\forall h\in [H],s\in \{s_{\text{good},s_{\text{bad}}}\},a\in\{1,2,3,\cdots,A\}$
\end{definition}

\begin{definition}[Tree Instance specified by $(S,A,H,r)$]
    Consider integer $S,A,H$ such that $S-1 = 2^{N}$ for some integer $N\geq 1$, and $A\geq 3$, 
    $H-1\geq 6N=6\log_2 (S-1)$, 
    and a real value $r\in (\frac{3}{8},\frac{5}{8})$. We consider a binary tree, with state space $\mathcal{S}=[2^N-1]\cup \{s_{\text{good}},s_{\text{bad}}\}$. The transition probability is
    \begin{align*}
        p_h(2s|s,a=1) = & 1,\forall s\leq 2^{N-1}-1, \forall h\in [H]\\
        p_h(2s+1|s, a=2) = & 1,\forall s\leq 2^{N-1}-1,\forall h\in [H]\\
        p_h(s_{\text{good}}|s, a\geq 3) = & r,\forall s\leq 2^{N-1}-1, \forall h\in [H]\\
        p_h(s_{\text{bad}}|s, a\geq 3) = & 1-r,\forall s\leq 2^{N-1}-1, \forall h\in [H]\\
        p_h(s_{\text{good}}|s, a) = & r,\forall 2^{N}-1\geq s\geq 2^{N-1}, \forall a\in [A], \forall h\in [H]\\
        p_h(s_{\text{bad}}|s, a) = & 1-r,\forall 2^{N}-1\geq s\geq 2^{N-1}, \forall a\in [A], \forall h\in [H]\\
        p_h(s_{\text{good}}|s_{\text{good}}, a) = & 1, \forall h\in [H],\forall a\in[A]\\
        p_h(s_{\text{bad}}|s_{\text{bad}}, a) = & 1, \forall h\in [H],\forall a\in[A]
    \end{align*}
    Assume we always start at the state 1, i.e. $\Pr_{\nu}(s_1^{\pi}=1)=1$ holds for any policy $\pi$. The reward function is $r_h(s,a)=0,\forall s\in [2^N-1]\cup\{s_{\text{bad}}\}$; $r_h(s_{\text{good}},a)=1,\forall h\in[H],a\in [A]$.
\end{definition}


\subsection{Proof of Lower Bound, for positive instances}
\label{sec:Proof-of-Lower-Bound-pos}

Before illustrating the proof of Theorem \ref{theorem:symmetric-positive-lower-bound-simplified}, we first introduce the definition of symmetric algorithm.
\begin{definition}[Symmetric Algorithm]
    Given a 1-policy-identification algorithm alg, an instance $\nu$, two different permutations $\sigma'=\{\sigma_{s,h}'\}_{(s,h)\in \mathcal{S}\times \mathcal{H}}$, $\sigma''=\{\sigma_{s,h}''\}_{(s,h)\in \mathcal{S}\times \mathcal{H}}$. We call the algorithm is symmetric, if for any $T\in\mathbb{N}$, any realization path $\{(s_{t,h},a_{t,h})_{h=1}^H\}_{t=1}^T$ up to $T$ episodes, any output policy $\{a(s,h)\}_{(s,h)\in \mathcal{S}\times\mathcal{H}}$ algorithm can fulfill
    \begin{small}
        \begin{align*}
            & \Pr_{\text{alg}, \nu_{\sigma'}}\Big(\forall t\in [T], \forall h\in [H], S_{t,h}=s_{t,h}, A_{t,h}=\sigma_{s_{t,h},h}'(a_{t,h});\forall s,h, \tau=T,\hat{\pi}_h(s)=\sigma_{s,h}'(a(s,h))\Big)\\
            = & \Pr_{\text{alg}, \nu_{\sigma''}}\Big(\forall t\in [T], \forall h\in [H], S_{t,h}=s_{t,h}, A_{t,h}=\sigma_{s_{t,h},h}''(a_{t,h});\forall s,h, \tau=T,\hat{\pi}_h(s)=\sigma_{s,h}''(a(s,h))\Big)
        \end{align*}
    \end{small}
    Here $S_{t,h}, A_{t,h}$ denote the state and action we take at round $h$, episode $t$.
\end{definition}

The following Theorem suggests that we suffice to consider symmetric algorithm for lower bounds
\begin{theorem}
    \label{theorem:suffice-consider-symmetric-in-lower}
    Denote $\Gamma$ as the set of all the possible permutations on $\nu$. For any instance $\nu$ and any $\delta$-PAC alg, we can find a $\delta$-PAC and symmetric $\tilde{\text{alg}}$, such that $\max_{\sigma\in \Gamma}\mathbb{E}_{\text{alg},\nu_{\sigma}}\tau \geq \mathbb{E}_{\tilde{\text{alg}},\nu}\tau$.
\end{theorem}

Given Theorem \ref{theorem:suffice-consider-symmetric-in-lower}, we suffice to prove the following theorem to show Theorem \ref{theorem:symmetric-positive-lower-bound-simplified} holds.
\begin{theorem}
    \label{theorem:symmetric-positive-lower-bound}
    For any $\delta$-PAC symmetric algorithm and any Uniform Instance specified by $(S,A,H,r,\epsilon)$, with $\frac{3H(r+\epsilon)}{8}+\frac{Hr}{8} > \mu_0 > \frac{Hr}{2}$, $\frac{3H(r+\epsilon)}{8}+\frac{Hr}{8} > \mu_0 > \frac{Hr}{2}$, $S,H\geq 10$, $A\geq 100$, $(S-2)H$ is divisible by $16$, $0<\epsilon < \frac{1}{200\sqrt{5}}$, it holds that
    \begin{align*}
        \mathbb{E}\tau \geq\Omega\left(\frac{H\log\frac{1}{\delta}}{(V^*_0(s_0|\nu)-\mu_0)^2}+\frac{SAH^2}{(V^*_0(s_0|\nu)-\mu_0)^2}\right)
    \end{align*}
\end{theorem}

We first present the proof of Theorem \ref{theorem:suffice-consider-symmetric-in-lower}. After that, we split the proof of Theorem \ref{theorem:symmetric-positive-lower-bound} into multiple steps.
\begin{proof}[Proof of Theorem \ref{theorem:suffice-consider-symmetric-in-lower}]
    We define the symmetric $\tilde{\text{alg}}$ as follows. We choose a permutation $\sigma=\{\sigma_{s,h}\}_{(s,h)\in \mathcal{S}\times \mathcal{H}}$ from $\Gamma$ uniformly and randomly, before the sampling process starts. Then apply the alg to the instance $\nu_{\sigma}$ to conduct the sampling process.

    We first validate $\tilde{\text{alg}}$ is $\delta-PAC$. If the instance $\nu$ is positive, we have
    \begin{align*}
        & \Pr_{\tilde{\text{alg}},\nu}(V_0^{\hat{\pi}_{\tau}}(s_0)\geq \mu_0,\tau<+\infty)\\
        = & \sum_{\sigma\in\Gamma}\frac{\Pr_{\text{alg},\nu_{\sigma}}(V_0^{\hat{\pi}_{\tau}}(s_0)\geq \mu_0,\tau<+\infty)}{(A!)^{SH}}\\
        \geq & \sum_{\sigma\in\Gamma}\frac{1-\delta}{(A!)^{SH}}\\
        = & 1-\delta.
    \end{align*}
    The idea applies to the case that $\nu$ is negative. 

    Then we turn to show $\tilde{\text{alg}}$ is symmetric. Given permutation $\sigma',\sigma''$, since we choose $\sigma\in\Gamma$ in uniformly random manner, we know instances ${(\nu_{\sigma'})}_{\sigma}$ and ${(\nu_{\sigma''})}_{\sigma}$ share the same distribution as $\nu_{\sigma}$. Thus, we can conclude
    \begin{small}
        \begin{align*}
            & \Pr_{\text{alg}, \nu_{\sigma'}}\Big(\forall t\in [T], \forall h\in [H], S_{t,h}=s_{t,h}, A_{t,h}=\sigma_{s_{t,h},h}'(a_{t,h});\forall s,h, \tau=T,\hat{\pi}_h(s)=\sigma_{s,h}'(a(s,h))\Big)\\
            = & \frac{\sum_{\sigma\in\Gamma}\Pr_{\text{alg}, \nu_{\sigma}}\Big(\forall t\in [T], \forall h\in [H], S_{t,h}=s_{t,h}, A_{t,h}=\sigma_{s_{t,h},h}(a_{t,h});\forall s,h, \tau=T,\hat{\pi}_h(s)=\sigma_{s,h}(a(s,h))\Big)}{(A!)^{SH}}\\
            = & \Pr_{\text{alg}, \nu_{\sigma''}}\Big(\forall t\in [T], \forall h\in [H], S_{t,h}=s_{t,h}, A_{t,h}=\sigma_{s_{t,h},h}''(a_{t,h});\forall s,h, \tau=T,\hat{\pi}_h(s)=\sigma_{s,h}''(a(s,h))\Big),
        \end{align*}
    \end{small}
    suggesting that $\tilde{\text{alg}}$ is symmetric.

    The last step is to prove $\max_{\sigma\in \Gamma}\mathbb{E}_{\text{alg},\nu_{\sigma}}\tau \geq \mathbb{E}_{\tilde{\text{alg}},\nu}\tau$. From the definition of algorithm $\tilde{\text{alg}}$, we know
    \begin{align*}
        \mathbb{E}_{\tilde{\text{alg}},\nu}\tau
        = \frac{\sum_{\sigma\in\Gamma}\mathbb{E}_{\text{alg},\nu_{\sigma}}\tau}{(A!)^{SH}}
        \leq \max_{\sigma\in \Gamma}\mathbb{E}_{\text{alg},\nu_{\sigma}}\tau.
    \end{align*}
    And we have completed the proof.
\end{proof}

To prove Theorem \ref{theorem:symmetric-positive-lower-bound}, we present the following two theorems, which is required to prove Theorem \ref{theorem:symmetric-positive-lower-bound}.

\begin{theorem}
    \label{theorem:delta-dependent-lower-bound-pos-meaningless-instance}
    Given a Uniform Instance $\nu$ with $\mu_0$ satisfying that 
    $\frac{H(r+\epsilon)}{2}\geq \mu_0>\frac{Hr}{2}$, 
    for any $\delta$-PAC alg, we have
    \begin{align*}
        \mathbb{E}_{\nu}\tau\geq \Omega\left(\frac{H\log\frac{1}{2.4\delta}}{(V^*_0(s_0|\nu)-\mu_0)^2}\right).
    \end{align*}
\end{theorem}

\begin{theorem}
    \label{theorem:delta-independent-lower-bound-pos-meaningless-instance}
    Given a Uniform Instance $\nu$ with $(S,A,H, r,\epsilon)$, $\mu_0$ satisfying that 
    $\frac{H(r+\epsilon)}{2}> \mu_0\geq \frac{H(r+\epsilon)}{4}+\frac{Hr}{4}$,
    $0<\delta < \frac{1}{2}$, for any $\delta$-PAC and symmetric algorithm, we have
    \begin{align*}
        \mathbb{E}\tau\geq \frac{\Omega(SA)}{\epsilon^2}.
    \end{align*}
\end{theorem}

We first present the proof of Theorem \ref{theorem:delta-dependent-lower-bound-pos-meaningless-instance} and \ref{theorem:delta-independent-lower-bound-pos-meaningless-instance}, which require Lemmas in section \ref{sec:Auxiliary-Lemma-in-the-Proof-of-Lower-Bound}. At the end of this subsection, we present the proof of Theorem \ref{theorem:symmetric-positive-lower-bound}.

\begin{proof}[Proof of Theorem \ref{theorem:delta-dependent-lower-bound-pos-meaningless-instance}]
    Define $r'(\eta)=\frac{\mu_0}{H}-\eta$ for arbitrary $\frac{\mu_0}{H}-r>\eta>0$. We consider an alternative instance $\nu'(\eta)$ which is Uniform Instance with parameters $(S,A,H, r,\frac{2\mu_0}{H}-r-\eta)$. It is evident to see $V^*_0(s_0|\nu'(\eta))=\mu_0-\frac{H}{2}\eta < \mu_0$, suggesting that $\Pr_{\text{alg}, \nu'(\eta)}(\hat{\pi}=\textsf{None})>1-\delta$.

    Notice that $\Pr_{\text{alg}, \nu}(\hat{\pi}=\textsf{None})<\delta$, by the Lemma 1 in \cite{kaufmann2016complexity}, we can conclude
    \begin{align}
        \sum_{s=1}^{S-2}\sum_{h:\text{odd}}\text{kl}(r+\epsilon,\frac{2\mu_0}{H}-\eta )\mathbb{E}_{\text{alg}, \nu}n_h^{\tau}(s,1) \geq \text{kl}(1-\delta,\delta)\label{eqn:delta-dependent-transportation-pos}
    \end{align}
    where $\text{kl}(x,y)=x\log\frac{x}{y}+(1-x)\log\frac{1-x}{1-y}$. From the definition of Simple Requirement, we know $\frac{1}{4}<r<\frac{2\mu_0}{H}-\eta<\frac{3}{4}$. By the Lemma \ref{lemma:Realized-KL-for-Bernoulli}, we know $\text{kl}(r+\epsilon,\frac{2\mu_0}{H}-r-\eta )\leq 3(r+\epsilon-\frac{2\mu_0}{H}+\eta)^2$ and (\ref{eqn:delta-dependent-transportation-pos}) implies
    \begin{align*}
        & \sum_{s=1}^{S-2}\sum_{h:\text{odd}}\mathbb{E}_{\text{alg}, \nu}n_h^{\tau}(s,1) \geq \frac{\text{kl}(1-\delta,\delta)}{3(r+\epsilon-\frac{2\mu_0}{H}+\eta)^2}\\
        \Rightarrow & \frac{1}{H}\sum_{s=1}^{S}\sum_{h=1}^H\sum_{a=1}^{A}\mathbb{E}_{\text{alg}, \nu}n_h^{\tau}(s,a) \geq \frac{H\text{kl}(1-\delta,\delta)}{12(V^*_0(s_0|\nu)-\mu_0+\frac{H}{2}\eta)^2}
    \end{align*}
    By the fact that $\mathbb{E}_{\nu}\tau =\frac{1}{H}\sum_{s=1}^{S}\sum_{h=1}^H\sum_{a=1}^{A}\mathbb{E}_{\text{alg}, \nu} n_h^{\tau}(s,a)$, we can conclude
    \begin{align*}
        \mathbb{E}_{\nu}\tau \geq \frac{H\log\frac{1}{2.4\delta}}{12(V^*_0(s_0|\nu)-\mu_0+\frac{H}{2}\eta)^2}
    \end{align*}
    holds for all $\frac{\mu_0}{H}-r>\eta>0$. Let $\eta\rightarrow 0$, we have completed the proof.
\end{proof}

\begin{proof}[Proof of Theorem \ref{theorem:delta-independent-lower-bound-pos-meaningless-instance}]
    By the Lemma \ref{lemma:delta-PAC-simple-instance-SH/8-output-1}, we can find $\{(s^{(i)},h^{(i)})\}_{i=1}^{\frac{(S-2)H}{16}}$, where $s^{(i)}\in [S-2]$, $h^{(i)}\in [H]$ is odd, such that
    \begin{align*}
        \Pr_{\text{alg},\nu}\big(\hat{\pi}_{h^{(i)}}(s^{(i)})=1,\tau<+\infty\big) \geq \frac{1}{7},\forall i\in \{1,2,\cdots,\frac{(S-2)H}{16}\}.
    \end{align*}
    By the Lemma \ref{lemma:symmetric-alg-must-explore-each-action}, we can conclude
    \begin{align}
        \label{eqn:lower-for-1/16-fraction-of-pairs}
        \mathbb{E}_{\nu, \tilde{\text{alg}}}n_{h^{(i)}}^{\tau}(s^{(i)},a)\geq \frac{1}{2\cdot 10^5\epsilon^2},\forall a\in\{2,\cdots, A\},\forall i\in \{1,2,\cdots,\frac{(S-2)H}{16}\}.
    \end{align}
    Then, we can conclude
    \begin{align*}
        \mathbb{E}_{\nu, \tilde{\text{alg}}}\tau = & \frac{1}{H}\sum_{a=1}^A\sum_{s=1}^{S}\sum_{h=1}^H\mathbb{E}_{\nu, \tilde{\text{alg}}}n_h(s,a)\\
        \geq & \frac{1}{H}\sum_{a=2}^A\sum_{i=1}^{\frac{(S-2)H}{16}}\mathbb{E}_{\nu, \tilde{\text{alg}}}n_{h^{(i)}}^{\tau}(s^{(i)},a)\\
        \geq & \frac{(A-1)(S-2)}{32\cdot 10^5\epsilon^2}.
    \end{align*}
    The last line is by the (\ref{eqn:lower-for-1/16-fraction-of-pairs}), which completes the proof.
\end{proof}

Given Theorem \ref{theorem:delta-dependent-lower-bound-pos-meaningless-instance} and \ref{theorem:delta-independent-lower-bound-pos-meaningless-instance}, we are ready to prove Theorem \ref{theorem:symmetric-positive-lower-bound}.
\begin{proof}[Proof of Theorem \ref{theorem:symmetric-positive-lower-bound}]
    By Theorem \ref{theorem:delta-dependent-lower-bound-pos-meaningless-instance} and \ref{theorem:delta-independent-lower-bound-pos-meaningless-instance}, we can conclude
    \begin{align*}
        \mathbb{E}_{\nu}\tau \geq \Omega\left(\frac{H\log\frac{1}{2.4\delta}}{(V^*_0(s_0|\nu)-\mu_0)^2} + \frac{SA}{\epsilon^2}\right)
    \end{align*}
    By the construction of Uniform Instance $\nu$, we know $V^*_0(s_0|\nu)=\frac{H(r+\epsilon)}{2}$. Since $\frac{3H(r+\epsilon)}{8}+\frac{Hr}{8} > \mu_0 > \frac{Hr}{2}$, we can conclude
    \begin{align*}
        (V^*_0(s_0|\nu)-\mu_0)^2\geq (\frac{H(r+\epsilon)}{2}-\frac{3H(r+\epsilon)}{8}-\frac{Hr}{8})^2=(\frac{H\epsilon}{8})^2=\frac{H^2\epsilon^2}{64}.
    \end{align*}
    The last ste implies the lower bound
    \begin{align*}
        \mathbb{E}_{\nu}\tau \geq \Omega\left(\frac{H\log\frac{1}{2.4\delta}}{(V^*_0(s_0|\nu)-\mu_0)^2} + \frac{SAH^2}{(V^*_0(s_0|\nu)-\mu_0)^2}\right).
    \end{align*}
\end{proof}

\subsection{Proof of Lower Bound, for negative instances}
\label{sec:Proof-of-Lower-Bound-neg}
Before delivering the proof of Theorem \ref{theorem:lower-bound-negative}, we first present the following lemma to calculate the exact value of $V^*_0(s_0|\nu)$ for a Tree instance $\nu$.
\begin{lemma}
    \label{lemma:optimal-value-of-Hard-Instance}
    For a Tree Instance $\nu$ with parameter $(S,A,H,r)$, $V^*_0(s_0|\nu) = (H-1)r$.
\end{lemma}
\begin{proof}
    Given a policy $\pi$, notice that $\{(s,a,h): h:\Pr_{\nu,\pi}(\text{hitting }s \text{ at round }h)>0, \pi_h(s)=a\geq 3\}$ contains at most 1 element. Denote the unique element as $(s^{\pi},a^{\pi},h^{\pi})$. We have $V^{\pi}_0(1)=(H-h^{\pi})r \leq (H-1)r$. And the equality holds only when $\pi_{h=1}(s=1)=a$ for $a\geq 3$.
\end{proof}

Then, we are ready to prove Theorem \ref{theorem:lower-bound-negative}.
\begin{thmstar}[Restatement of Theorem \ref{theorem:lower-bound-negative}]
    Apply any $\delta$-PAC algorithm to a Tree instance specified by $(S,A,H,r)$ with $S,A,H\geq 4$ and $ \frac{3}{4}(H-\log_2 (S-3))> \mu_0 > (H-1)r$, we have
    \begin{align*}
        \mathbb{E}_{\nu}\tau \geq \Omega\left(\frac{HAS\log\frac{1}{\delta}}{(\mu_0 - V^*_0(s_0|\nu))^2}\right).
    \end{align*}
\end{thmstar}
\begin{proof}[Proof of Theorem \ref{theorem:lower-bound-negative}]
    For any pair $(s_0, a_0)$, where $1\leq s_0\leq 2^{N}-1$, $A\geq a_0\geq 3$, we define $h(s)=\lceil \log_2 (s_0+1)\rceil$ and an instance $\nu_{s_0,a_0}(\eta)$, $1-\frac{\mu_0}{H-h(s)}>\eta>0$, such that $\nu_{s_0,a_0}$ only differs in $p_{h(s_0)}^{\nu_{s_0,a_0}}(s_{\text{good}}|s_0,a_0)=\frac{\mu_0}{H-h(s)} +\eta$ compared to $\nu$. 

    We first validate the instance $\nu_{s_0,a_0}(\eta)$ is positive for any $\eta > 0$, by figuring out a qualified policy. From Lemma \ref{lemma:shortest-path-hard-instance}, we know there exists a policy $\pi$ such that $\Pr_{\nu}(s_{h(s_0)}^{\pi}=s_0)=1$. Now, we consider a policy $\pi'$ such that $\pi_h'(s)=\begin{cases}
        a_0 & s=s_0,h=h(s_0)\\
        \pi_h(s) & \text{else}
    \end{cases}$. Evident to see $V^{\pi'}_0(s_0; \nu_{s_0,a_0}(\eta))=(H-h(s_0))\left(\frac{\mu_0}{H-h(s)} +\eta\right) = \mu_0 + \eta(H-h(s_0)) > \mu_0$. 

    Given $\nu_{s_0,a_0}(\eta)$ is positive and $\nu$ is negative, we can conclude $\Pr_{\nu,\text{alg}}(\hat{\pi}=\textsf{None}) \geq 1-\delta$, $\Pr_{\nu_{s_0,a_0}(\eta),\text{alg}}(\hat{\pi}=\textsf{None}) \leq \delta$. Apply the Lemma 1 in \cite{kaufmann2016complexity}, we have
    \begin{align*}
        \mathbb{E}_{\nu,\text{alg}} n_{h(s_0)}^{\tau}(s_0,a_0)\text{kl}(r, \frac{\mu_0}{H-h(s)} +\eta) \geq \log\frac{1}{2.4\delta}.
    \end{align*}
    Since the above inequality holds for arbitrary $\eta$, letting $\eta\rightarrow 0$ implies
    \begin{align}
        & \mathbb{E}_{\nu,\text{alg}} n_{h(s_0)}^{\tau}(s_0,a_0)\text{kl}(r, \frac{\mu_0}{H-h(s)}) \geq \log\frac{1}{2.4\delta}\notag\\
        \Rightarrow & \mathbb{E}_{\nu,\text{alg}} n_{h(s_0)}^{\tau}(s_0,a_0)\cdot3\big(\frac{\mu_0}{H-h(s)}-r\big)^2\geq \log\frac{1}{2.4\delta}\label{eqn:kl-square-bound-neg}\\
        \Rightarrow & \mathbb{E}_{\nu,\text{alg}} n_{h(s_0)}^{\tau}(s_0,a_0)\cdot3\big(\frac{\mu_0}{H-\lceil\log_2 (S-1)\rceil }-r\big)^2\geq \log\frac{1}{2.4\delta}\label{eqn:s_0-no-higher-than_S-2-neg}\\
        \Rightarrow & \mathbb{E}_{\nu,\text{alg}} n_{h(s_0)}^{\tau}(s_0,a_0)\cdot3\big(\frac{\mu_0}{\frac{2}{3}(H-1) }-r\big)^2\geq \log\frac{1}{2.4\delta}\label{eqn:H>=6N-neg}\\
        \Rightarrow & \mathbb{E}_{\nu,\text{alg}} n_{h(s_0)}^{\tau}(s_0,a_0) \geq \frac{\frac{4}{9}(H-1)^2\log\frac{1}{2.4\delta}}{3(\mu_0-(H-1)r)^2} \label{eqn:expected-lower-s_0-a_0}
    \end{align}
    Step (\ref{eqn:kl-square-bound-neg}) is by the assumption $ \frac{3}{4}(H-\log_2 (S-3))> \mu_0 > (H-1)r$, which suggests that $r, \frac{\mu_0}{H-h(s)}\in (\frac{1}{4},\frac{3}{4})$. Lemma \ref{lemma:kl-divergence-quadratic} concludes $\text{kl}(r, \frac{\mu_0}{H-h(s)})\leq 3(\frac{\mu_0}{H-h(s)}-r)^2$. Step (\ref{eqn:s_0-no-higher-than_S-2-neg}) is by $s_0\leq S-2$ and $ \frac{3}{4}(H-\log_2 (S-3))> \mu_0 > (H-1)r$ implies $\frac{\mu_0}{H-\lceil\log_2 (S-1)\rceil } \geq \frac{\mu_0}{H-h(s)} > r$. Step (\ref{eqn:H>=6N-neg}) is by the definition of Tree Instance, which requires $H\geq 6N, 2^N=S-3,N\geq 1$, implying
    \begin{align*}
        & \lceil\log_2 (S-1)\rceil\\
        = & \lceil\log_2 (2^N)\rceil\\
        \leq & N\\
        \leq & \frac{H}{6}.
    \end{align*}
    The last line means $H-\lceil\log_2 (S-1)\rceil\geq \frac{2(H-1)}{3}$, which is step (\ref{eqn:H>=6N-neg}).
    
    Summing up (\ref{eqn:expected-lower-s_0-a_0}) for $s_0\in [S-2]$ and $a_0\in \{3,\cdots, A\}$, we can conclude
    \begin{align*}
        & \sum_{s=1}^{S-2}\sum_{a=3}^A\mathbb{E}_{\nu,\text{alg}} n_{h(s)}^{\tau}(s,a)\geq \frac{4(S-2)(A-1)(H-1)^2\log\frac{1}{2.4\delta}}{27(\mu_0-(H-1)r)^2}\\
        \Rightarrow & \mathbb{E}_{\nu,\text{alg}} \tau = \frac{1}{H}\sum_{h=1}^H\sum_{s=1}^{S}\sum_{a=1}^A\mathbb{E}_{\nu,\text{alg}} n_{h(s)}^{\tau}(s,a) \geq \Omega\left(\frac{HAS\log\frac{1}{\delta}}{(\mu_0 - V^*_0(s_0|\nu))^2}\right).
    \end{align*}
    We have completed the proof.
\end{proof}

The following is an auxiliary Lemma, showing that there exists a shorted path from state 1 to $s$ for any $s\in[S-2]$ in the Tree Instance.
\begin{lemma}
    \label{lemma:shortest-path-hard-instance}
    Consider a Tree Instance $\nu$ specified by parameter $(S,A,H,r)$. For any state $s\in [S-2]$, there exists a policy $\pi$, such that $\Pr_{\nu}(s_{\lceil \log_2 (s+1)\rceil}^{\pi}=s)=1$.
\end{lemma}
\begin{proof}[Proof of Lemma \ref{lemma:shortest-path-hard-instance}]
    We prove the Lemma by induction. For $s=1$, the lemma is true, as we always start at the state $1$. Now, assume for all $s\in [2^{n},2^{n+1}-1]\cap \mathbb{Z}$, the conclusion is true, we consider a state $s \in [2^{n+1},2^{n+2}-1]\cap \mathbb{Z}$. Easy to see $\lfloor \frac{s}{2}\rfloor\in [2^{n},2^{n+1}-1]$. Thus, there exists a policy $\pi$, such that $\Pr_{\nu}(s_{\lceil \log_2 (\lfloor \frac{s}{2}\rfloor+1)\rceil}^{\pi}=\lfloor \frac{s}{2}\rfloor)=1$. Since there exists an action $a_s\in\{1,2\}$ in state $\lfloor \frac{s}{2}\rfloor$ such that $p_{\lceil \log_2 (\lfloor \frac{s}{2}\rfloor+1)\rceil}(s|\lfloor \frac{s}{2}\rfloor, a_s)=1$, we can conclude there exists a policy $\pi'$ such that $\Pr_{\nu}(s_{\lceil \log_2 (\lfloor \frac{s}{2}\rfloor+1)\rceil + 1}^{\pi'}=s)=1$.

    Notice that $2^n+1\leq \lfloor \frac{s}{2}\rfloor+1\leq 2^{n+1}$, suggesting that
    \begin{align*}
        & \lceil \log_2 (\lfloor \frac{s}{2}\rfloor+1)\rceil + 1\\
        = & n+1 + 1\\
        = & n+2\\
        = & \lceil \log_2 (s+1)\rceil.
    \end{align*}
    The last line completes the induction, also the proof of Lemma \ref{lemma:shortest-path-hard-instance}.
\end{proof}

\subsection{Auxiliary Lemmas in the Proof of Lower Bound, for Uniform Instances}
\label{sec:Auxiliary-Lemma-in-the-Proof-of-Lower-Bound}
In this section, We present auxiliary lemmas to support Theorem \ref{theorem:delta-independent-lower-bound-pos-meaningless-instance}. Without extra description, we assume $\frac{3H(r+\epsilon)}{8}+\frac{Hr}{8} > \mu_0 > \frac{Hr}{2}$, $S,H\geq 10$, $A\geq 100$, $(S-2)H$ is divisible by $16$, $0<\epsilon < \frac{1}{200\sqrt{5}}$ always hold in the remaining subsection.

The first lemma is to show a $\delta$-PAC algorithm must output policy that identifies action 1 as the optimal action for ``enough'' pairs of $(s,h)$.
\begin{lemma}
    \label{lemma:delta-PAC-simple-instance-SH/8-output-1}
    Apply a $\delta$-PAC algorithm alg to a Uniform Instance $\nu$ with $(S,A,H, r,\epsilon)$. Assume $\mu_0$ satisfies $\frac{H(r+\epsilon)}{2}> \mu_0\geq \frac{H(r+\epsilon)}{4}+\frac{Hr}{4}$, $0<\delta < \frac{1}{2}$, we can find $\{(s^{(i)},h^{(i)})\}_{i=1}^{\frac{(S-2)H}{16}}$, where $s^{(i)}\in [S-2]$, $h^{(i)}\in [H]$ is odd, such that
    \begin{align*}
        \Pr_{\text{alg},\nu}\big(\hat{\pi}_{h^{(i)}}(s^{(i)})=1,\tau<+\infty\big) \geq \frac{1}{7},\forall i\in \{1,2,\cdots,\frac{(S-2)H}{16}\}
    \end{align*}
\end{lemma}
\begin{proof}[Proof of Lemma \ref{lemma:delta-PAC-simple-instance-SH/8-output-1}]
    Given a policy $\pi$ on the Uniform Instance, we know
    \begin{align*}
        V_0^{\pi}(s_0) = \frac{\sum_{s\in [S-2],\text{odd }h}(r+\epsilon)\mathds{1}(\pi_h(s)=1)+\sum_{s\in [S-2],\text{odd }h}r\mathds{1}(\pi_h(s)\neq 1)}{S-2}.
    \end{align*}
    By the assumption that $\mu_0\geq \frac{H(r+\epsilon)}{4}+\frac{Hr}{4}$ and the $\delta$-PAC requirement, we know
    \begin{align*}
        \Pr_{\text{alg},\nu}\left(\sum_{s\in [S-2],\text{odd }h}\mathds{1}\big(\hat{\pi}_h(s)=1\big) \geq \frac{H(S-2)}{4},\tau<+\infty\right)\geq 1-\delta.
    \end{align*}
    Denote $\{s^{(i)},h^{(i)}\}_{i=1}^{SH}$ as a permutation of $[S-2]\times \{h\in[H]:h \text{ is odd}\}$, such that
    \begin{align*}
        \Pr_{\text{alg},\nu}\left(\hat{\pi}_{h^{(i)}}(s^{(i)})=1,\tau<+\infty\right)\geq \Pr_{\text{alg},\nu}\left(\hat{\pi}_{h^{(i+1)}}(s^{(i+1)})=1,\tau<+\infty\right)
    \end{align*}
    Define $\Lambda=\{C:C\subset [S-2]\times \{h\in[H]:h \text{ is odd}\}, |C|\geq \frac{(S-2)H}{4}\}$, $\Lambda_{\text{all}}=\{C:C\subset [S-2]\times \{h\in[H]:h \text{ is odd}\}\}$. Notice that
    \begin{align*}
        & \left\{\sum_{(s,h)}\mathds{1}\big(\hat{\pi}_h(s)=1\big) \geq \frac{H(S-2)}{4}\right\}\\
        = & \cup_{C\in \Lambda} \{\forall (s,h)\in C, \hat{\pi}_h(s)=1; \forall (s,h)\notin C,\hat{\pi}_h(s)\neq 1\}
    \end{align*}
    $\Pr_{\text{alg},\nu}\left(\sum_{(s,h)}\mathds{1}\big(\hat{\pi}_h(s)=1\big) \geq \frac{H(S-2)}{4},\tau<+\infty\right)\geq 1-\delta$ implies
    \begin{align}
        \label{eqn:selection-of-subset-output-optimal}
        \sum_{C\in \Lambda} \Pr_{\text{alg},\nu}\left(\forall (s,h)\in C, \hat{\pi}_h(s)=1; \forall (s,h)\notin C,\hat{\pi}_h(s)\neq 1\right)\geq 1-\delta.
    \end{align}
    By straight forward calculation,
    \begin{small}
        \begin{align*}
            & \sum_{s\in [S-2],\text{odd }h}\Pr_{\text{alg},\nu}\left( \hat{\pi}_h(s)=1,\tau<+\infty\right)\\
            = & \sum_{s\in [S-2],\text{odd }h}\sum_{C: C\in \Lambda_{\text{all}}, (s,h)\in C}\Pr_{\text{alg},\nu}\left(\forall (s',h')\in C, \hat{\pi}_{h'}(s')=1; \forall (s',h')\notin C, \hat{\pi}_{h'}(s')\neq 1,\tau<+\infty\right)\\
            \geq & \sum_{s\in [S-2],\text{odd }h}\sum_{C: C\in \Lambda, (s,h)\in C}\Pr_{\text{alg},\nu}\left(\forall (s',h')\in C, \hat{\pi}_{h'}(s')=1; \forall (s',h')\notin C, \hat{\pi}_{h'}(s')\neq 1,\tau<+\infty\right)\\
            = & \sum_{s\in [S-2],\text{odd }h}\sum_{C: C\in \Lambda}\mathds{1}\Big((s,h)\in C\Big)\Pr_{\text{alg},\nu}\left(\forall (s',h')\in C, \hat{\pi}_{h'}(s')=1; \forall (s',h')\notin C, \hat{\pi}_{h'}(s')\neq 1,\tau<+\infty\right)\\
            = & \sum_{C: C\in \Lambda}\sum_{s\in [S-2],\text{odd }h}\mathds{1}\Big((s,h)\in C\Big)\Pr_{\text{alg},\nu}\left(\forall (s',h')\in C, \hat{\pi}_{h'}(s')=1; \forall (s',h')\notin C, \hat{\pi}_{h'}(s')\neq 1,\tau<+\infty\right)\\
            \geq & \sum_{C: C\in \Lambda}\frac{(S-2)H}{4}\Pr_{\text{alg},\nu}\left(\forall (s',h')\in C, \hat{\pi}_{h'}(s')=1; \forall (s',h')\notin C, \hat{\pi}_{h'}(s')\neq 1,\tau<+\infty\right)\\
            \geq & \frac{(S-2)H}{4}(1-\delta)\\
            \geq & \frac{(S-2)H}{8}.
        \end{align*}
    \end{small}
    The second last step is by (\ref{eqn:selection-of-subset-output-optimal}) and the last step is by $\delta<\frac{1}{2}$.

    Recall $\Pr_{\text{alg},\nu}\left(\hat{\pi}_{h^{(i)}}(s^{(i)})=1,\tau<+\infty\right)$ is non-decreasing regarding $i$, we can conclude
    \begin{align*}
        & \sum_{s,h}\Pr_{\text{alg},\nu}\left( \hat{\pi}_h(s)=1,\tau<+\infty\right)\\
        = & \sum_{i=1}^{(S-2)H/2}\Pr_{\text{alg},\nu}\left(\hat{\pi}_{h^{(i)}}(s^{(i)})=1,\tau<+\infty\right)\\
        = & \sum_{i=1}^{(S-2)H/16}\Pr_{\text{alg},\nu}\left( \hat{\pi}_{h^{(i)}}(s^{(i)},\tau<+\infty)=1\right) + \sum_{i=(S-2)H/16+1}^{(S-2)H/2}\Pr_{\text{alg},\nu}\left(\hat{\pi}_{h^{(i)}}(s^{(i)})=1,\tau<+\infty\right)\\
        \leq & (S-2)H/16 + \frac{7 (S-2)H}{16}\Pr_{\text{alg},\nu}\left( \hat{\pi}_{h^{((S-2)H/16+1)}}(s^{((S-2)H/16+1)})=1,\tau<+\infty\right).
    \end{align*}
    The last line suggests that
    \begin{align*}
        \frac{(S-2)H}{8}\leq (S-2)H/16 + \frac{7 (S-2)H}{16}\Pr_{\text{alg},\nu}\left( \hat{\pi}_{h^{((S-2)H/16+1)}}(s^{((S-2)H/16+1)})=1,\tau<+\infty\right),
    \end{align*}
    which means 
    \begin{align*}
        \frac{1}{7} \leq \Pr_{\text{alg},\nu}\left( \hat{\pi}_{h^{((S-2)H/16+1)}}(s^{((S-2)H/16+1)})=1,\tau<+\infty\right).
    \end{align*}
    From the non-decreasing order, we can conclude the Lemma.
\end{proof}

Lemma \ref{lemma:delta-PAC-simple-instance-SH/8-output-1} shows there exists at least $\Omega(SH)$ pairs of $(s,h)$ such that a $\delta$-PAC algorithm will identify its optimal action with probability $1/7$. The following lemmas shows a $\delta$-PAC and symmetric algorithm must distinguish arm $1$ and other arms by pulling them $\Omega(\frac{1}{\epsilon^2})$ times in these pairs $(s,h)$.
\begin{lemma}
    \label{lemma:symmetric-alg-must-pull-common-times-prob}
    Apply a $\delta$-PAC symmetric algorithm $\tilde{\text{alg}}$ to a Uniform Instance $\nu$ with $(S,A,H, r,\epsilon)$. If $\exists s_0\in [S-2]$ and even $h_0$ such that $\Pr_{\nu, \tilde{\text{alg}}}(\hat{\pi}_{h_0}(s_0)=1,\tau<+\infty)\geq \frac{1}{7}$, we have
    \begin{align*}
        \Pr_{\nu, \tilde{\text{alg}}}(\hat{\pi}_{h_0}(s_0)=1, n^{\tau}_{h_0}(s_0,1) + n^{\tau}_{h_0}(s_0,a)\geq \frac{1}{M\epsilon^2},\tau<+\infty)\geq \frac{1}{14},\forall a=2,\cdots, A.
    \end{align*}
    where $M=2000$.
\end{lemma}
\begin{proof}[Proof of Lemma \ref{lemma:symmetric-alg-must-pull-common-times-prob}]
    We first show for $a',a''\in \{2,\cdots, A\},a'\neq a''$, we have
    \begin{align}
        \Pr_{\nu, \tilde{\text{alg}}}(\hat{\pi}_{h_0}(s_0)=a', \tau<+\infty)
        = \Pr_{\nu, \tilde{\text{alg}}}(\hat{\pi}_{h_0}(s_0)=a'',\tau<+\infty)\label{eqn:symmetric-treat-same-arm-same-prob}
    \end{align}
    Consider a permutation $\sigma=\{\sigma_{s,h}\}_{i=1}^{SH}$ such that $\sigma_{s,h}(a)=\begin{cases}
        a' & s=s_0,h=h_0,a=a''\\
        a'' & s=s_0,h=h_0,a=a'\\
        a & \text{else}
    \end{cases}$. From the definition of Uniform Instance, we know $p_{h_0}(s'|s_0,a')=p_{h_0}(s'|s_0,a'') $, $r_{h_0}(s_0,a')=r_{h_0}(s_0,a'')$, which means $\nu_{\sigma}=\nu$. Meanwhile, by the definition of symmetric algorithm, we know
    \begin{align*}
        \Pr_{\nu, \tilde{\text{alg}}}(\hat{\pi}_{h_0}(s_0)=a', \tau<+\infty)
        = \Pr_{\nu_{\sigma}, \tilde{\text{alg}}}(\hat{\pi}_{h_0}(s_0)=a'',\tau<+\infty).
    \end{align*}
    Together with the fact that $\nu_{\sigma}=\nu$, we complete the proof.

    In the following, we assume $\frac{1}{M\epsilon^2}$ is an integer for simplicity. To derive Lemma \ref{lemma:symmetric-alg-must-pull-common-times-prob}, we prove it by contradiction. Assume
    \begin{align*}
        \Pr_{\nu, \tilde{\text{alg}}}(\hat{\pi}_{h_0}(s_0)=1, n^{\tau}_{h_0}(s_0,1) + n^{\tau}_{h_0}(s_0,a)\geq \frac{1}{M\epsilon^2},\tau<+\infty)<\frac{1}{14}
    \end{align*}
    holds for some $a\in \{2,3,\cdots, A\}$. From the assumption that $\Pr_{\nu, \tilde{\text{alg}}}(\hat{\pi}_{h_0}(s_0)=1,\tau<+\infty)\geq \frac{1}{7}$, we know
    \begin{align}
        \label{eqn:output-1-with-insufficient-pulling-prob}
        \Pr_{\nu, \tilde{\text{alg}}}(\hat{\pi}_{h_0}(s_0)=1, n^{\tau}_{h_0}(s_0,1) + n^{\tau}_{h_0}(s_0,a)< \frac{1}{M\epsilon^2},\tau<+\infty)\geq\frac{1}{14}
    \end{align}
    Denote event $\mathcal{E}^{(1)}=\{\hat{\pi}_{h_0}(s_0)=1,\tau<+\infty\}$, $\mathcal{E}^{(a)}=\{\hat{\pi}_{h_0}(s_0)=a,\tau<+\infty\}$, $\mathcal{E}_{\text{pull}}=\{n^{\tau}_{h_0}(s_0,1) + n^{\tau}_{h_0}(s_0,a)< \frac{1}{M\epsilon^2}\}$. We consider a permutation $\sigma^{(a)}$ with $\sigma_{s,h}^{(a)}(\tilde{a})=\begin{cases}
        1 & s=s_0,h=h_0,\tilde{a}=a\\
        a & s=s_0,h=h_0,\tilde{a}=1\\
        \tilde{a} & \text{else}
    \end{cases}$. By the symmetry of $\tilde{\text{alg}}$ and (\ref{eqn:output-1-with-insufficient-pulling-prob}), we can conclude
    \begin{align}
        \label{eqn:output-prob-for-permutation-if-insufficient-pull}
        \Pr_{\nu_{\sigma^{(a)}}, \tilde{\text{alg}}}(\mathcal{E}^{(a)}\cap \mathcal{E}_{\text{pull}})\geq\frac{1}{14}.
    \end{align}
    Apply the Transportation Equality(Lemma 18 in \cite{kaufmann2016complexity}), which is based on the idea of Change of Measure, we have
    \begin{align}
        \label{eqn:transportation-output-a-with-insufficient-pulling}
        \Pr_{\nu, \tilde{\text{alg}}}(\mathcal{E}^{(a)}\cap \mathcal{E}_{\text{pull}})
        = \mathbb{E}_{\nu_{\sigma^{(a)}}, \tilde{\text{alg}}}\mathds{1}(\mathcal{E}^{(a)}\cap \mathcal{E}_{\text{pull}})\exp\left(-\sum_{t=1}^{n^{\tau}_{h_0}(s_0,1)}\hat{\text{KL}}_{1,t} - \sum_{t=1}^{n^{\tau}_{h_0}(s_0,a)}\hat{\text{KL}}_{a,t}\right)
    \end{align}
    where $\hat{\text{KL}}_{1,t}=X_{1,t}\log\frac{r+\epsilon}{r} + (1-X_{1,t})\log\frac{1-r-\epsilon}{1-r}$, $\hat{\text{KL}}_{a,t}=X_{a,t}\log\frac{r}{r+\epsilon} + (1-X_{1,t})\log\frac{1-r}{1-r-\epsilon}$, $X_{1,t}|\{X_{1,s}\}_{s=1}^{t-1}\sim \text{Bern}(r+\epsilon)$, $X_{a,t}|\{X_{a,s}\}_{s=1}^{t-1}\sim \text{Bern}(r)$. $\{X_{1,s}\}_{s=1}^{+\infty
    }$ are independent with $\{X_{a,s}\}_{s=1}^{+\infty
    }$. 
    
    Denote 
    \begin{align*}
        \text{kl}(x,y)= & x\log\frac{x}{y}+(1-x)\log\frac{1-x}{1-y}\\
        \mathcal{E}_{\text{kl-cnt}}^{(1)}= & \left\{\max_{1\leq n\leq \frac{1}{M\epsilon^2}}\sum_{t=1}^{n}(\hat{\text{KL}}_{1,t}-\text{kl}(r+\epsilon,r))\leq \frac{1}{4}\right\}\\
        \mathcal{E}_{\text{kl-cnt}}^{(a)}= & \left\{\max_{1\leq n\leq \frac{1}{M\epsilon^2}}\sum_{t=1}^{n}(\hat{\text{KL}}_{a,t}-\text{kl}(r,r+\epsilon))\leq \frac{1}{4}\right\}
    \end{align*}  
    By the Lemma \ref{lemma:Realized-KL-for-Bernoulli}, we know $\hat{\text{KL}}_{1,t}-\text{kl}(r+\epsilon,r)$ and $\hat{\text{KL}}_{a,t}-\text{kl}(r,r+\epsilon)$ are both $8\epsilon^2$-subgaussian. By the Lemma \ref{lemma:Martingale-Concentration}, since $2\cdot\lfloor\frac{1}{M\epsilon^2} \rfloor \cdot 8\epsilon^2\leq \frac{16}{M}$, we have $\Pr_{\nu_{\sigma^{(a)}}, \tilde{\text{alg}}}(\mathcal{E}_{\text{kl-cnt}}^{(1)}) \geq 1-\exp(-\frac{M}{16}\frac{1}{16})$, $\Pr_{\nu_{\sigma^{(a)}}, \tilde{\text{alg}}}(\mathcal{E}_{\text{kl-cnt}}^{(a)})\geq 1-\exp(-\frac{M}{16}\frac{1}{16})$.

    Following (\ref{eqn:transportation-output-a-with-insufficient-pulling}), we have
    \begin{align}
        & \Pr_{\nu, \tilde{\text{alg}}}(\mathcal{E}^{(a)}\cap \mathcal{E}_{\text{pull}})\notag\\
        \geq & \mathbb{E}_{\nu_{\sigma^{(a)}}, \tilde{\text{alg}}}\mathds{1}(\mathcal{E}^{(a)}\cap \mathcal{E}_{\text{pull}}\cap\mathcal{E}_{\text{kl-cnt}}^{(1)}\cap \mathcal{E}_{\text{kl-cnt}}^{(a)})\exp\left(-\sum_{t=1}^{n^{\tau}_{h_0}(s_0,1)}\hat{\text{KL}}_{1,t} - \sum_{t=1}^{n^{\tau}_{h_0}(s_0,a)}\hat{\text{KL}}_{a,t}\right)\notag\\
        = & \mathbb{E}_{\nu_{\sigma^{(a)}}, \tilde{\text{alg}}}\mathds{1}(\mathcal{E}^{(a)}\cap \mathcal{E}_{\text{pull}}\cap\mathcal{E}_{\text{kl-cnt}}^{(1)}\cap \mathcal{E}_{\text{kl-cnt}}^{(a)})\exp(-n^{\tau}_{h_0}(s_0,1)\text{kl}(r+\epsilon, r)-n^{\tau}_{h_0}(s_0,a)\text{kl}(r,r+\epsilon))\notag\\
        &\exp\left(-\sum_{t=1}^{n^{\tau}_{h_0}(s_0,1)}\left(\hat{\text{KL}}_{1,t}-\text{kl}(r+\epsilon, r)\right) - \sum_{t=1}^{n^{\tau}_{h_0}(s_0,a)}\left(\hat{\text{KL}}_{a,t}-\text{kl}( r,r+\epsilon)\right)\right)\notag\\
        \geq & \mathbb{E}_{\nu_{\sigma^{(a)}}, \tilde{\text{alg}}}\mathds{1}(\mathcal{E}^{(a)}\cap \mathcal{E}_{\text{pull}}\cap\mathcal{E}_{\text{kl-cnt}}^{(1)}\cap \mathcal{E}_{\text{kl-cnt}}^{(a)})\exp(-\frac{3\epsilon^2}{M\epsilon^2}-\frac{3\epsilon^2}{M\epsilon^2})\exp\left(-\frac{1}{2}\right)\label{eqn:concentrate-realzied-KL-take-effect}\\
        \geq & \exp\left(-\frac{6}{M}-\frac{1}{2}\right)\left(\Pr_{\nu_{\sigma^{(a)}}, \tilde{\text{alg}}}\left(\mathcal{E}^{(a)}\cap \mathcal{E}_{\text{pull}}\right)-\Pr_{\nu_{\sigma^{(a)}}, \tilde{\text{alg}}}(\neg\mathcal{E}_{\text{kl-cnt}}^{(1)})-\Pr_{\nu_{\sigma^{(a)}}, \tilde{\text{alg}}}(\neg\mathcal{E}_{\text{kl-cnt}}^{(a)})\right)\notag\\
        \geq & \exp\left(-\frac{6}{M}-\frac{1}{2}\right)\left(\Pr_{\nu_{\sigma^{(a)}}, \tilde{\text{alg}}}\left(\mathcal{E}^{(a)}\cap \mathcal{E}_{\text{pull}}\right)-2\exp(-\frac{M}{256})\right)\label{eqn:concentrate-realzied-KL-prob-bound}.
    \end{align}
    Step (\ref{eqn:concentrate-realzied-KL-take-effect}) is by the fact that under event $\mathcal{E}_{\text{pull}}$, $n^{\tau}_{h_0}(s_0,1), n^{\tau}_{h_0}(s_0,a)\leq \frac{1}{M\epsilon^2}$. Given this, event $\mathcal{E}_{\text{kl-cnt}}^{(1)}, \mathcal{E}_{\text{kl-cnt}}^{(a)}$ guarantee that
    \begin{align*}
        & \sum_{t=1}^{n^{\tau}_{h_0}(s_0,1)}\left(\hat{\text{KL}}_{1,t}-\text{kl}(r+\epsilon, r)\right)\leq \max_{1\leq n\leq \frac{1}{M\epsilon^2}}\sum_{t=1}^{n}\left(\hat{\text{KL}}_{1,t}-\text{kl}(r+\epsilon, r)\right)\leq \frac{1}{4}\\
        & \sum_{t=1}^{n^{\tau}_{h_0}(s_0,a)}\left(\hat{\text{KL}}_{a,t}-\text{kl}( r,r+\epsilon)\right)\leq \max_{1\leq n\leq \frac{1}{M\epsilon^2}}\sum_{t=1}^{n}\left(\hat{\text{KL}}_{a,t}-\text{kl}(r,r+\epsilon)\right)\leq \frac{1}{4}
    \end{align*}
    In addition, Lemma \ref{lemma:kl-divergence-quadratic} suggests that $\text{kl}(r+\epsilon, r), \text{kl}(r,r+\epsilon)\leq 3\epsilon^2$, since we require $\frac{1}{4}<r<r+\epsilon<\frac{3}{4}$ in the definition of Uniform Instance. Step (\ref{eqn:concentrate-realzied-KL-prob-bound}) is by the inequalities $\Pr_{\nu_{\sigma^{(a)}}, \tilde{\text{alg}}}(\mathcal{E}_{\text{kl-cnt}}^{(1)}) \geq 1-\exp(-\frac{M}{256})$, $\Pr_{\nu_{\sigma^{(a)}}, \tilde{\text{alg}}}(\mathcal{E}_{\text{kl-cnt}}^{(a)})\geq 1-\exp(-\frac{M}{256})$.

    Combining (\ref{eqn:output-prob-for-permutation-if-insufficient-pull}) and (\ref{eqn:concentrate-realzied-KL-prob-bound}), we have
    \begin{align*}
        \Pr_{\nu, \tilde{\text{alg}}}(\mathcal{E}^{(a)}\cap \mathcal{E}_{\text{pull}}) \geq \exp\left(-\frac{6}{M}-\frac{1}{2}\right)\left(\frac{1}{14}-2\exp(-\frac{M}{256})\right) \geq \frac{1}{100},
    \end{align*}
    which implies $\Pr_{\nu, \tilde{\text{alg}}}(\hat{\pi}_{h_0}(s_0)=a, \tau<+\infty)\geq \frac{1}{100}$. (\ref{eqn:symmetric-treat-same-arm-same-prob}) implies
    \begin{align*}
        \Pr_{\nu, \tilde{\text{alg}}}\big(\hat{\pi}_{h_0}(s_0)\in\{2,\cdots, A\}, \tau<+\infty\big)\geq \frac{A}{100}.
    \end{align*}
    By the requirement of $A>100$ in the definition of Uniform Instance, we have found a contradiction, suggesting we have proved Lemma \ref{lemma:symmetric-alg-must-pull-common-times-prob}.
\end{proof}

Based on Lemma \ref{lemma:symmetric-alg-must-pull-common-times-prob}, the following lemma claims that a symmetric algorithm is unable to distinguish action $1$ and $a$ when the total pulling times of action $1$ and $a$ is below $\frac{1}{2000 \epsilon^2}$.
\begin{lemma}
    \label{lemma:lemma:symmetric-alg-must-pull-inferior-action-times-prob}
    Apply a $\delta$-PAC symmetric algorithm $\tilde{\text{alg}}$ to a Uniform Instance $\nu$ with $(S,A,H, r,\epsilon)$. If $\exists s_0\in [S-2]$ and even $h_0$ such that $\Pr_{\nu, \tilde{\text{alg}}}(\hat{\pi}_{h_0}(s_0)=1,\tau<+\infty)\geq \frac{1}{7}$, we have
    \begin{align*}
        \Pr_{\nu, \tilde{\text{alg}}}(n^{\tau}_{h_0}(s_0,a)\geq \frac{1}{2M\epsilon^2},\tau<+\infty)\geq \frac{1}{50},\forall a=2,\cdots, A.
    \end{align*}
    where $M=2000$.
\end{lemma}
\begin{proof}[Proof of Lemma \ref{lemma:lemma:symmetric-alg-must-pull-inferior-action-times-prob}]
    We first show for $a',a''\in\{2,\cdots,A\},a'\neq a''$, we have
    \begin{align}
        \label{eqn:pulling-dist-same-for-symmetric-algorithm}
        \Pr_{\nu, \tilde{\text{alg}}}(n^{\tau}_{h_0}(s_0,a')\geq \frac{1}{2M\epsilon^2},\tau<+\infty) = \Pr_{\nu, \tilde{\text{alg}}}(n^{\tau}_{h_0}(s_0,a'')\geq \frac{1}{2M\epsilon^2},\tau<+\infty).
    \end{align}
    Similar to the statement at the proof of Lemma \ref{lemma:symmetric-alg-must-pull-common-times-prob}, consider a permutation $\sigma=\{\sigma_{s,h}\}_{i=1}^{SH}$ such that $\sigma_{s,h}(a)=\begin{cases}
        a' & s=s_0,h=h_0,a=a''\\
        a'' & s=s_0,h=h_0,a=a'\\
        a & \text{else}
    \end{cases}$. From the definition of Uniform Instance, we know $p_{h_0}(s'|s_0,a')=p_{h_0}(s'|s_0,a'') $, $r_{h_0}(s_0,a')=r_{h_0}(s_0,a'')$, which means $\nu_{\sigma}=\nu$. Meanwhile, by the definition of symmetric algorithm, we know
    \begin{align*}
        \Pr_{\nu, \tilde{\text{alg}}}(n^{\tau}_{h_0}(s_0,a')\geq \frac{1}{2M\epsilon^2},\tau<+\infty)
        = \Pr_{\nu_{\sigma}, \tilde{\text{alg}}}(n^{\tau}_{h_0}(s_0,a'')\geq \frac{1}{2M\epsilon^2},\tau<+\infty).
    \end{align*}
    Together with the fact that $\nu_{\sigma}=\nu$, we complete the proof.

    (\ref{eqn:pulling-dist-same-for-symmetric-algorithm}) suggests that we suffice to prove $\Pr_{\nu, \tilde{\text{alg}}}(n^{\tau}_{h_0}(s_0,a)\geq \frac{1}{2M\epsilon^2},\tau<+\infty)\geq \frac{1}{50}$ for $a=2$. In the following proof, we fix $a=2$ and we assume $\frac{1}{M\epsilon^2}$ is an integer for simplicity.
    
    Define $\tilde{\tau}_{1,a} = \min\{t: n^{t}_{h_0}(s_0,1) + n^{t}_{h_0}(s_0,a)= \lfloor\frac{1}{M\epsilon^2}\}\rfloor$. 
    From Lemma \ref{lemma:symmetric-alg-must-pull-common-times-prob}, we know
    \begin{align*}
        & \Pr_{\nu, \tilde{\text{alg}}}(\tilde{\tau}_{1,a}\leq \tau < +\infty)\\
        = & \Pr_{\nu, \tilde{\text{alg}}}(n^{\tau}_{h_0}(s_0,1) + n^{\tau}_{h_0}(s_0,a)\geq \frac{1}{M\epsilon^2},\tau<+\infty)\\
        \geq & \Pr_{\nu, \tilde{\text{alg}}}(\hat{\pi}_{h_0}(s_0)=1,n^{\tau}_{h_0}(s_0,1) + n^{\tau}_{h_0}(s_0,a)\geq \frac{1}{M\epsilon^2},\tau<+\infty)\\
        \geq & \frac{1}{14}.
    \end{align*}
    Denote $\mathcal{E}_{1,a}^{\text{pull}} = \{\tilde{\tau}_{1,a}\leq \tau<+\infty, n^{\tilde{\tau}_{1,a}}_{h_0}(s_0,1) \leq n^{\tilde{\tau}_{1,a}}_{h_0}(s_0,a)\}$ and $\mathcal{E}_{a,1}^{\text{pull}} = \{\tilde{\tau}_{1,a}\leq \tau<+\infty, n^{\tilde{\tau}_{1,a}}_{h_0}(s_0,a) \leq n^{\tilde{\tau}_{1,a}}_{h_0}(s_0,1)\}$. Evident to see
    \begin{align}
        \label{eqn:factorize-tilde-tau-by-comparison-result}
        \Pr_{\nu, \tilde{\text{alg}}}(\mathcal{E}_{1,a}^{\text{pull}}) + \Pr_{\nu, \tilde{\text{alg}}}(\mathcal{E}_{a,1}^{\text{pull}})\geq \Pr_{\nu, \tilde{\text{alg}}}(\tilde{\tau}_{1,a}\leq \tau < +\infty)\geq \frac{1}{14}.
    \end{align}
    
    We consider a permutation $\sigma^{(a)}$ with $\sigma_{s,h}^{(a)}(\tilde{a})=\begin{cases}
        1 & s=s_0,h=h_0,\tilde{a}=a\\
        a & s=s_0,h=h_0,\tilde{a}=1\\
        \tilde{a} & \text{else}
    \end{cases}$. By the symmetry of $\tilde{\text{alg}}$, we know
    \begin{align}
        \label{eqn:symmetry-pulling-comparison}
        \Pr_{\nu, \tilde{\text{alg}}}(\mathcal{E}_{1,a}^{\text{pull}}) = \Pr_{\nu_{\sigma^{(a)}}, \tilde{\text{alg}}}(\mathcal{E}_{a,1}^{\text{pull}}), \Pr_{\nu, \tilde{\text{alg}}}(\mathcal{E}_{a,1}^{\text{pull}}) = \Pr_{\nu_{\sigma^{(a)}}, \tilde{\text{alg}}}(\mathcal{E}_{1,a}^{\text{pull}}).
    \end{align}
    
    Apply the Transportation Equality(Lemma 18 in \cite{kaufmann2016complexity}) on the event $\mathcal{E}_{1,a}^{\text{pull}}$, we can derive
    \begin{align}
        \label{eqn:transport-temp-E_a1}
        \Pr_{\nu_{\sigma^{(a)}}, \tilde{\text{alg}}}(\mathcal{E}_{a,1}^{\text{pull}})
        = \mathbb{E}_{\nu, \tilde{\text{alg}}}\mathds{1}(\mathcal{E}_{a,1}^{\text{pull}})\exp\left(-\sum_{t=1}^{n^{\tilde{\tau}_{1,a}}_{h_0}(s_0,1)}\hat{\text{KL}}_{1,t} - \sum_{t=1}^{n^{\tilde{\tau}_{1,a}}_{h_0}(s_0,a)}\hat{\text{KL}}_{a,t}\right),
    \end{align}
    where $\hat{\text{KL}}_{1,t}=X_{1,t}\log\frac{r+\epsilon}{r} + (1-X_{1,t})\log\frac{1-r-\epsilon}{1-r}$, $\hat{\text{KL}}_{a,t}=X_{a,t}\log\frac{r}{r+\epsilon} + (1-X_{1,t})\log\frac{1-r}{1-r-\epsilon}$, $X_{1,t}|\{X_{1,s}\}_{s=1}^{t-1}\sim \text{Bern}(r+\epsilon)$, $X_{a,t}|\{X_{a,s}\}_{s=1}^{t-1}\sim \text{Bern}(r)$. $\{X_{1,s}\}_{s=1}^{+\infty
    }$ are independent with $\{X_{a,s}\}_{s=1}^{+\infty
    }$. Similar to the proof of Lemma \ref{lemma:symmetric-alg-must-pull-common-times-prob}, we introduce 
    \begin{align*}
        \text{kl}(x,y)= & x\log\frac{x}{y}+(1-x)\log\frac{1-x}{1-y}\\
        \mathcal{E}_{\text{kl-cnt}}^{(1)}= & \left\{\max_{1\leq n\leq \frac{1}{M\epsilon^2}}\sum_{t=1}^{n}(\hat{\text{KL}}_{1,t}-\text{kl}(r+\epsilon,r))\leq \frac{1}{4}\right\}\\
        \mathcal{E}_{\text{kl-cnt}}^{(a)}= & \left\{\max_{1\leq n\leq \frac{1}{M\epsilon^2}}\sum_{t=1}^{n}(\hat{\text{KL}}_{a,t}-\text{kl}(r,r+\epsilon))\leq \frac{1}{4}\right\}
    \end{align*}  
    Following (\ref{eqn:transport-temp-E_a1}) and (\ref{eqn:symmetry-pulling-comparison}), we have
    \begin{align}
        & \Pr_{\nu, \tilde{\text{alg}}}(\mathcal{E}_{1,a}^{\text{pull}})\notag\\
        \geq & \mathbb{E}_{\nu, \tilde{\text{alg}}}\mathds{1}(\mathcal{E}_{a,1}^{\text{pull}}\cap \mathcal{E}_{\text{kl-cnt}}^{(1)} \cap \mathcal{E}_{\text{kl-cnt}}^{(a)} )\exp\left(-\sum_{t=1}^{n^{\tilde{\tau}_{1,a}}_{h_0}(s_0,1)}\hat{\text{KL}}_{1,t} - \sum_{t=1}^{n^{\tilde{\tau}_{1,a}}_{h_0}(s_0,a)}\hat{\text{KL}}_{a,t}\right)\notag\\
        = & \mathbb{E}_{\nu, \tilde{\text{alg}}}\mathds{1}(\mathcal{E}_{a,1}^{\text{pull}}\cap \mathcal{E}_{\text{kl-cnt}}^{(1)} \cap \mathcal{E}_{\text{kl-cnt}}^{(a)} )\exp(-n^{\tilde{\tau}_{1,a}}_{h_0}(s_0,1)\text{kl}(r+\epsilon, r)-n^{\tilde{\tau}_{1,a}}_{h_0}(s_0,a)\text{kl}(r,r+\epsilon))\notag\\
        & \exp\left(-\sum_{t=1}^{n^{\tilde{\tau}_{1,a}}_{h_0}(s_0,1)}\big(\hat{\text{KL}}_{1,t}-\text{kl}(r+\epsilon, r)\big) - \sum_{t=1}^{n^{\tilde{\tau}_{1,a}}_{h_0}(s_0,a)}\big(\hat{\text{KL}}_{a,t}-\text{kl}(r,r+\epsilon)\big)\right)\notag\\
        \geq & \mathbb{E}_{\nu, \tilde{\text{alg}}}\mathds{1}(\mathcal{E}_{a,1}^{\text{pull}}\cap \mathcal{E}_{\text{kl-cnt}}^{(1)} \cap \mathcal{E}_{\text{kl-cnt}}^{(a)} )\exp(-\frac{3\epsilon^2}{M\epsilon^2}-\frac{3\epsilon^2}{M\epsilon^2})\exp(-\frac{1}{2})\label{eqn:concentrate-realzied-KL-take-effect-again}\\
        \geq & \exp(-\frac{6}{M}-\frac{1}{2})\left(\Pr_{\nu, \tilde{\text{alg}}}(\mathcal{E}_{a,1}^{\text{pull}})-\Pr_{\nu, \tilde{\text{alg}}}(\neg\mathcal{E}_{\text{kl-cnt}}^{(1)}) - \Pr_{\nu, \tilde{\text{alg}}}(\neg\mathcal{E}_{\text{kl-cnt}}^{(a)})\right)\notag\\
        \geq & \exp(-\frac{6}{M}-\frac{1}{2})\left(\Pr_{\nu, \tilde{\text{alg}}}(\mathcal{E}_{a,1}^{\text{pull}})-2\exp(-\frac{M}{256})\right)\label{eqn:prob-operation-cnt-realized-kl}
    \end{align}
    Step (\ref{eqn:concentrate-realzied-KL-take-effect-again}) is by the fact that $n^{\tilde{\tau}_{1,a}}_{h_0}(s_0,1), n^{\tilde{\tau}_{1,a}}_{h_0}(s_0,a)\leq n^{\tilde{\tau}_{1,a}}_{h_0}(s_0,1)+ n^{\tilde{\tau}_{1,a}}_{h_0}(s_0,a)\leq\frac{1}{M\epsilon^2}$, which is guaranteed by the definition of stopping time $\tilde{\tau}_{1,a}$. And $\mathcal{E}_{\text{kl-cnt}}^{(1)}, \mathcal{E}_{\text{kl-cnt}}^{(a)}$ guarantee
    \begin{align*}
        \exp\left(-\sum_{t=1}^{n^{\tilde{\tau}_{1,a}}_{h_0}(s_0,1)}\big(\hat{\text{KL}}_{1,t}-\text{kl}(r+\epsilon, r)\big) - \sum_{t=1}^{n^{\tilde{\tau}_{1,a}}_{h_0}(s_0,a)}\big(\hat{\text{KL}}_{a,t}-\text{kl}(r,r+\epsilon)\big)\right)\geq \exp(\frac{1}{2}).
    \end{align*}
    Step (\ref{eqn:prob-operation-cnt-realized-kl}) is by the fact that $\Pr_{\nu, \tilde{\text{alg}}}(\neg\mathcal{E}_{\text{kl-cnt}}^{(1)}),\Pr_{\nu, \tilde{\text{alg}}}(\neg\mathcal{E}_{\text{kl-cnt}}^{(a)}),\leq\exp(-\frac{M}{256})$, proved by Lemma \ref{lemma:Martingale-Concentration} and \ref{lemma:Realized-KL-for-Bernoulli}.

    Combining (\ref{eqn:factorize-tilde-tau-by-comparison-result}) and (\ref{eqn:prob-operation-cnt-realized-kl}), we have
    \begin{align*}
        & \Pr_{\nu, \tilde{\text{alg}}}(\mathcal{E}_{1,a}^{\text{pull}}) \geq \exp(-\frac{6}{M}-\frac{1}{2})\left(\frac{1}{14}-\Pr_{\nu, \tilde{\text{alg}}}(\mathcal{E}_{1,a}^{\text{pull}})-2\exp(-\frac{M}{256})\right)\\
        \Leftrightarrow & \Pr_{\nu, \tilde{\text{alg}}}(\mathcal{E}_{1,a}^{\text{pull}})\geq \frac{\exp(-\frac{6}{M}-\frac{1}{2})\left(\frac{1}{14}-2\exp(-\frac{M}{256})\right)}{1+\exp(-\frac{6}{M}-\frac{1}{2})}\\
        \Rightarrow & \Pr_{\nu, \tilde{\text{alg}}}(\mathcal{E}_{1,a}^{\text{pull}})\geq \frac{1}{50}\\
        \Rightarrow & \Pr_{\nu, \tilde{\text{alg}}}\left(\tilde{\tau}_{1,a}\leq \tau <+\infty, n^{\tilde{\tau}_{1,a}}_{h_0}(s_0,a) + n^{\tilde{\tau}_{1,a}}_{h_0}(s_0,1)=\frac{1}{M\epsilon^2}, n^{\tilde{\tau}_{1,a}}_{h_0}(s_0,a) \geq n^{\tilde{\tau}_{1,a}}_{h_0}(s_0,1)\right)\geq \frac{1}{50}\\
        \Rightarrow & \Pr_{\nu, \tilde{\text{alg}}}\left(n^{ \tau}_{h_0}(s_0,a)\geq \frac{1}{2M\epsilon^2}, \tau <+\infty\right)\geq \frac{1}{50}
    \end{align*}
    The last line suggests that we have proved the Lemma.
\end{proof}

Given Lemma \ref{lemma:lemma:symmetric-alg-must-pull-inferior-action-times-prob}, we are ready to prove the lower bounds for expected pulling times.
\begin{lemma}
    \label{lemma:symmetric-alg-must-explore-each-action}
    Apply a $\delta$-PAC symmetric algorithm $\tilde{\text{alg}}$ to a Uniform Instance $\nu$ with $(S,A,H, r,\epsilon)$. If $\exists s_0\in [S-2]$ and even $h_0$ such that $\Pr_{\nu, \tilde{\text{alg}}}(\hat{\pi}_{h_0}(s_0)=1)\geq \frac{1}{7}$, we have
    \begin{align*}
        \mathbb{E}_{\nu, \tilde{\text{alg}}}n_{h_0}^{\tau}(s_0,a)\geq \frac{1}{100 M\epsilon^2},\forall a\in\{2,\cdots, A\}
    \end{align*}
    where $M=2000$.
\end{lemma}
\begin{proof}[Proof of Lemma \ref{lemma:symmetric-alg-must-explore-each-action}]
    By the Lemma \ref{lemma:lemma:symmetric-alg-must-pull-inferior-action-times-prob}, we know
    \begin{align*}
        \Pr_{\nu, \tilde{\text{alg}}}\left(n^{ \tau}_{h_0}(s_0,a)\geq \frac{1}{2M\epsilon^2}\right)\geq \frac{1}{50}
    \end{align*}
    holds for all $a\in \{2,\cdots, A\}$. We can prove Lemma \ref{lemma:symmetric-alg-must-explore-each-action} by applying the Markov Inequality.
\end{proof}

\section{Technical Inequality}
In this section, we present technical inequalities used by the above proof for Lemma and Theorems, including multiple inequalities and concentration theorems.
\begin{lemma}
    \label{lemma:Martingale-Concentration}
    Consider Random variable $\{X_t\}_{t=1}^{+\infty}$ such that $X_t|\{X_s\}_{s=1}^{t-1}$ is $\sigma^2$-subgaussian with mean value $0$. Then for an integer $N$, we have
    \begin{align*}
        \Pr\left(\max_{1\leq n\leq N}\sum_{t=1}^n X_t > x\right) \leq \exp(-\frac{x^2}{2N\sigma^2}).
    \end{align*}
\end{lemma}
\begin{proof}[Proof of Lemma \ref{lemma:Martingale-Concentration}]
    Take $\lambda = \frac{x}{N\sigma^2}$. We first validate $\{\exp\left(\lambda\sum_{t=1}^nX_t\right)\}_{n=1}^{+\infty}$ is a submartingale. By Jensen's Inequality, 
    \begin{align*}
        \mathbb{E}\left[\exp\left(\lambda\sum_{t=1}^nX_t\right)|X_1,\cdots, X_{n-1}\right]\geq \exp\left(\lambda\sum_{t=1}^{n-1}X_t\right)\exp(\mathbb{E}\lambda X_n) = \exp\left(\lambda\sum_{t=1}^{n-1}X_t\right).
    \end{align*}
    
    By direct calculation, we have
    \begin{align}
        & \Pr\left(\max_{1\leq n\leq N}\sum_{t=1}^n X_t > x\right)\notag\\
        = & \Pr\left(\max_{1\leq n\leq N}\exp\left(\lambda\sum_{t=1}^nX_t\right)  > \exp(\lambda x)\right)\notag\\
        \leq & \frac{\mathbb{E}\exp\left(\lambda\sum_{t=1}^NX_t\right)}{\exp(\lambda x)}\label{eqn:submartingale-maximal}\\
        \leq & \frac{\exp(\frac{N\lambda^2\sigma^2}{2})}{\exp(\lambda x)}\notag\\
        = & \exp(-\frac{x^2}{2N\sigma^2}).
    \end{align}
    Step (\ref{eqn:submartingale-maximal}) is by the maximal inequality for the submartingale.
\end{proof}

\begin{lemma}
    \label{lemma:Realized-KL-for-Bernoulli}
    Assume $p,q\in(\frac{1}{4},\frac{3}{4})$, and $X\sim \text{Bern}(p)$, let
    \begin{align*}
        Z:=X\log\frac{p}{q}+(1-X)\log\frac{1-p}{1-q}-(p\log\frac{p}{q}+(1-p)\log\frac{1-p}{1-q}),
    \end{align*}
    $Z$ is $8(p-q)^2$-subgaussian, i.e. for any $\lambda\in\mathbb{R}$, we have $\mathbb{E}[\exp(\lambda Z)\leq \exp(4\lambda^2 (p-q)^2)$
\end{lemma}
\begin{proof}[Proof of Lemma \ref{lemma:Realized-KL-for-Bernoulli}]
    Not hard to see $\mathbb{E}Z=0$ and 
    \begin{align*}
        Z = & (X-p)\left(\log\frac{p}{q}-\log\frac{1-p}{1-q}\right)\\
        = & (X-p)\log\frac{\frac{p}{1-p}}{\frac{q}{1-q}}
    \end{align*}
    Since $X-p$ is a $\frac{1}{4}$-subgaussian random variable, we suffice to show $\frac{1}{4}\left(\log\frac{\frac{p}{1-p}}{\frac{q}{1-q}}\right)^2\leq 8(p-q)^2, \forall p,q\in(\frac{1}{4},\frac{3}{4})$. Notice that
    \begin{align*}
        & \frac{1}{4}\left(\log\frac{\frac{p}{1-p}}{\frac{q}{1-q}}\right)^2\leq 8(p-q)^2\\
        \Leftrightarrow & |\frac{\log\frac{p}{1-p}-\log\frac{q}{1-q}}{p-q}| \leq 2\sqrt{8}\\
        \Leftarrow & \sup_{\frac{1}{4}\leq x\leq \frac{3}{4}} |\frac{d\log\frac{x}{1-x}}{dx}| \leq  2\sqrt{8}\\
        \Leftrightarrow & \sup_{\frac{1}{4}\leq x\leq \frac{3}{4}}|\frac{1}{x}+\frac{1}{1-x}|\leq 2\sqrt{8}
    \end{align*}
    The second last step is by the median value theorem. Since $\sup_{\frac{1}{4}\leq x\leq \frac{3}{4}}|\frac{1}{x}+\frac{1}{1-x}|=|\frac{1}{x}+\frac{1}{1-x}|_{x=\frac{1}{4}\text{ or }\frac{3}{4}}=\frac{16}{3}<2\sqrt{8}$, we have completed the proof.
\end{proof}

\begin{lemma}
    \label{lemma:kl-divergence-quadratic}
    Define $\text{kl}(x,y)=x\log\frac{x}{y}+(1-x)\log\frac{1-x}{1-y}$
    . For $x,y$ such that $\frac{1}{4}<x,y<\frac{3}{4}$, we have
    \begin{align*}
        \text{kl}(x,y)\leq 3(x-y)^2
    \end{align*}
\end{lemma}
\begin{proof}[Proof of Lemma \ref{lemma:kl-divergence-quadratic}]
    Given $x\in (\frac{1}{4},\frac{4}{3})$, define $f(y)=x\log\frac{x}{y}+(1-x)\log\frac{1-x}{1-y} - 3(x-y)^2$. We have
    \begin{align*}
        f'(y)=& -\frac{x}{y}+\frac{1-x}{1-y} - 6(y-x)\\
        = & \frac{-x+y}{y(1-y)} - 6(y-x)\\
        = & (y-x)\left(\frac{1}{y(1-y)}-6\right)
    \end{align*}
    Notice that $\frac{3}{16}\leq y(1-y)\leq \frac{1}{4}$ holds for $y\in (\frac{1}{4},\frac{3}{4})$, which means $\frac{1}{y(1-y)}-6\leq 0$ holds for $y\in (\frac{1}{4},\frac{3}{4})$. We can conclude $f'(y)$ is decreasing in the interval $(x, \frac{3}{4})$ and increasing in the interval $(\frac{1}{4},x)$. That means
    $\max_{\frac{1}{4}<y<\frac{3}{4}}f(y)=f(x)=0$, which proves the conclusion.
\end{proof}

\begin{lemma}
    \label{lemma:inequality-t-logt}
    Given any $b\geq a\geq 2$. If $x>0$ satisfies $x\leq b+a\log x$, we can conclude $x< b+3a\log b$.
\end{lemma}
\begin{proof}
    Prove by contradiction. Assume $x\geq b+3a\log b$. Notice that $\frac{d(x-a\log x-b)}{dx} = 1-\frac{a}{x}$, which implies $x-a\log x-b$ increases at the interval $(b + 3a\log b, +\infty)$. In addition, we have
    \begin{align*}
        & x-a\log x-b |_{x=b + 3a\log b}\\
        = & b + 3a\log b - a\log(b + 3a\log b) - b\\
        = & 3a\log b - a\log(b + 3a\log b)\\
        = & a\log\frac{b^3}{b+3a\log b}\\
        \geq & a\log\frac{b^3}{b+3b\log b}\\
        = & a\log\frac{b^2}{1+3\log b}.
    \end{align*}
    As $b^2\geq 1+3\log b$ for any $b\geq 2$, ($2^2\geq 1+3\log 2$), we can conclude $x-a\log x-b \geq 0$ holds for all $x\geq b+3a\log b$, which implies a contradiction.
\end{proof}

\begin{lemma}
    \label{lemma:log_a_plus_b-upper_bound}
    $\log(a+b) \leq \log (4ab)=\log (2a) + \log (2b)$ for $a,b\geq 1$
\end{lemma}
\begin{proof}
    Notice that for  $a,b\geq 1$, we have $4ab-a-b=a(2b-1)+b(2a-1)>0$, we can conclude $\log(a+b) \leq \log (4ab)$.
\end{proof}

\begin{lemma}
    \label{lemma:x-alogx>0}
    Given $a\geq 1$, we have $x\geq 2a\log(3a)\Rightarrow x\geq a\log x$
\end{lemma}
\begin{proof}
    Define $f(x)=x-a\log x$. By $\frac{df}{dx}=1-\frac{a}{x}$, we know $f$ is increasing in the interval $(a, +\infty)$. By the fact that $a\geq 1$, we have $\log(3a)\geq 1$, further $2a\log(3a)\geq a$. Meanwhile, not hard to see
    \begin{align*}
        & a\geq 1\\
        \Rightarrow & \log(3a)\geq \log(2a), \log(3a)\geq \log(\log(3a))\\
        \Rightarrow & 2\log(3a)\geq \log(2a)+\log(\log(3a))\\
        \Leftrightarrow & 2\log(3a)\geq \log(2a\log(3a))\\
        \Rightarrow & 2a\log(3a)-a\log(2a\log(3a))\geq0.
    \end{align*}
    We can conclude $f(2a\log(3a))\geq 0$, which means $x\geq 2a\log(3a)\Rightarrow x\geq a\log x$.
\end{proof}

\begin{lemma}[Lemma 19 in \cite{JMLR:v11:jaksch10a}]
    \label{lemma:summation-of-prob}
    For $0\leq z_k\leq 1$, $M\geq0$, we have $\sum_{k=1}^n \frac{z_k}{\sqrt{\max\{ M+\sum_{i=1}^{k-1}z_i, 1\}}} \leq (\sqrt{2}+1) \sqrt{\sum_{i=1}^n z_i}$
\end{lemma}
\begin{proof}
    Denote $Z_{k-1}:=\max\big\{1,\sum_{i=1}^{k-1}z_i\big\}$. Trivial to see $Z_{k-1}\geq z_k$. We use induction over $n$. 
    First, we show the lemma holds for $n=1$. The reason is
    \begin{align*}
        \sum_{k=1}^1 \frac{z_k}{\sqrt{\max\{ M+\sum_{i=1}^{n-1}z_i, 1\}}}=\frac{z_1}{\max\{ M, 1\}}\leq z_1\leq \sqrt{z_1}< (\sqrt{2}+1) \sqrt{z_1}.
    \end{align*}
    
    Then, we turn to work on two cases for a general $n$, given $n-1$ is correct, we consider two cases $\sum_{k=1}^{n-1}z_k \leq 1$ or $\sum_{k=1}^{n-1}z_k > 1$. If $\sum_{k=1}^{n-1}z_k \leq 1$, we have 
    $\sum_{k=1}^{n}z_k=z_n + \sum_{k=1}^{n-1}z_k\leq 2$, which implies $\sum_{k=1}^{n}z_k \leq \sqrt{2} \sqrt{\sum_{k=1}^{n}z_k}$. It follows that 
    \begin{align*}
        & \sum_{k=1}^n \frac{z_k}{\sqrt{\max\{ M+\sum_{i=1}^{k-1}z_i, 1\}}}\\
        \leq & \sum_{k=1}^{n-1} z_k + z_n\\
        \leq & \sqrt{2} \sqrt{\sum_{k=1}^{n}z_k}\\
        < & (\sqrt{2}+1) \sqrt{\sum_{k=1}^{n}z_k}.
    \end{align*}
    Then, we show that the lemma also holds for $n$, when $\sum_{k=1}^{n-1}z_k > 1$. By the induction hypothesis, we have
    \begin{align*}
        & \sum_{k=1}^n \frac{z_k}{\max\{ M+\sum_{i=1}^{k-1}z_i, 1\}}\\
        \leq & (\sqrt{2}+1)\sqrt{\sum_{k=1}^{n-1} z_k} + \frac{z_n}{\sqrt{\max\{\sum_{i=1}^{n-1}z_i, 1\}}}\\
        = & (\sqrt{2}+1)\sqrt{\sum_{k=1}^{n-1} z_k} + \frac{z_n}{\sqrt{\sum_{i=1}^{n-1}z_i}}\\
        = & \sqrt{\left((\sqrt{2}+1)\sqrt{\sum_{k=1}^{n-1} z_k} + \frac{z_n}{\sqrt{\sum_{i=1}^{n-1}z_i}}\right)^2}\\
        = & \sqrt{(\sqrt{2}+1)^2(\sum_{k=1}^{n-1} z_k) + 2(\sqrt{2}+1) z_n + \frac{z_n^2}{\sum_{k=1}^{n-1} z_k}}\\
        \stackrel{z_n\leq 1, \sum_{k=1}^{n-1} z_k\geq 1}{\leq} & \sqrt{(\sqrt{2}+1)^2(\sum_{k=1}^{n-1} z_k) + 2(\sqrt{2}+1) z_n + z_n}\\
        = & \sqrt{(2\sqrt{2}+3)(\sum_{k=1}^{n-1} z_k) + (2\sqrt{2}+3) z_n}\\
        = & (\sqrt{2}+1)\sqrt{\sum_{k=1}^{n} z_k},
    \end{align*}
    which means the induction holds.
\end{proof}

The remaining two lemmas are from \cite{pmlr-v267-li25f}. We don't repeat the proof as they are exactly the same.
\begin{lemma}[Lemma D.2 in \cite{pmlr-v267-li25f}]
    \label{lemma:inequality-application-t-loglogt}
    For any $\Delta \in (0, 1], K\geq 2, \delta\in (0, \frac{1}{2}], C\geq 1$, we can conclude
    \begin{align*}
        & t > \frac{28C^2\log\frac{2K}{\delta}}{\Delta^2} + \frac{16 C^2\log\left(\log\left(\frac{24C^2}{\Delta^2}\right) \right)}{\Delta^2}\\
        \Rightarrow & C\sqrt{\frac{4\log\frac{2K (\log_2 2t)^2}{\delta}}{t}} < \Delta
    \end{align*}
\end{lemma}

\begin{lemma}[Lemma D.3 in \cite{pmlr-v267-li25f}]
    \label{lemma:lil-concentraion-event}
    Assume $\{X_i\}_{i=1}^{+\infty}$ are i.i.d random variables with mean value $\mu$. Assume $\{X_i-\mu\}_{i=1}^{+\infty}$ are $\sigma^2$-subgaussian. We have
    \begin{align*}
        \Pr\left(\forall 
        N\in\mathbb{N}, \left|\frac{\sum_{i=1}^N X_i}{N}-\mu \right| < \sqrt{\frac{4\sigma^2 \log\frac{2(\log_2 2N)^2}{\delta}}{N}}\right) \geq 1-\frac{\pi^2}{6}\delta
    \end{align*}
    holds for all $\delta>0$.
\end{lemma}


\end{document}